\RequirePackage{fix-cm}
\documentclass[twocolumn]{svjour3}            
\smartqed  

\usepackage{graphicx}
\usepackage{epstopdf}
\usepackage{subfigure}
\usepackage{amsmath,amssymb}

\usepackage{url}              
\usepackage{amsfonts}         
\usepackage{nicefrac}         
\usepackage{microtype}        
\usepackage{algorithm}        
\usepackage{algorithmic}      
\usepackage{multirow}         

\usepackage{epstopdf}
\usepackage{paralist}
\usepackage{setspace}
\usepackage{color}
\usepackage{xcolor}
\usepackage{multirow}
\usepackage{colortbl}
\usepackage{array,caption}
\usepackage{textcomp,booktabs}
\usepackage{bbm}
\usepackage{lipsum}
\usepackage{bbding}
\usepackage{pifont}

\definecolor{tscite}{RGB}{65 105 225}
\definecolor{tslink}{RGB}{205 140 149}
\usepackage[colorlinks, linkcolor=tslink, anchorcolor=black, citecolor=tscite]{hyperref}

\usepackage{apacite}
\usepackage{natbib}
\usepackage{bbding}

\newcommand{\cmark}{\ding{51}}%
\newcommand{\xmark}{\ding{55}}%

\newcommand{\changed}[1]{\textcolor{black}{#1}}

\def\modelshortname{DIFO++} 
\def\modelshortnameold{DIFO} 
\definecolor{cmblu}{RGB}{51,102,240}
\definecolor{cmred}{RGB}{51,102,240}
\definecolor{varcolor}{RGB}{0,153,102}

\newtheorem{reslemma}{Restatement of Lemma}

\journalname{International Journal of Computer Vision}
\begin{document}

\title{Source-Free Domain Adaptation with Vision-Language Prior}

\author{Song~Tang \and Yunxiang~Bai \and Wenxin~Su \and Mao~Ye{\Envelope} \and Jianwei~Zhang \and Xiatian~Zhu{\Envelope}}


\institute{Song Tang \at
              Institute of Machine Intelligence~(IMI), University of Shanghai for Science and Technology, Shanghai, China; Technical Aspects of Multimodal Systems~(TAMS) Group, Department of Informatics, Universität Hamburg, Hamburg, Germany. \\
              \email{song.tang@uni-hamburg.de}
          \and
          Yunxiang~Bai \at
              Institute of Machine Intelligence~(IMI), University of Shanghai for Science and Technology, Shanghai, China.  \\
              \email{baiyunxiang11@gmail.com}
          \and 
          Wenxin~Su \at
              European Molecular Biology Laboratory, Heidelberg, Germany.  \\
              \email{wenxin.su@embl.de}
          \and 
          Mao Ye \at
              School of Computer Science and Engineering, University of Electronic Science and Technology of China, Chengdu, China.  \\
              \email{cvlab.uestc@gmail.com}
          \and
          Jianwei Zhang \at
              TAMS Group, Department of Informatics, Universität Hamburg, Hamburg, Germany.\\ 
              \email{jianwei.zhang@uni-hamburg.de}
          \and
          Xiatian Zhu \at
              Surrey Institute for People-Centred Artificial Intelligence, and Centre for Vision, Speech and Signal Processing (CVSSP), University of Surrey, Guildford, UK.  \\
              \email{xiatian.zhu@surrey.ac.uk}
}

\date{Received: date / Accepted: date}

\maketitle

\begin{abstract}
Source-Free Domain Adaptation~(SFDA) seeks to adapt a source model, which is pre-trained on a supervised source domain, for a target domain, with only access to unlabeled target training data.  
Relying on pseudo labeling and/or auxiliary supervision, conventional methods are inevitably error-prone.
To mitigate this limitation, in this work we for the first time explore the potentials of off-the-shelf vision-language (ViL) multimodal models (e.g., CLIP) with rich whilst heterogeneous knowledge. 
We find that directly applying the ViL model to the target domain in a zero-shot fashion is unsatisfactory, as it is not specialized for this particular task but largely generic. 
To make it task-specific, we propose a novel {\modelshortname} approach.
Specifically, {\modelshortname} alternates between two steps during adaptation:
(i) Customizing the ViL model by maximizing the mutual information with the target model in a prompt learning manner, 
(ii) Distilling the knowledge of this customized ViL model to the target model, centering on gap region reduction.
During progressive knowledge adaptation, we first identify and focus on the gap region, where enclosed features are entangled and class-ambiguous, as it often captures richer task-specific semantics.
Reliable pseudo-labels are then generated by fusing predictions from the target and ViL models, supported by a memory mechanism.
Finally, gap region reduction is guided by category attention and predictive consistency for semantic alignment, complemented by referenced entropy minimization to suppress uncertainty. 
Extensive experiments show that {\modelshortname} significantly outperforms the state-of-the-art alternatives. 
Our code and data are available at \url{https://github.com/tntek/DIFO-Plus}. 

\end{abstract}

\section{Introduction} \label{sec:introduction}
Unsupervised Domain Adaptation (UDA) relies on both well-annotated source data and unannotated target data. 
\changed{Such a paradigm is particularly crucial in applications like autonomous driving (e.g., domain-adaptive semantic segmentation~\citep{chen2023pipa,chen2024transferring}) and medical image analysis (e.g., cross-modality diagnosis~\citep{mai_wang2024consistency}).}
However, due to heightened safety and privacy concerns, accessing source data freely has become difficult \citep{2019Distant, lao2021hypothesis}.
In response, Source-Free Domain Adaptation (SFDA) has gained attention as a more practical solution, aiming to transfer a pre-trained source model to the target domain using only unlabeled target data.  

Due to the absence of source samples, traditional distribution matching approaches are no longer viable~\citep{ganin2015unsupervised,kang2019contrastive}. The predominant alternative is self-supervised learning, which generates or mines auxiliary information to facilitate unsupervised adaptation. Two main approaches exist: constructing a pseudo source domain to leverage established UDA methods such as adversarial learning~\citep{xia2021adaptive,kurmi2021domain} or domain shift minimization based on distribution measurement~\citep{ding2022source,tian2021vdm,kundu2022balancing,ZHOU2024109974} and mining extra supervision from the source model~\citep{lao2021hypothesis,wang2022exploring,huang2021model} or target data~\citep{yang2022attracting,tang2023source,wang2025silan}.
In the presence of domain distribution shift, applying the source model to the target domain introduces inevitable errors in pseudo-labeling or auxiliary supervision, thereby limiting adaptation performance.

\begin{figure}[t] 
    \begin{center}
     \includegraphics[width=0.97\linewidth]{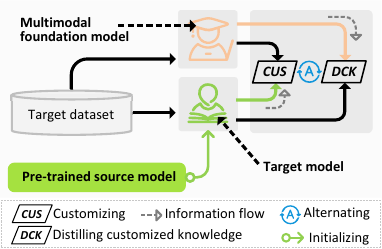}
    \end{center}
    \caption{
    We expand beyond traditional SFDA methods that rely solely on a pretrained source model and unlabeled target data. Instead, we innovate by exploring off-the-shelf multimodal foundation models, such as CLIP, in an unsupervised manner (marked by the box with blue background).
    }
    \label{fig:idea-comp}
\end{figure}

To address identified limitations, we pioneer the exploration of off-the-shelf multimodal foundation models, such as the vision-language (ViL) model CLIP~\citep{radford2021learning}, transcending the constraints of both the source model and target data knowledge. 
However, direct application of the ViL model proves unsatisfactory, lacking specialization for specific tasks. 
To tackle this problem, in our preliminary conference work~\citep{tang2024sourcefree}, we propose a novel distillation approach named \textit{\textbf{D}istilling mult\textbf{I}modal \textbf{F}oundation m\textbf{O}del (\textbf{\modelshortnameold})}. 
Fig.~\ref{fig:idea-comp} demonstrates the idea of {\modelshortnameold}. 
Initially, {\modelshortnameold} customizes the ViL model through unsupervised prompt learning to incorporate task-specific information. 
Since both the ViL model and the source model struggle to make accurate predictions in a target domain, we adopt unbiased mutual information maximization to regulate this prompt learning. 
Consequently, {\modelshortnameold} distills the knowledge from this customized ViL model to the target model, with joint supervision through two designed regularization terms. 
The first strategy, most-likely category encouragement, prompts the model to concentrate on the most likely classes, thus facilitating coarse-grained distillation. In contrast, the second strategy, predictive consistency, ensures fine-grained distillation by maintaining consistency in predictions.  

We observe that training samples contribute differently to expressing task specificity during model adaptation.
In particular, easy samples offer limited new information, whereas task-specific knowledge is often concentrated in hard samples, many of which reside in the {\it gap region} characterized by uncertainty and class ambiguity.
However, the {\modelshortnameold} approach treats all samples equally during training. 
This uniform attention strategy may cause {\modelshortnameold} to overlook valuable information hidden within the gap region, especially as the proportion of easy data increases. 

Therefore, we propose {\modelshortname}, an enhanced SFDA framework with a vision-language prior, building upon {\modelshortnameold}.  
Specifically, in the phase of distilling customized knowledge (Fig.~\ref{fig:idea-comp}), we propose a novel gap region-driven scheme, allowing the distillation to occur through a process of reducing the gap region. 
To address this, we first identify the gap region using a novel metric named referenced entropy, which captures uncertainty and class ambiguity in feature space.
We then assign pseudo-labels to the uncertain region by fusing historical cross-model predictions.
To progressively reduce the gap region, we introduce three regularization strategies: {\it Category attention calibration}, {\it Predictive Consistency}, and {\it Gap region compression}, which collectively drive the features in the gap region toward more discriminative and class-consistent areas.  

Beyond {\modelshortnameold}~\citep{tang2024sourcefree} published in CVPR24, we introduce several key differences with  {\modelshortname}:
(1) Additionally considering data uncertainty to enhance distilling of a frozen ViL model. 
(2) Extending the DIFO framework by introducing a new gap region-driven progressive distillation strategy. 
This approach ensures that adaptation focuses on semantically rich, hard-to-classify regions rather than treating all samples uniformly. 
(3) Providing more experimental analysis and discussion, including the rationale analysis of evolving DIFO to the {\modelshortname} model.

To sum up, we make these \textbf{contributions} below:  
\begin{itemize}
\item[(1)] Pioneering the use of generic but heterogeneous knowledge sources (e.g., the off-the-shelf ViL model) for the SFDA problem, transcending the limited knowledge boundary of a pre-trained source model and unlabeled target data. 

\item[(2)] Development of a novel {\modelshortname} approach to effectively distill useful task-specific knowledge from the general-purpose ViL model. 
Specifically, our method leverages referenced uncertainty to identify and rank gap region samples by their difficulty, enabling the model to progressively distill knowledge from increasingly informative and task-specific examples. 

\item[(3)] Extensive evaluation on standard benchmarks, demonstrating the significant superiority of our {\modelshortname} over previous state-of-the-art alternatives under conventional closed-set settings, as well as more challenging partial-set and open-set settings.
\end{itemize}

\begin{figure*}[t]
    \begin{center}
        \includegraphics[width=0.98\linewidth,angle=0]{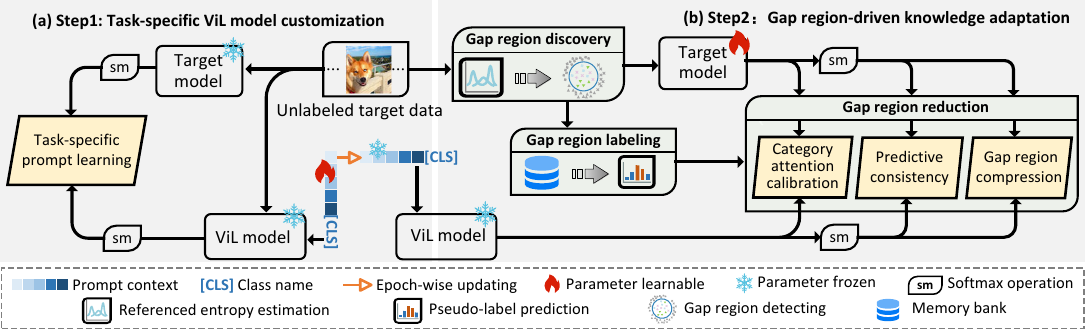} 
    \end{center}
    \caption{
    Overview of {\modelshortname}: Our method alternates between two steps. 
    First, we perform (a) {\it task-specific customization} (\texttt{Sec.\ref{sec:tsc}}) of a ViL model through task-specific prompt learning ($\mathcal{L}_{\text{S-I}}$). 
    This is achieved under soft predictive guidance using mutual information maximization. 
    Second, we perform (b) {\it gap region-driven knowledge adaptation} (\texttt{Sec.\ref{sec:gra}}), which recasts knowledge transfer as the task of shrinking the gap region, in which feature samples are entangled and exhibit high class ambiguity. 
    To this end, we first identify the gap region using referenced entropy, and then label the gap region by fusing the predictions of the target and ViL models that are stored in a memory bank. 
    Finally, we reduce the gap region by incorporating regularizations: Category attention calibration ($\mathcal{L}_{\text{cac}}$), Predictive consistency ($\mathcal{L}_{\text{pc}}$), and Gap region compression ($\mathcal{L}_{\text{rc}}$).     
    This pipeline is structured into three functional modules: (1) Gap region discovery (\texttt{Sec.\ref{sec:ga-disco}}), (2) Gap region labeling (\texttt{Sec.\ref{sec:ga-label}}), and (3) Gap region reduction (\texttt{Sec.\ref{sec:ga-redu}}). 
    }
    \label{fig:ov}
\end{figure*}

\section{Related Work}\label{sec:rework}
\subsection{Source-Free Domain Adaptation}
Existing SFDA approaches fall into three distinct categories. The first explicitly aligns the pseudo source domain with the target domain, treating SFDA as a specialized case of unsupervised domain adaptation. This alignment is achieved by constructing the pseudo source domain through a generative model~\citep{tian2022vdm} or by splitting the target domain based on prior source hypotheses~\citep{du2021ps}.

The second group extracts cross-domain factors from the source domain and transfers them via successive model adaptation to align feature distributions across the two domains. For example, \citep{tang2019adaptive} establishes a mapping relationship from a sample and its exemplar Support Vector Machine (SVM) (an individual classifier) on the source domain to ensure individual classification on the target domain. Some approaches leverage pre-trained source models to generate auxiliary factors, such as multi-hypothesis~\citep{lao2021hypothesis}, prototypes~\citep{tanwisuth2021pct,ZHOU2024109974}, source/target distribution estimation~\citep{ding2022source,lee2022confidence}, or hard samples~\citep{li2021divergence} to aid in feature alignment.

The third group incorporates auxiliary information refined from the unlabeled target domain. In addition to widely used pseudo-labels~\citep{liang2020we, chen2022self, rai2025label}, prediction confidence~\citep{wang2025silan, xu2025revisiting}, geometry information, such as intrinsic neighborhood structure~\citep{Litrico_2023_CVPR}, graph/hyper-graph connection~\citep{hwang2024sfda, jiang2025hg}, and target data manifold~\citep{tang2022sclm,tang2023source}, has also been exploited.

Despite continual advancements, these methods are limited by the knowledge derived solely from the pretrained source model and unlabeled target data. We break this limitation by tapping into the rich knowledge encoded in off-the-shelf multimodal foundation models.


\vspace{-0.5cm}
\changed{
\subsection{Uncertainty-aware Domain Adaptation}
}

\changed{
Within the Domain Adaptation (DA) context, uncertainty serves as a critical extrinsic manifestation of domain shift. Consequently, uncertainty-aware approaches have long been a hallmark of DA research. 
While these methods share a common philosophy of easy-to-hard adaptation, their specific methodological distinctions emerge in the following two aspects:
}

\changed{
\textit{(1) In terms of the mechanism for identifying hard samples}, existing work is in two lines.
The first one is output-based methods using standard entropy~\citep{ent7bishop2006pattern}, loss values~\citep{r2yang2022divide, g0yu2025smdanet}, energy-based scores~\citep{liu2020energy}, and class margin~\citep{f0song2024multi}. 
The second one is the characteristic-based methods that leverage geometric feature structures~\citep{r3zhu2021rich, i0yang2021exploiting} or historical information~\citep{h0lai2023memory}.
For instance, \citep{r3zhu2021rich} utilizes sliding window voting to propagate labels and distinguish between easy and hard samples. 
Alternatively, \citep{i0yang2021exploiting} employs local affinity for sample mining, and \citep{h0lai2023memory} characterizes hard samples via sub-prototypes derived from historical feature memory.
}

\changed{
\textit{(2) In terms of the strategy for exploiting the mined hard data to facilitate adaptation},
existing strategies generally fall into three categories. The first follows a divide-and-conquer paradigm: \citep{r3zhu2021rich} applies pseudo-labels to easy samples while using adversarial learning to enforce hard-to-easy feature alignment; \citep{g0yu2025smdanet} employs KL-regularized weighting to align target data with the source domain; and \citep{f0song2024multi} focuses exclusively on high-confidence samples, effectively discarding hard data. 
The second category \citep{r2yang2022divide,r4shin2020two} reformulates the adaptation task as semi-supervised learning, treating easy-sample pseudo-labels as ground truth. 
The third category leverages consistency learning to drive feature clustering-based adaptation: \citep{h0lai2023memory} utilizes sub-prototypes and cycle-consistent matching to classify private target classes, while \citep{i0yang2021exploiting} enforces label consistency among hard samples with high local affinity.
}

\changed{
In summary, regarding the two methodological dimensions above, {\modelshortname} diverges significantly from existing approaches:
\begin{itemize}
    \item [(1)] The proposed referenced entropy explicitly leverages the initial semantic cues provided by the source model for the target domain. This addresses a critical gap in existing approaches, which often ignore this inherent prior information in favor of generic uncertainty metrics.
    \item [(2)] Our curriculum-driven distillation employs an unbiased alignment anchored by mutual information entropy regularization. By avoiding absolute reliance on specific subsets, this design circumvents confirmation bias and error propagation. In contrast, existing methods' over-dependence on easy samples results in a restrictive unidirectional alignment, where labeling errors in easy data are easily amplified. 
\end{itemize}
}

\vspace{0.5cm}

\subsection{Large multimodal models}
Multimodal vision-language (ViL) models, such as CLIP \citep{radford2021learning} and OpenCLIP~\citep{cherti2023reproducible}, have shown promise across various mono-modal and multimodal tasks by capturing modality-invariant features. Approaches in this domain can be broadly categorized into two lines.
The first line focuses on enhancing ViL model performance. For instance, in~\citep{zhou2022learning,ge2025pluda}, prompt learning optimizes the text encoder through the use of tailored, learnable prompts designed for specific scenarios. Other efforts aim to improve data efficiency by repurposing noisy data~\citep{andonian2022robust} or by leveraging multimodal synthetic data~\citep{li2025enhancing}. 

The second line utilizes ViL models as external knowledge to enhance downstream tasks, as demonstrated in this paper. Previous work in knowledge transfer primarily falls into two frameworks. For the first scheme, where the ViL model is directly applied to the target task in a zero-shot fashion~\citep{liang2023open}, domain generality is leveraged without task-specific refinement. The second scheme does not focus on source model adaptation. Instead, it fine-tunes the ViL model~\citep{li2025unified} to the target domain through prompt or adaptor learning with an amount of manual labels~\citep{Cho_2023_ICCV}.

Most recently, multimodal large models such as CLIP and  ShareGPT~\citep{chen2024sharegpt4v} have been introduced to advance SFDA, substantially enhancing adaptation performance by leveraging generic knowledge in these models. 
{\bf Three orthogonal strategies} are adopted: {(1)} leveraging knowledge encoded in a single foundation model via knowledge adaptation~\citep{tang2024sourcefree,zhang2025source}, {(2)} integrating multi-LLM-teacher guidance via curriculum learning~\citep{chen2024empowering}, and {(3)} denoising ViL predictions \citep{tang2025proxy}. 
Our {\modelshortname} aligns with the first strategy.

Among previous work, the most relevant work to {\modelshortname} is Co-learn-V~\citep{zhang2025source}. 
Specifically, in each adapting epoch, {\modelshortname} first tailors CLIP’s generic knowledge to the specific task, and then performs gap-region-driven knowledge distillation that explicitly accounts for data uncertainty during adaptation.
Co-learn-V adopts a prototype-guided approach where zero-shot CLIP predictions are injected to refine the target prototypes. 
Two key features distinguish {\modelshortname} from Co-learn-V:
(1) a new customize–then–distill paradigm rather than a conventional prototype-based scheme integrating the original CLIP knowledge, and
(2) explicit modeling of data uncertainty, which Co-learn-V does not consider.

\section{Methodology} \label{sec:method}

{\bf Problem statement.} 
We have two distinct yet related domains, the source and target domains, characterized by the same $C$ categories.
The source domain is represented by a source model $\theta_s\!:\!\mathcal{X}_s \!\to \!\mathcal{Y}_s$ pretrained on a labeled training data $ \mathcal{D}_s = \{\mathcal{X}_s, \mathcal{Y}_s\}$.
For the target domain, we only have  $n$ unlabeled training data
$\mathcal{X}_t\!=\!\{{\boldsymbol{x}_{i}\}_{i=1}^{n}}$ without any category labels.
The objective is to derive a target model $\theta_t\!:\!\mathcal{X}_t\! \to \!\mathcal{Y}_t$ by adapting the source model $\theta_s$ on $\mathcal{X}_t$.
To that end, we explore the potential of a Visual-Language (ViL) model ${{\theta}}_{v}$.

\vspace{0.1cm}
\noindent
\textbf{\textit{Overview of \modelshortname.}}
To address SFDA with a ViL model, we propose the {\modelshortname} framework,
as depicted in Fig.~\ref{fig:ov}.
Our method alternates between two distinct stages to customize and adapt rich ViL knowledge reliably during adaptation. 
In the {\it first} stage, we customize the ViL model via prompt learning.
This serves to suppress the guidance errors from the ViL model. 

In the {\it second} stage, we perform gap region-driven knowledge adaptation, which frames the knowledge transfer as a process of reducing the gap region where the enclosed features are highly entangled and class-ambiguous. 
This process begins by identifying gap regions, whose difficulty gradually increases, using the proposed metric of referenced entropy.
Next, the identified gap regions are labeled through a fusion of predictions from the target model and the customized ViL models, assisted by a memory bank that retains historical information.
Subsequently, we reduce the gap region by adapting knowledge between the target model and the customized ViL models.
This adaptation is guided by a tailored constraint that promotes the selection of the most probable category labels in the logit space, while simultaneously preserving predictive consistency. 
Additionally, gap region compression regularization further strengthens the adaptation by minimizing referenced uncertainty.

\subsection{Stage I: ViL Model Customization} \label{sec:tsc}
We adopt the prompt learning framework for ViL model customization,
with ViL's parameters all frozen throughout.
The only learnable part in customization is the prompts each assigned for a specific class.
To optimize these prompts, we need a quality supervision.
In SFDA, however, it is challenging to customize such a domain-generic ViL model towards the target domain, at the absence of a well-trained target domain model.
This is because, none of them can reasonably make predictions.
That means there is no clearly preferred supervision signals available.

To address this challenge, we propose to explore the wisdom of the crowd
by leveraging their predictive interaction as more reliable supervision.
Formally, given an unlabeled target sample ${\boldsymbol{x}}_{k}$, we denote the predictions by the target model and the ViL model as $\theta_{t}\left({\boldsymbol{x}}_{i}\right)$ and $\theta_{v}\left({\boldsymbol{x}}_{i}\right)$, respectively.
We conduct the ViL model customization by maximizing the mutual information of the two predictions as:
\begin{equation}
    \label{eqn:loss_mim}
    \begin{aligned}
        \mathcal{L}_{\text{S-I}} =- \min_{\boldsymbol{v}}{\mathbb{E}_{{\boldsymbol{x}}_{i} \in {\mathcal{X}_t}}}{{\rm{I}}}\left(\theta_{t}\left({\boldsymbol{x}}_{i}\right), {\theta_{v}} \left(\boldsymbol{x}_{i}, \boldsymbol{v} \right)\right)
    \end{aligned}
\end{equation}
where $\boldsymbol{v}$ is the prompt context to be learned and the function ${{\rm{I}}}(\cdot,\cdot)$ measures the mutual information~\citep{ji2019invariant}. 

This alignment design differs significantly from the conventional adoption of the Kullback–Leibler (KL) divergence. 
First, mutual information is a lower optimization bound than KL divergence, facilitating deeper alignment \changed{(see Lemma~\ref{thm-one} which can be recalled from \citep{cover2006elements})}. 
\begin{lemma}
\textit{Given two random variables $X$, $Y$. Their mutual information ${\rm{I}}\left( X, Y \right)$ and KL divergence $D_{\rm{KL}}\left( X||Y \right)$ satisfy the unequal relationship as follows.} 
\begin{equation}
    \label{eqn:dsib}
    -{\rm{I}}\left( X, Y \right) \leq D_{\rm{KL}}\left( X, Y \right). 
\end{equation} 
\label{thm-one} 
\end{lemma}

\vspace{-0.7cm}
\changed{
{\bf Remark.} We presented this lemma to provide a self-contained theoretical motivation for using mutual information in our context, rather than claiming it as an original mathematical contribution.
}


Crucially, the KL divergence exhibits an inherent bias towards a specific prediction, making it less suitable for our context where none of the predictions holds a clear advantage. On the contrary, mutual information considers the joint distribution or correlation between the two predictions. This distinction arises from their respective definitions: 
$-{\rm{I}}\left( X, Y \right)=-H\left( X\right)+H\left( X|Y \right)$ and $D_{\rm{KL}}\left( X, Y \right)=-H\left( X\right)+H\left( X:Y \right)$, where
\begin{equation}
    \label{eqn:discuss-bias}
    \begin{split}
        H\left( X~|~Y \right) &=-\sum p(\boldsymbol{x},\boldsymbol{y})\log p(\boldsymbol{x} | \boldsymbol{y})\\
        H\left( X:Y \right)   &=-\sum p(\boldsymbol{x})\log p(\boldsymbol{y}).
    \end{split}
\end{equation}

\subsection{Stage II: Gap Region-driven Knowledge Adaptation} \label{sec:gra} 

At this stage, we adapt the customized ViL knowledge by explicitly reducing the gap region. 
This strategy is motivated by an empirical observation of the target feature distribution under the source model’s mapping.
As illustrated in Fig.~\ref{fig:gap-region} (a), the feature space presents a dual structural pattern: (1) class-consistent, island-like clusters characterized by high intra-class compactness, and (2) transitional regions situated between these clusters, where features are entangled and class-ambiguous. 
These transitional areas often give rise to model uncertainty and misclassification due to their semantic ambiguity. 
Given their role in separating well-formed class clusters, we refer to these regions as {\it gap regions}.

In practice, we first identify the region and subsequently assign pseudo-labels to it. 
With the resulting pseudo semantics, the gap region is progressively reduced.
The following subsections provide an elaboration on this process.  

\subsubsection{Gap region discovery} \label{sec:ga-disco}

The most salient characteristic of the gap regions is the intermixing of data points from multiple classes, which is typically reflected in high uncertainty. 
For the issue of gap region discovery, an uncertainty metric capable of achieving high discriminability plays a fundamental role. 
Existing approaches often use the absolute value of the signal/prediction's attribution, such as entropy~\citep{Bi2022Entropyweighted}, energy~\citep{banitalebi2023ebcdet}, and cluster density~\citep{choi2019pseudo}, to evaluate data uncertainty. 
They implicitly reference a uniform probability distribution as the baseline. 
However, in the SFDA context, the learning process begins with a non-uniform probability distribution, as defined by the pre-trained source model. 
Thus, the previous metrics are not suitable for SFDA. 

\changed{{\bf Remark}: Although prior theories suggest that entropy-based uncertainty can be brittle \citep{gal2016dropout,gawlikowski2023survey,maddox2019simple}, these critiques are grounded in supervised learning. This diverges from the SFDA setting (lacking labels and involving domain shift), rendering such classical insights less applicable. 
We therefore follow the established paradigm of using information entropy as a viable proxy for uncertainty in SFDA \citep{ent1xu2025revisiting,ent2lee2025duet,ent5safaei2025certainty}.}

To address this issue, we propose a novel metric, {\it referenced entropy}. 
For a specific training sample, its semantic attributes, such as category, foreground, background, and classification difficulty, are fixed. 
A learned well model should map these fixed attributes to stable outputs. 
In this view, we regard high uncertainty as significant differences in prediction, correspondingly reflected as vibration in terms of information entropy. 
Inspired by this, we introduce the referenced entropy, specifically designed to measure entropy variation relative to an evolving, adaptation-aware probability distribution reference. 
This design incorporates both the source model and the adaptation process. 

\begin{figure}[t]
    \setlength{\abovecaptionskip}{-2pt} 
    \begin{center}
        \includegraphics[width=0.95\linewidth,angle=0]{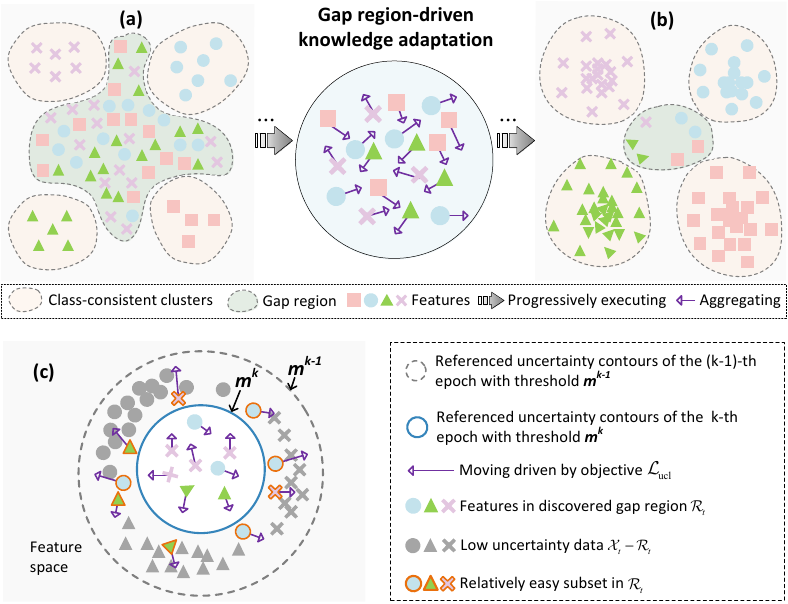} 
    \end{center}
    \caption{
    \changed{
     An illustrative idea of gap region reduction. 
    {\bf (a)} The gap region refers to the transitional area between class-consistent clusters, detected by the proposed metric of referenced entropy.
    {\bf (b)} Through a gap region-driven knowledge adaptation, features within the gap region are properly consolidated, reducing its extent.
    {\bf (c)} A snapshot of this progressive adaptation, in which $\mathcal{L}_{\text{rc}}$ guides the movement of features within the gap region $\mathcal{R}_t$, ultimately resulting in gap region compression.}
    }
    \label{fig:gap-region} 
\end{figure}

Assume the gap region is collectively denoted by ${\mathcal{R}}_t \subset \mathcal{X}_t$.  
With the proposed referenced uncertainty, we accomplish the gap region discovery, as outlined below.

\vspace{0.2cm}
\noindent
\textbf{\textit{(a) Referenced entropy estimation.}}
First, we derive an evolving adaptation-aware reference, formulated as:  
\begin{equation}
    \label{eqn:uncer-anchor}
    \begin{split}
        \rho_i^{k} &= \rho_i^{k-1} + \delta h_i^k,~k \in [0, K-1],
    \end{split}
\end{equation}
where $\delta$ is a weighting parameter; $K$ is the number of epochs, and
$h_i^k$ is the entropy of the prediction by the current in-training model.

Each sample's reference $\rho_i^{k}$ 
is initiated (when $k=0$) 
by calculating the entropy of the source model's prediction.
As in Eq.~\eqref{eqn:uncer-anchor}, the reference updates in the Exponential Moving Average (EMA) fashion, providing accumulative uncertainty estimation by continuously incorporating the adaptation dynamics over epochs.

With $\rho_i^{k}$, the referenced uncertainty of any target sample $\boldsymbol{x}_i\in\mathcal{X}_t$ in the $k$-th epoch is then measured as: 
\begin{equation}
    \label{eqn:uncer-esti}
    \begin{split}
        U(\boldsymbol{x}_i) &= H(\theta_t(\boldsymbol{x}_i)) - \rho_i^{k},
    \end{split}
\end{equation}
where $H(\theta_t(\boldsymbol{x}_i))$ returns the predictive entropy of $\boldsymbol{x}_i$ with the current target model $\theta_t$.

\vspace{0.1cm}
\noindent
\textbf{\textit{(b) Detecting gap region.}}
To achieve a gradual discovery process, we introduce an uncertainty threshold $m^k, k\in [0, K-1]$ that increases epoch-wise to schedule the adaptation process. 
For each training iteration in the $k$-th epoch, we split the hard set from the entire target data, denoted \changed{by $\mathcal{R}_t$}, which consists of the data whose uncertainty is larger than the current epoch's threshold $m^k$. 
Formally,  
\begin{equation}\label{eqn:confi-split} 
{\mathcal{R}}_t = \left\{ \boldsymbol{x}_i ~|~ U(\boldsymbol{x}_i) > m^k \right\}_{i=1}^{|{\mathcal{R}}_t|},~m^k = \epsilon \cdot \gamma^{k},
\end{equation}
where $\epsilon \in (0, 1)$ and $ \gamma > 1$.


The discovery rule in Eq.~\eqref{eqn:confi-split} suggests that by increasing $m^k$ for $0 \leq k \leq K-1$, training can begin with a subset that includes all potential hard samples. 
As training progresses, samples are gradually excluded based on their associated uncertainty, following a trajectory from low to high uncertainty. 
This strategy leads to a progressive increase in the difficulty level of the discovered gap regions.

\subsubsection{Gap region labeling} \label{sec:ga-label}

In practice, we achieve the gap region labeling by assigning confused pseudo-labels to the samples contained in it. 
This process builds upon a memory bank, as shown in Fig.~\ref{fig:ov} (b). 
This memory bank manages two types of knowledge for all target samples: 
(1) predictions by the target model: $\{\boldsymbol{p}'_i\}_{i=1}^{n}$, and  
(2) predictions by the ViL model: $\{\boldsymbol{p}''_i\}_{i=1}^{n}$.

Throughout the adaptation, the predictions from the target model are updated iteratively. At the end of each training iteration, the newly predicted labels for the training
batch from the target model replace their counterparts in the prediction bank. In contrast, predictions from the ViL model are updated collectively in an epoch-wise manner, triggering
updates every $M$ iterations. This mixed-update strategy is designed to strike a balance between maintaining the stability of the customized ViL model’s guidance and capturing the task-specific dynamics inherent in the adaptation. 

Based on the stored prediction information, the creation of a pseudo-label for any samples in the gap region $\mathcal{R}_t$ involves a historical prediction fusion process as: 
\begin{equation} 
    \label{eqn:onehot-fusion}
    \begin{split}
        {\boldsymbol{\bar{p}}}_{i} = \tau~{\boldsymbol{p}'}_{i} + (1-\tau)~\boldsymbol{p}''_{i},
    \end{split}
\end{equation}
where the weight $\tau$, drawn from an Exponential distribution with parameter $\lambda$, is a crucial factor.

The rationale behind incorporating this confusion is to harness the collective knowledge of both the target model and the ViL model in order to enhance the discernment of probable category labels for each sample.
Additionally, this fusion introduces dynamic bias rectification (represented by ${\boldsymbol{p}}_{i}$) based on the guidance from the customized ViL model ($\boldsymbol{p}''_{i}$). The role of ${\boldsymbol{p}'}_{i}$ is to provide adjustments, leading us to adopt an asymmetric random weighting approach represented by $\tau$.

\subsubsection{Gap region reduction} \label{sec:ga-redu}

From the perspective of feature distribution, reducing the gap region relies on guiding the samples within it toward class-consistent, island-like clusters.
To this end, we introduce three key designs.
Specifically, {\it category attention calibration} and {\it predictive consistency} jointly encouraging this sample migration by adapting quality semantic knowledge.
Moreover, {\it gap region compression}, achieved through uncertainty minimization, further reinforces this process.
The details of these designs are presented below.

\vspace{0.1cm}
\noindent
\textbf{\textit{(a) Category attention calibration.}}
The gap region labeling fuses the semantics from the target and ViL model (see Eq.~\eqref{eqn:onehot-fusion}). 
Since the prediction errors of both models are not strictly constrained, the fused pseudo-label does not guarantee that the highest-probability category is reliable.
Nevertheless, the incorporation of multi-perspective knowledge significantly improves the likelihood that the top-N predicted categories include the ground truth.

Inspired by the above insight, we formulate a regularization to guide the target model’s focus toward these categories, referred to as {\it category attention calibration}. 
Fig.~\ref{fig:aa} illustrates this regularization.  
Specifically, we begin by identifying the {\it top-N} most probable categories using ${\boldsymbol{\bar{p}}}_{i}$. The indices of these identified categories are denoted by $\mathcal{M}_i=\{m_k\}_{k=1}^{N}$. With $\mathcal{M}_i$, the target model's logit of a target domain sample $x_i$, denoted as $\boldsymbol{l}_i$, is segregated into positive and negative category groups.
We can define this regularization as: 
\begin{equation}\label{eqn:loss_mce} 
    \begin{split} 
    \mathcal{L}_{\text{cac}} &= \min_{\theta{t}}\mathbb{E}_{{\boldsymbol{x}}_{i} \in {{\mathcal{R}}_t}}  
    \log \frac{\exp\left( a_i / \iota \right)} {\sum_{\substack{j \neq \mathcal{M}_i}}
    \exp{\left( b_i \cdot \boldsymbol{l}_{i,j}/\iota \right)}}\\
    a_i &= \prod\limits_{k=1}^{N} {\boldsymbol{l}_{i,m_k}}, ~~~
    b_i = \sum\limits_{k=1}^{N} {\boldsymbol{l}_{i,m_k}}
    \end{split}
\end{equation}
where $\boldsymbol{l}_{i, a}$ denotes the $a$-th element of $\boldsymbol{l}_{i}$, $\iota$ is the temperature parameter, and
${\mathcal{R}}_t$ is the gap region identified by Eq.~\eqref{eqn:confi-split}.

\begin{figure}[t]
    \centering
    \includegraphics[width=0.97\linewidth]{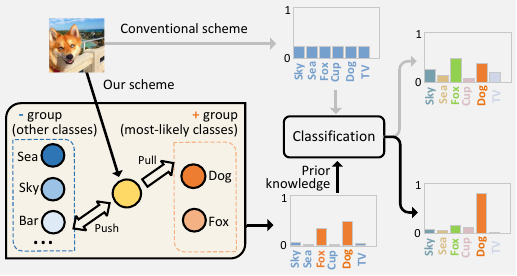}
    \caption{Illustration of the design of category attention calibration. 
    In contrast to the conventional approach that assigns equal importance to all categories (depicted by the gray line), our approach (represented by the black line) introduces additional supervision by incorporating extra knowledge about the two most likely categories.}
    \label{fig:aa}
\end{figure}


In Eq.~\eqref{eqn:loss_mce}, we note that the product operation with $a_i$ in the numerator amplifies penalties for the probability decrease on the most likely categories compared to the sum form. Similarly, the sum with $b_i$ in the denominator serves as an increasing weighting parameter to enhance suppression of values at other locations. Moreover, $a_i$ is more sensitive to changes than $b_i$ due to $\frac{{\partial a_i}}{{\partial m_k}} \propto {O}(n^{N-1})$ and $\frac{\partial b_i}{\partial m_k} \propto {O}(1)$. By combining the use of $a_i$ and $b_i$, we globally impose a calibration effect on the elements corresponding to the most likely categories within the logit $\boldsymbol{l}_i$. 
Essentially, attention is introduced to these potential categories, as illustrated in the box with a yellow background in Fig.~\ref{fig:aa}.

\vspace{0.1cm}
\noindent
\textbf{\textit{(b) Predictive consistency.}} 
For the purpose of knowledge adaptation, we incorporate the conventional predictive consistency loss as:
\begin{equation}
    \label{eqn:loss_pc}
    \begin{aligned}
        \mathcal{L}_{\text{pc}} =\min_{\theta_{t}} \left[- \mathbb{E}_{{\boldsymbol{x}}_{i} \in {{\mathcal{R}}_t}}{{\rm{I}}}\left(\theta_{t}\left({\boldsymbol{x}}_{i}\right), {\theta_{v}} \left(\boldsymbol{x}_{i}, \boldsymbol{v}*\right)\right)+\alpha \mathcal{L}_b \right], 
    \end{aligned}
\end{equation}
where
${\mathcal{R}}_t$ is the gap region, $\theta_{t}(\boldsymbol{x}_{i})$ represents the target prediction, $\theta_{v}(\boldsymbol{x}_{i},\boldsymbol{v})$ denotes the ViL prediction, and $\boldsymbol{v}$ is the prompt context learned during the initial phase of task-specific customization. The function ${{\rm{I}}}(\cdot,\cdot)$ corresponds to the mutual information function. The parameter $\alpha$ serves as a trade-off parameter, and the category balance term ${L}_{\rm{B}} ={\rm{KL}} (\left. {\bar{\boldsymbol{q}}} \right|| {\boldsymbol{\frac{1}{C}}})$ aligns with previous approaches~\citep{i0yang2021exploiting,tang2023source}, preventing solution collapse by ensuring the empirical label distribution $\bar{\boldsymbol{q}}$ matches the uniform distribution $\boldsymbol{\frac{1}{C}}$.
For the reasons elaborated in \texttt{Section~\ref{sec:tsc}}, we also employ mutual information for alignment.

\begin{algorithm}[t]
    \caption{Pseudo code of training {\modelshortname}}
    \label{alg:algorithm}
    \raggedright
    \textbf{Input}: Pre-trained source model $\theta_s$, target model $\theta_t$, ViL model $\theta_{v}$, unlabelled target domain $\mathcal{X}_t$, learnable prompt context $\boldsymbol{v}$, \#epoch $K$, \#iteration per epoch $M$.\\
    \textbf{Output}: The adapted target model $\theta_t$.\\
    \textbf{Procedure}:
    \begin{algorithmic}[1] 
        \STATE \textbf{Initialisation}: Set model $\theta_t= \theta_s$, $\boldsymbol{v}$ = "\textit{a photo of a [CLS]."}
        \FOR{$k$ = 1:$K$}
        \STATE Update ViL predictions in the prediction bank.
        \STATE ========== \textit{Step1} ==========
        \FOR{$j$ = 1:$M$} 
        \STATE Sample a batch from $\mathcal{X}_t$;
        \STATE Forward prompt $\boldsymbol{v}$, this batch $\mathcal{X}_t^b$ through $\theta_{v}$; 
        \STATE Forward $\mathcal{X}_t^b$ data through $\theta_t$; 
        \STATE Customize $\theta_{v}$ by optimizing $\mathcal{L}_{\rm{S-I}}$ (Eq. \eqref{eqn:loss_mim}) and obtain task-specific prompt context $\boldsymbol{v}^*$.
        \ENDFOR
        \STATE ========== \textit{Step2} ==========
        \STATE Compute current uncertainty margin $m^k$.
        \STATE Compute current adaptation-aware reference $\rho^k$.
        \FOR{$j$ = 1:$M$}
        \STATE Sample a batch from $\mathcal{X}_t$;
        \STATE Forward the $\boldsymbol{v}^*$, this batch ${\mathcal{X}}_t^b$ through $\theta_{v}$;
        \STATE Forward ${\mathcal{X}}_t^b$ through $\theta_t$;
        \STATE Estimate the referenced uncertainty of ${\mathcal{X}}_t^b$ (Eq.~\eqref{eqn:uncer-esti});
        \STATE Detect gap region ${\mathcal{R}}_t^b$ from ${\mathcal{X}}_t^b$ (Eq.~\eqref{eqn:confi-split});
        \STATE Label gap region (Eq.~\eqref{eqn:onehot-fusion});
        \STATE Reducing gap region by optimizing ${\mathcal{L}_{{\text{S-II}}}}$ (Eq.~\eqref{eqn:loss-ka}), updating target model $\theta_{t}$.
        \STATE Update target predictions in the prediction bank.
        \ENDFOR
        \STATE Set $\boldsymbol{v} =\boldsymbol{v}^*$. 
        \ENDFOR 
        \STATE \textbf{return} Adapted model $\theta_t$.
    \end{algorithmic}
\end{algorithm}

\vspace{0.1cm}
\noindent
\textbf{\textit{(c) Gap region compression.}} 
Beyond semantics-guided reduction via $\mathcal{L}_{\text{cac}}$ and $\mathcal{L}_{\text{pc}}$, we further promote uncertainty-oriented gap-region compression by minimizing referenced entropy, formulated as:
\begin{equation}\label{eqn:loss_uncer_min} 
    \begin{split} 
    \mathcal{L}_{\text{rc}} &= \min_{\theta_{t}}\mathbb{E}_{{\boldsymbol{x}}_{i} \in {\mathcal{R}_t}} 
    U\left({\boldsymbol{x}}_{i}\right),
    \end{split}
\end{equation}
where $m^k$ is the uncertainty margin used in the $k$-th epoch, its value is the same as the threshold in Eq.~\eqref{eqn:confi-split}.

For better clarity, \changed{Fig.~\ref{fig:gap-region} (c)} shows how $\mathcal{L}_{\rm{rc}}$ compresses the gap region. 
In Eq.~\eqref{eqn:loss_uncer_min}, minimizing the loss enables the relatively easy subset within the current gap region (indicated by orange edge) to move beyond the current referenced uncertainty contour $m^{k}$ (shown in blue), entering the low uncertainty area (outside $m^{k}$). 
So, the relatively harder samples remain in the high-uncertainty zone (inside $m^{k}$), enabling the discovery of a more challenging gap region for the $(k+1)$-th epoch.

\subsection{Objective and Model Training}
To systematically distill and leverage task-specific knowledge from the ViL model, we adopt an epoch-wise training approach for {\modelshortname}. The training process is divided into $K$ epochs, each comprising two stages aligned with the two steps in the {\modelshortname} framework (Fig.~\ref{fig:ov}). During the first stage, training is governed by the objective ${\mathcal{L}_{{\text{S-I}}}}$ (see Eq.~\eqref{eqn:loss_mim}), and in the subsequent second stage, the objective function transitions to
\begin{equation} 
    \label{eqn:loss-ka}
     \mathcal{L}{_{{\text{S-II}}}} = \mathcal{L}{_{{\text{pc}}}} + \beta \mathcal{L}{_{{\text{cac}}} + \eta \mathcal{L}{_{{\text{rc}}}}},
\end{equation}
where $\beta$ and $\eta$ are trade-off parameters. 
In Eq.~\eqref{eqn:loss-ka}, the component losses work with different functionality.
\changed{
$\mathcal{L}_{\text{pc}}$ also encourages interactive learning between the target and ViL models through mutual information. This design choice stems from the unsupervised nature of the Stage I ViL customization, which produces predictions with residual noise. Although this noise is partially mitigated, the outputs remain insufficiently reliable to serve as definitive ground truth.
}
$\mathcal{L}_{\text{rc}}$ promotes a shift of the hard data from within to outside the contour as shown in \changed{Fig.~\ref{fig:gap-region} (c)}, whilst $\mathcal{L}_{\text{pc}}$ and $\mathcal{L}_{\text{cac}}$ jointly impose the guidance of moving direction to potential categories.
For a clear view, we summarize the whole training procedure of {\modelshortname} in Algorithm~\ref{alg:algorithm}.

\begin{table*}[t]
\renewcommand\tabcolsep{11.1pt}
\renewcommand\arraystretch{1.0}
\scriptsize
\centering
\caption{Closed-set SFDA accuracies (\%) on the \textbf{Office-31} dataset. 
SF means source-data-free; the best result is highlighted in {\color{cmblu}\textbf{blue}}.}
\begin{tabular}{l c c c c c c c |c}
\toprule
Method & {ViL} & A$\to$D & A$\to$W & D$\to$A & D$\to$W & W$\to$A & W$\to$D & Avg. \\
\midrule
Source-model &\textbf{--} &79.1 &76.6 &59.9 &95.5 &61.4 &98.8 &78.6 \\
\midrule
SHOT~\citep{liang2020we}         &\xmark &93.7 &91.1 &74.2 &98.2 &74.6 &\textbf{\color{cmblu}100.} &88.6 \\
NRC~\citep{i0yang2021exploiting}          &\xmark &96.0 &90.8 &75.3 &99.0 &75.0 &\textbf{\color{cmblu}100.} &89.4 \\
HCL~\citep{huang2021model}       &\xmark &94.7 &92.5 &75.9 &98.2 &77.7 &\textbf{\color{cmblu}100.} &89.8 \\
AaD~\citep{yang2022attracting}   &\xmark &96.4 &92.1 &75.0 &99.1 &76.5 &\textbf{\color{cmblu}100.} &89.9 \\
CoWA~\citep{lee2022confidence}   &\xmark &94.4 &95.2 &76.2 &98.5 &77.6 &99.8 &90.3 \\
SCLM~\citep{tang2022sclm}        &\xmark &95.8 &90.0 &75.5 &98.9 &75.5 &99.8 &89.4 \\
ELR~\citep{yi2023source}         &\xmark &93.8 &93.3 &76.2 &98.0 &76.9 &\textbf{\color{cmblu}100.} &89.6 \\
PLUE~\citep{Litrico_2023_CVPR}   &\xmark &89.2 &88.4 &72.8 &97.1 &69.6 &97.9 &85.8 \\
TPDS~\citep{tang2023source}      &\xmark &97.1 &94.5 &75.7 &98.7 &75.5 &99.8 &90.2 \\
SF(DA)$^2$~\citep{hwang2024sfda}     &\xmark &97.1 &94.5 &75.7 &98.7 &75.5 &99.8 &90.2 \\
SiLAN~\citep{wang2025silan}      &\xmark &97.1 &95.8 &76.4 &98.9 &76.9 &\textbf{\color{cmblu}100.} &90.7 \\
UCon-SFDA~\citep{xu2025revisiting}      &\xmark &94.8 & 95.4 &98.9 &\textbf{\color{cmblu}100.} &77.1 &77.1 &90.6 \\ 
\midrule
SHOT-ViL~\citep{liang2020we}     &\cmark &95.4 &94.0 &76.8 &98.0 &78.4 &\textbf{\color{cmred}100.} &90.4 \\
GKD-ViL~\citep{tang2021model}    &\cmark &95.2 &95.7 &75.6 &98.1 &75.9 &\textbf{\color{cmred}100.} &90.1 \\
NRC-ViL~\citep{i0yang2021exploiting}      &\cmark &94.6 &94.6 &75.8 &97.7 &78.0 &99.6 &90.1 \\
AaD-ViL~\citep{yang2022attracting}&\cmark &95.6 &95.5 &75.3 &97.6 &76.4 &\textbf{\color{cmred}100.} &90.1 \\
TPDS-ViL~\citep{tang2023source}    &\cmark &95.0 &93.8 &77.3 &98.6 &77.2 &99.8 &90.3 \\
Co-learn-V~\citep{zhang2025source} &\cmark &90.4 &89.9 &\textbf{\color{cmred}92.2} &89.9 &\textbf{\color{cmred}92.2} &91.0 &90.9 \\
ViLAaD-V~\citep{tarashima2025vilaad}  &\cmark  &\textbf{\color{cmred}97.6} &93.8 &79.6 &96.7 &79.7 &\textbf{\color{cmred}100.} &91.2 \\
DIFO~\citep{tang2024sourcefree}  &\cmark &97.2 &95.5 &83.0 &97.2 &83.2 &98.8 &92.5 \\
\rowcolor{gray!40}
\textbf{\modelshortname}~(Ours)   &\cmark &97.2 &\textbf{\color{cmred}96.9} &82.9 &97.2 &83.3 &99.0 &\textbf{\color{cmred}92.8} \\
\bottomrule
\end{tabular}
\label{tab:oc}
\end{table*}

\section{Experimental Results}\label{sec:exp}
{\bf Datasets}
To evaluate {\modelshortname}, we adopt four common SFDA benchmarks, as outlined below.

\begin{itemize}
\item[(1)] \textbf{Office-31}~\citep{saenko2010adapting} is a widely recognized small-scale dataset employed in visual domain adaptation. It encompasses three distinct domains: Amazon~(A), Webcam~(W), and Dslr~(D), each capturing real-world objects situated within diverse office settings. In aggregate, the dataset comprises 4,652 images distributed across 31 categories.

\item[(2)] \textbf{Office-Home}~\citep{venkateswara2017deep} is a medium-scale dataset predominantly employed for domain adaptation. It amasses 15k images distributed across 65 categories, reflecting both working and family environments. The dataset introduces four unique domains: Artistic images~(Ar), Clip Art~(CL), Product images~(Pr), and Real-world images~(Rw).

\item[(3)] \textbf{VisDA}~\citep{peng2017visda} presents a challenging landscape for domain adaptation with 12 synthetic-to-real transfer recognition tasks. The source domain houses 152k synthetic images, whereas the target domain encompasses 55k real object images sourced from Microsoft COCO.

\item[(4)] \textbf{DomainNet-126}~\citep{peng2019moment} is another large-scale dataset. As a subset of DomainNet containing 600k images of 345 classes from 6 domains of image styles, it has 145k images from 126 classes, sampled from 4 domains, Clipart~(C), Painting~(P), Real~(R), Sketch~(S), as~\citep{saito2019semi} identify severe noisy labels in the dataset. 
\end{itemize}

\subsection{Implementation Details}
\textbf{Network structure}.
The {\modelshortname} model contains two network branches.
In the target model branch, the feature extractor consists of a deep architecture and a fully-connected layer followed by a batch-normalization layer. 
Same to the previous work~\citep{long2018conditional,xu2019larger,liang2020we,i0yang2021exploiting,roy2022uncertainty}, the deep architecture is transferred from the deep models pre-trained on ImageNet (i.e., ResNet-50 is used on {Office-31}, {Office-Home} and {DomainNet-126}, whilst ResNet-101 is adopted on {VisDA}). 
The ending classifier is a fully-connected layer with weight normalization. 
On the other hand, the ViL model branch chooses the most adopted CLIP as the implementation where the text encoder's transformer-based architecture follows modification proposed in ~\citep{radford2021learning} as the backbone. 
Regarding the image encoder, we adopt a single version corresponding to the implementation of {DIFO}~\citep{tang2024sourcefree}. 
Specifically, in {\modelshortname}, the image encoder adopts the ViT-B/32 architecture proposed in CLIP~\citep{radford2021learning}.

\textbf{Parameter setting}.
For the trade-off parameters $\alpha$, $\beta$, and $\eta$, we set the same values on all datasets. 
Specifically, $\beta$ and $\eta$ in objective $\mathcal{L}{_{{\text{S-II}}}}$ (Eq.~\eqref{eqn:loss-ka}) are 0.4 and 0.05, whilst $\alpha$ in $\mathcal{L}{_{{\text{pc}}}}$ (Eq.~\eqref{eqn:loss_pc}) is 1.0.  
The parameter of Exponential distribution $\tau$ in Eq.~\eqref{eqn:onehot-fusion} is specified to 10.0. 
The temperature parameters in Eq.~\eqref{eqn:loss_mce} are $\tau=0.1$. 
The number of the most-likely categories is set to $N=2$. 
The parameters of uncertainty margin $m^k$ is set to $(\epsilon, \gamma) = (0.01, 1.01)$.

\textbf{Training setting}.
We adopt the batch size of 64, SGD optimizer with momentum 0.9 and 15 training epochs on all datasets. 
The prompt template for initiation is the mostly used \textit{'a photo of a [CLASS].'}~\citep{radford2021learning} where [CLASS] stands for the class name. 
All experiments are conducted with PyTorch on a single \changed{GPU of NVIDIA TITAN RTX}. 

\begin{table*}[t]
\renewcommand\tabcolsep{0.2pt}
\renewcommand\arraystretch{1.0} 
\scriptsize
\centering
\caption{Closed-set SFDA result (\%) on the \textbf{Office-Home} dataset. {ViL} means whether using ViL model; the best result is highlighted in {\color{cmblu}\textbf{blue}}.}
\label{tab:oh}
\begin{tabular}{l c c c c c c c c c c c c c | c}
\toprule
Method & {ViL} 
&Ar$\to$Cl &Ar$\to$Pr &Ar$\to$Rw &Cl$\to$Ar &Cl$\to$Pr &Cl$\to$Rw &Pr$\to$Ar &Pr$\to$Cl &Pr$\to$Rw
&Rw$\to$Ar &Rw$\to$Cl &Rw$\to$Pr &Avg.\\
\midrule
Source-model &\textbf{--} &43.7 &67.0 &73.9 &49.9 &60.1 &62.5 &51.7 &40.9 &72.6 &64.2 &46.3 &78.1 &59.2 \\
\midrule
SHOT~\citep{liang2020we}         &\xmark &56.7 &77.9 &80.6 &68.0 &78.0 &79.4 &67.9 &54.5 &82.3 &74.2 &58.6 &84.5 &71.9 \\
NRC~\citep{i0yang2021exploiting}          &\xmark &57.7 &80.3 &82.0 &68.1 &79.8 &78.6 &65.3 &56.4 &83.0 &71.0 &58.6 &85.6 &72.2 \\
AaD~\citep{yang2022attracting}   &\xmark &59.3 &79.3 &82.1 &68.9 &79.8 &79.5 &67.2 &57.4 &83.1 &72.1 &58.5 &85.4 &72.7 \\
CoWA~\citep{lee2022confidence}   &\xmark &56.9 &78.4 &81.0 &69.1 &80.0 &79.9 &67.7 &57.2 &82.4 &72.8 &60.5 &84.5 &72.5 \\
SCLM~\citep{tang2022sclm}        &\xmark &58.2 &80.3 &81.5 &69.3 &79.0 &80.7 &69.0 &56.8 &82.7 &74.7 &60.6 &85.0 &73.0 \\
ELR~\citep{yi2023source}         &\xmark &58.4 &78.7 &81.5 &69.2 &79.5 &79.3 &66.3 &58.0 &82.6 &73.4 &59.8 &85.1 &72.6 \\
PLUE~\citep{Litrico_2023_CVPR}   &\xmark &49.1 &73.5 &78.2 &62.9 &73.5 &74.5 &62.2 &48.3 &78.6 &68.6 &51.8 &81.5 &66.9 \\
TPDS~\citep{tang2023source}      &\xmark &59.3 &80.3 &82.1 &70.6 &79.4 &80.9 &69.8 &56.8 &82.1 &74.5 &61.2 &85.3 &73.5 \\
SF(DA)$^2$~\citep{hwang2024sfda} &\xmark &57.3 &79.1 &81.6 &68.4 &78.2 &78.7 &69.5  &56.5 &81.9 &74.0 &59.9 &85.1 &72.5 \\
SiLAN~\citep{wang2025silan} &\xmark &58.2 &81.2 &82.5 &69.8 &78.6 &80.3 &68.4 &58.6 &82.5 &75.6 &60.8 &86.1 &73.6 \\
UCon-SFDA~\citep{xu2025revisiting}  &\xmark &65.6 &87.8 &\textbf{\color{cmred}91.0} &78.6 &79.3 &87.6 &80.2 &65.9 &87.3 &83.2 &69.1 &88.7 &80.3 \\
\midrule
SHOT-ViL~\citep{liang2020we}         &\cmark &64.7 &84.0 &86.4 &79.4 &87.3 &85.4 &78.1 &65.8 &87.6 &81.6 &65.0 &88.5 &79.5 \\
NRC-ViL~\citep{i0yang2021exploiting}          &\cmark &55.0 &83.4 &86.7 &78.3 &86.5 &85.1 &77.1 &69.9 &87.2 &90.9 &65.5 &88.4 &79.5 \\
AaD-ViL~\citep{yang2022attracting}   &\cmark &69.4 &83.0 &84.9 &76.4 &85.2 &84.6 &73.3 &71.1 &86.3 &79.0 &72.1 &88.7 &79.5 \\
TPDS-ViL~\citep{tang2023source}      &\cmark &67.7 &82.7 &86.1 &76.5 &88.0 &86.3 &73.6 &69.4 &87.1 &79.3 &70.5 &89.7 &79.7 \\
Co-learn-V~\citep{zhang2025source}   &\cmark &71.6 &87.7 &88.0 &72.3 &83.5 &86.4 &73.7 &67.0 &87.7 &77.1 &70.9 &90.8 &79.7 \\
{ViLAaD-V}~\citep{tarashima2025vilaad}   &\cmark  &66.6 &86.8 &87.5 &77.2 &88.1 &87.1 &77.1 &65.9 &88.5 &78.7 &68.0 &89.1 &80.0\\ 
DIFO~\citep{tang2024sourcefree}      &\cmark &70.6 &90.6 &88.8 &82.5 &90.6 &88.8 &\textbf{\color{cmred}80.9} &70.1 &88.9 &83.4 &70.5 &91.2 &83.1 \\
\rowcolor{gray!40}
\textbf{\modelshortname}~(Ours)       &\cmark &\textbf{\color{cmred}73.2} &\textbf{\color{cmred}91.4} &89.7 &\textbf{\color{cmred}83.6} &\textbf{\color{cmred}91.0} &\textbf{\color{cmred}89.8} &80.8 &\textbf{\color{cmred}73.8} &\textbf{\color{cmred}90.0} &\textbf{\color{cmred}84.3} &\textbf{\color{cmred}74.2} &\textbf{\color{cmred}92.3} &\textbf{\color{cmred}84.5} \\
\bottomrule
\end{tabular}
\end{table*}

\begin{table*}[t] 
\renewcommand\tabcolsep{3.0pt}
\renewcommand\arraystretch{0.95}
\scriptsize
\centering
\caption{Closed-set SFDA result (\%) on the \textbf{VisDA} dataset. {ViL} means whether using ViL model; the best result is highlighted in {\color{cmblu}\textbf{blue}}.}
\label{tab:vc}
\begin{tabular}{l c c c c c c c c c c c c c | c}
\toprule
Method & {ViL} & plane & bcycl & bus & car & horse & knife & mcycl & person & plant & sktbrd & train & truck & Per-class \\
\midrule
Source                  &--   &60.7 &21.7 &50.8 &68.5 &71.8 &5.4  &86.4 &20.2 &67.1 &43.3 &83.3 &10.6 &49.2 \\
\midrule
SHOT~\citep{liang2020we}         &\xmark &95.0 &87.4 &80.9 &57.6 &93.9 &94.1 &79.4 &80.4 &90.9 &89.8 &85.8 &57.5 &82.7 \\
NRC~\citep{i0yang2021exploiting}          &\xmark &96.8 &91.3 &82.4 &62.4 &96.2 &95.9 &86.1 &90.7 &94.8 &94.1 &90.4 &59.7 &85.9 \\
AaD~\citep{yang2022attracting}   &\xmark &97.4 &90.5 &80.8 &76.2 &97.3 &96.1 &89.8 &82.9 &95.5 &93.0 &92.0 &64.7 &88.0 \\
CoWA~\citep{lee2022confidence}   &\xmark &96.2 &89.7 &83.9 &73.8 &96.4 &97.4 &89.3 &86.8 &94.6 &92.1 &88.7 &53.8 &86.9 \\
SCLM~\citep{tang2022sclm}         &\xmark &97.1 &90.7 &85.6 &62.0 &97.3 &94.6 &81.8 &84.3 &93.6 &92.8 &88.0 &55.9 &85.3 \\
ELR~\citep{yi2023source}         &\xmark &97.1 &89.7 &82.7 &62.0 &96.2 &97.0 &87.6 &81.2 &93.7 &94.1 &90.2 &58.6 &85.8 \\
PLUE~\citep{Litrico_2023_CVPR}   &\xmark &94.4 &91.7 &89.0 &70.5 &96.6 &94.9 &92.2 &\textcolor{cmred}{\textbf{88.8}} &92.9 &95.3 &91.4 &61.6 &88.3 \\
TPDS~\citep{tang2023source}      &\xmark &97.6 &91.5 &89.7 &83.4 &97.5 &96.3 &92.2 &82.4 &\textcolor{cmred}{\textbf{96.0}} &94.1 &90.9 &40.4 &87.6 \\
SF(DA)$^2$~\citep{hwang2024sfda}  &\xmark   &96.8 &89.3 &82.9 &81.4 &96.8 &95.7 &90.4 &81.3 &95.5 &93.7 &88.5 &64.7 &88.1\\
SiLAN~\citep{wang2025silan}      &\xmark &97.5 &90.1 &85.8 &80.4 &97.6 &95.5 &92.0 &82.9 &96.5 &95.3 &92.6 &53.6 &88.3 \\
UCon-SFDA~\citep{xu2025revisiting}  &\xmark &98.4 &90.7 &88.6 &80.7 &97.9 &96.9 &93.1 &83.8 &97.6 &95.9 &92.6 &59.1 &89.6\\
\midrule
SHOT-ViL~\citep{liang2020we}     &\cmark &97.6 &89.6 &85.4 &56.5 &94.9 &93.5 &73.5 &80.4 &86.5 &95.7 &93.1 &65.5 &85.2 \\
NRC-ViL~\citep{i0yang2021exploiting}      &\cmark &97.5 &92.0 &80.3 &63.2 &95.8 &95.1 &77.5 &78.1 &96.6 &94.2 &93.4 &60.0 &85.3 \\
AaD-ViL~\citep{yang2022attracting} &\cmark &99.0 &91.4 &\textcolor{cmred}{\textbf{94.0}} &46.9 &95.6 &93.6 &75.7 &66.2 &90.7 &93.8 &93.2 &61.2 &82.6 \\
TPDS-ViL~\citep{tang2023source}  &\cmark &\textcolor{cmred}{\textbf{98.5}} &91.6 &89.2 &85.5 &97.4 &12.3 &88.7 &50.1 &88.2 &92.0 &92.8 &48.9 &78.0 \\
Co-learn-V~\citep{zhang2025source} &\cmark &97.8 &91.0 &86.5 &70.0 &96.0 &94.5 &85.0 &78.3 &93.0 &93.5 &93.8 &70.6 &87.5 \\

{ViLAaD-V}~\citep{tarashima2025vilaad}   &\cmark  &-- &-- &-- &-- &-- &-- &-- &-- &-- &-- &-- &-- &88.8\\

DIFO~\citep{tang2024sourcefree}  &\cmark &97.5 &89.0 &90.8 &\textcolor{cmred}{\textbf{83.5}} &\textcolor{cmred}{\textbf{97.8}} &97.3 &\textcolor{cmred}{\textbf{93.2}} &83.5 &95.2 &\textcolor{cmred}{\textbf{96.8}} &93.7 &65.9 &90.3 \\
\rowcolor{gray!40}
\textbf{\modelshortname}         &\cmark &98.0 &91.3 &89.2 &79.1 &96.4 &\textcolor{cmred}{\textbf{98.6}} &92.6 &88.3 &93.2 &92.9 &\textcolor{cmred}{\textbf{94.2}} &\textcolor{cmred}{\textbf{71.7}} &\textcolor{cmred}{\textbf{90.5}} \\
\bottomrule
\end{tabular}
\end{table*}

\begin{table*}[t]
\renewcommand\tabcolsep{3.5pt}   
\renewcommand\arraystretch{0.95}
\scriptsize
\centering
\caption{Closed-set SFDA result (\%) on the \textbf{DomainNet-126} dataset. {ViL} means whether using ViL model; the best result is highlighted in {\color{cmblu}\textbf{blue}}.}
\label{tab:dn}
\begin{tabular}{l c c c c c c c c c c c c c | c}
\toprule
Method & {ViL}
&C$\to$P &C$\to$R &C$\to$S &P$\to$C &P$\to$R &P$\to$S &R$\to$C &R$\to$P &R$\to$S &S$\to$C &S$\to$P &S$\to$R &Avg.\\
\midrule
Source                          &--   &44.6 &59.8 &47.5 &53.3 &75.3 &46.2 &55.3 &62.7 &46.4 &55.1 &50.7 &59.5 &54.7 \\
\midrule
SHOT~\citep{liang2020we}        &\xmark &63.5 &78.2 &59.5 &67.9 &81.3 &61.7 &67.7 &67.6 &57.8 &70.2 &64.0 &78.0 &68.1 \\
NRC~\citep{i0yang2021exploiting}         &\xmark &62.6 &77.1 &58.3 &62.9 &81.3 &60.7 &64.7 &69.4 &58.7 &69.4 &65.8 &78.7 &67.5 \\
AaD~\citep{yang2022attracting}  &\xmark &60.6 &75.2 &58.0 &58.2 &79.7 &54.6 &63.3 &67.8 &54.6 &70.6 &66.8 &76.4 &65.5 \\
CoWA~\citep{lee2022confidence}  &\xmark &64.6 &80.6 &60.6 &66.2 &79.8 &60.8 &69.0 &67.2 &60.0 &69.0 &65.8 &79.9 &68.6 \\
PLUE~\citep{Litrico_2023_CVPR}  &\xmark &59.8 &74.0 &56.0 &61.6 &78.5 &57.9 &61.6 &65.9 &53.8 &67.5 &64.3 &76.0 &64.7 \\
TPDS~\citep{tang2023source}     &\xmark &62.9 &77.1 &59.8 &65.6 &79.0 &61.5 &66.4 &67.0 &58.2 &68.6 &64.3 &75.3 &67.1 \\
SF(DA)$^2$~\citep{hwang2024sfda} &\xmark  &61.5 &75.3 &59.6 &67.8 &83.5 &61.6 &68.8 &70.5 &60.2 &68.5 &67.7 &75.1 &68.3 \\ 
SiLAN~\citep{wang2025silan}     &\xmark &61.3 &75.7 &58.9 &59.2 &77.3 &57.2 &65.3 &66.1 &56.5 &71.3 &65.6 &74.5 &65.7 \\
\midrule
SHOT-ViL~\citep{liang2020we}    &\cmark &72.8 &84.5 &72.0 &75.1 &85.2 &72.2 &78.2 &75.5 &71.6 &75.2 &73.6 &83.2 &76.6 \\
NRC-ViL~\citep{i0yang2021exploiting}     &\cmark &73.7 &84.4 &73.6 &76.1 &85.4 &73.6 &80.4 &76.9 &73.9 &74.9 &74.8 &82.1 &77.5 \\
AaD-ViL~\citep{yang2022attracting} &\cmark &72.2 &84.2 &72.0 &75.3 &85.3 &72.6 &78.2 &77.4 &73.6 &74.8 &74.7 &82.8 &76.9 \\
TPDS-ViL~\citep{tang2023source} &\cmark &72.5 &82.8 &71.8 &77.7 &84.7 &72.9 &81.9 &76.5 &72.8 &73.5 &74.2 &83.0 &77.0 \\
Co-learn-V~\citep{zhang2025source} &\cmark &\textcolor{cmred}{\textbf{78.2}} &89.2 &76.2 &\textcolor{cmred}{\textbf{82.7}} &89.2 &76.3 &\textcolor{cmred}{\textbf{82.5}} &78.0 &\textcolor{cmred}{\textbf{76.0}} &\textcolor{cmred}{\textbf{82.5}} &78.3 &89.2 &81.5 \\
DIFO~\citep{tang2024sourcefree} &\cmark &76.6 &87.2 &74.9 &80.0 &87.4 &75.6 &80.8 &77.3 &75.5 &80.5 &76.7 &87.3 &80.0 \\
\rowcolor{gray!40}
\textbf{\modelshortname}         &\cmark &78.1 &\textcolor{cmred}{\textbf{89.7}} &\textcolor{cmred}{\textbf{76.5}} &81.7 &\textcolor{cmred}{\textbf{89.6}} &\textcolor{cmred}{\textbf{77.2}} &82.3 &\textcolor{cmred}{\textbf{80.2}} &75.9 &82.3 &\textcolor{cmred}{\textbf{79.4}} &\textcolor{cmred}{\textbf{89.7}} &\textcolor{cmred}{\textbf{81.9}} \\
\bottomrule
\end{tabular}
\end{table*}

\begin{table*}[t]
\renewcommand\tabcolsep{2.0pt}
\renewcommand\arraystretch{1.05}
\scriptsize
\centering
\caption{Results (\%) of Partial-set SFDA (Top) and Open-set SFDA (Bottom) on the \textbf{Office-Home} dataset.}
\label{tab:ob-ps-os}
\begin{tabular}{l c c c c c c c c c c c c | c}
\toprule
\multirow{1}{*}{Partial-set SFDA} 
    & Ar$\to$Cl & Ar$\to$Pr & Ar$\to$Rw 
    & Cl$\to$Ar & Cl$\to$Pr & Cl$\to$Rw 
    & Pr$\to$Ar & Pr$\to$Cl & Pr$\to$Rw 
    & Rw$\to$Ar & Rw$\to$Cl & Rw$\to$Pr & Avg. \\
\cmidrule{1-14}
Source                           &45.2 &70.4 &81.0 &56.2 &60.8 &66.2 &60.9 &40.1 &76.2 &70.8 &48.5 &77.3 &62.8 \\
\midrule
SHOT~\citep{liang2020we}         &64.8 &85.2 &92.7 &76.3 &77.6 &88.8 &79.7 &64.3 &89.5 &80.6 &66.4 &85.8 &79.3 \\
HCL~\citep{huang2021model}       &65.6 &85.2 &92.7 &77.3 &76.2 &87.2 &78.2 &66.0 &89.1 &81.5 &68.4 &87.3 &79.6 \\
CoWA~\citep{lee2022confidence}   &69.6 &93.2 &92.3 &78.9 &81.3 &92.1 &79.8 &71.7 &90.0 &83.8 &72.2 &93.7 &83.2 \\
AaD~\citep{yang2022attracting}   &67.0 &83.5 &93.1 &80.5 &76.0 &87.6 &78.1 &65.6 &90.2 &83.5 &64.3 &87.3 &79.7 \\
CRS~\citep{zhang2023class}       &68.6 &85.1 &90.9 &80.1 &79.4 &86.3 &79.2 &66.1 &90.5 &82.2 &69.5 &89.3 &80.6 \\
DIFO~\citep{tang2024sourcefree}  &70.2 &91.7 &91.5 &87.8 &\textcolor{cmred}{\textbf{92.6}} &\textcolor{cmred}{\textbf{92.9}} &87.3 &70.7 &92.9 &88.5 &69.6 &91.5 &85.6 \\
\rowcolor{gray!40}
\textbf{\modelshortname}          &\textcolor{cmred}{\textbf{75.9}} &90.2 &90.9 &\textcolor{cmred}{\textbf{88.4}} &91.8 &92.5 &\textcolor{cmred}{\textbf{87.5}} &\textcolor{cmred}{\textbf{75.6}} &\textcolor{cmred}{\textbf{93.3}} &\textcolor{cmred}{\textbf{89.5}} &\textcolor{cmred}{\textbf{74.3}} &90.8 &\textcolor{cmred}{\textbf{86.7}} \\
\midrule[\heavyrulewidth]
\multirow{1}{*}{Open-set SFDA}   
    & Ar$\to$Cl & Ar$\to$Pr & Ar$\to$Rw 
    & Cl$\to$Ar & Cl$\to$Pr & Cl$\to$Rw 
    & Pr$\to$Ar & Pr$\to$Cl & Pr$\to$Rw 
    & Rw$\to$Ar & Rw$\to$Cl & Rw$\to$Pr & Avg. \\
\cmidrule{1-14}
Source                           &36.3 &54.8 &69.1 &33.8 &44.4 &49.2 &36.8 &29.2 &56.8 &51.4 &35.1 &62.3 &46.6 \\
\midrule
SHOT~\citep{liang2020we}         &64.5 &80.4 &84.7 &63.1 &75.4 &81.2 &65.3 &59.3 &83.3 &69.6 &64.6 &82.3 &72.8 \\
HCL~\citep{huang2021model}       &64.0 &78.6 &82.4 &64.5 &73.1 &80.1 &64.8 &59.8 &75.3 &78.1 &69.3 &81.5 &72.6 \\
CoWA~\citep{lee2022confidence}   &63.3 &79.2 &85.4 &67.6 &83.6 &82.0 &66.9 &56.9 &81.1 &68.5 &57.9 &85.9 &73.2 \\
AaD~\citep{yang2022attracting}   &63.7 &77.3 &80.4 &66.0 &72.6 &77.6 &69.1 &62.5 &79.8 &71.8 &62.3 &78.6 &71.8 \\
CRS~\citep{zhang2023class}       &65.2 &76.6 &80.2 &66.2 &75.3 &77.8 &70.4 &61.8 &79.3 &71.1 &61.1 &78.3 &73.2 \\
DIFO~\citep{tang2024sourcefree}  &64.5 &86.2 &87.9 &68.2 &79.3 &86.1 &67.2 &62.1 &88.3 &71.9 &65.3 &84.4 &75.9 \\
\rowcolor{gray!40}
\textbf{\modelshortname}          &63.0 &\textcolor{cmred}{\textbf{87.5}} &\textcolor{cmred}{\textbf{87.0}} &\textcolor{cmred}{\textbf{69.6}} &\textcolor{cmred}{\textbf{81.0}} &\textcolor{cmred}{\textbf{87.1}} &\textcolor{cmred}{\textbf{68.0}} &\textcolor{cmred}{\textbf{64.0}} &\textcolor{cmred}{\textbf{89.1}} &\textcolor{cmred}{\textbf{72.9}} &\textcolor{cmred}{\textbf{66.4}} &\textcolor{cmred}{\textbf{88.0}} &\textcolor{cmred}{\textbf{77.0}} \\
\bottomrule
\end{tabular}
\end{table*}

\begin{table*}[t]
    \centering
    \caption{
    \changed{Adaptation results (\%) on \textbf{Office-Home} under the Continual SFDA setting. 
    The first column of each sub-table indicates the adaptation sequence. 
    $\downarrow$ denotes the Relative Average Accuracy Drop ($RD_{Acc}$) of a test domain on the adaptation path compared with the performance when the domain is first seen.
    Formally, consider a sequence of accuracy results $a_i, a_{i+1}, \dots, a_{i+K}$ along an adaptation path. Here, $a_i$ represents the initial accuracy when a domain is first encountered, and $a_{i+k}$ denotes the test accuracy after the model has adapted to $k$ subsequent domains; $RD_{Acc} = 1/K \sum_{k=1}^{K}(a_i - a_{i+k}) / a_i \times 100\%$. 
    In each sub-table, the final value in each row marked with $\downarrow$ represents the average of the three preceding values.
    \textcolor{cmred}{\textbf{Blue}} font indicates the best mean of $RD_{Acc}$, while \textcolor{purple}{purple} font denotes cases where the {\modelshortname} is outperformed by {\modelshortnameold} within the same adaptation flow.}
    }
    \label{tab:cda-contin}
    \scriptsize
    \renewcommand\arraystretch{1.0}
    \setlength{\tabcolsep}{4.5pt} 
    \begin{tabular}{ |c|cccc| c |c|cccc| c |c|cccc| c |c|cccc| }
        \hline
        \multicolumn{23}{|c|}{DIFO} \\
        \hline 
        \multirow{2}{*}{~} & \multicolumn{4}{c|}{Test} & & \multirow{2}{*}{~} & \multicolumn{4}{c|}{Test} & & \multirow{2}{*}{~} & \multicolumn{4}{c|}{Test} & & \multirow{2}{*}{~} & \multicolumn{4}{c|}{Test} \\
        \cline{2-5} \cline{8-11} \cline{14-17} \cline{20-23}
        & Ar & Cl & Pr & Rw & & & Cl & Ar & Pr & Rw & & & Pr & Ar & Cl & Rw & & & Rw & Ar & Cl & Pr \\
        \hline
        Ar & 92.6 & 41.4 & 67.6 & 77.3 & & Cl & 93.8 & 60.1 & 60.6 & 69.3 & & Pr & 94.8 & 58.4 & 40.7 & 78.0 & & Rw & 94.3 & 68.7 & 48.7 & 79.3 \\
        Cl & 77.2 & 64.1 & 63.3 & 72.2 & & Ar & 78.0 & 78.2 & 64.2 & 76.6 & & Ar & 86.7 & 73.7 & 43.5 & 80.0 & & Ar & 89.7 & 79.0 & 45.3 & 75.9 \\
        Pr & 72.4 & 64.7 & 84.9 & 76.8 & & Pr & 74.3 & 68.3 & 85.4 & 77.5 & & Cl & 75.7 & 70.0 & 61.8 & 76.4 & & Cl & 79.4 & 74.0 & 64.3 & 70.9 \\
        Rw & 74.9 & 60.6 & 83.1 & 83.3 & & Rw & 70.1 & 66.7 & 84.5 & 83.7 & & Rw & 83.8 & 68.7 & 55.1 & 83.0 & & Pr & 82.1 & 69.6 & 60.1 & 87.2 \\
        \hline
        $\downarrow$ & 19.1 & 2.2 & 2.1 & 7.8 & & $\downarrow$ & 21.0 & 13.7 & 1.1 & 11.9 & & $\downarrow$ & 13.5 & 5.9 & 10.7 & 10.0 & & $\downarrow$ & 11.2 & 9.1 & 6.5 & 8.9 \\
        \hline
    \end{tabular}
    \begin{tabular}{ |c|cccc| c |c|cccc| c |c|cccc| c |c|cccc| }  
        \hline
        \multicolumn{23}{|c|}{DIFO++} \\
        \hline 
        \multirow{2}{*}{~} & \multicolumn{4}{c|}{Test} & & \multirow{2}{*}{~} & \multicolumn{4}{c|}{Test} & & \multirow{2}{*}{~} & \multicolumn{4}{c|}{Test} & & \multirow{2}{*}{~} & \multicolumn{4}{c|}{Test} \\
        \cline{2-5} \cline{8-11} \cline{14-17} \cline{20-23}
        & Ar & Cl & Pr & Rw & & & Cl & Ar & Pr & Rw & & & Pr & Ar & Cl & Rw & & & Rw & Ar & Cl & Pr \\
        \hline
        Ar & 92.6 & \textcolor{purple}{41.2} & \textcolor{purple}{66.9} & \textcolor{purple}{75.0} & & Cl & 93.8 & 60.5 & 61.3 & 70.4 & & Pr & 94.8 & \textcolor{purple}{57.6} & 41.4 & \textcolor{purple}{77.1} & & Rw & 94.3 & \textcolor{purple}{68.3} & \textcolor{purple}{47.1} & 77.5 \\
        Cl & 79.5 & 66.1 & \textcolor{purple}{62.8} & 72.5 & & Ar & 79.6 & \textcolor{purple}{76.5} & 64.9 & 76.8 & & Ar & 88.5 & 77.0 & 50.5 & 81.2 & & Ar & 89.7 & \textcolor{purple}{77.0} & 45.5 & \textcolor{purple}{76.3} \\
        Pr & 75.0 & 65.1 & 86.9 & 78.9 & & Pr & 77.9 & 68.3 & 87.6 & 78.7 & & Cl & 79.3 & 70.0 & 66.4 & \textcolor{purple}{75.0} & & Cl & 82.3 & \textcolor{purple}{73.4} & 65.7 & 71.4 \\
        Rw & 77.8 & 63.2 & 84.9 & 83.7 & & Rw & 71.8 & 67.3 & 84.7 & 84.4 & & Rw & 85.6 & 73.4 & 59.7 & 86.9 & & Pr & 82.6 & 70.1 & 60.4 & 86.0 \\
        \hline
        $\downarrow$ & 16.3 & 3.0 & 2.3 & \textcolor{cmred}{\textbf{7.2}} & & $\downarrow$ & 18.5 & 11.4 & 3.3 & \textcolor{cmred}{\textbf{11.1}} & & $\downarrow$ & 10.9 & 6.9 & 10.0 & \textcolor{cmred}{\textbf{9.3}} & & $\downarrow$ & 10.0 & 6.8 & 8.1 & \textcolor{cmred}{\textbf{8.3}} \\
        \hline
    \end{tabular}
\end{table*}

\begin{table}[t]
    \centering
    \caption{
    \changed{Accuracies (\%) in SF-MTDA setting.}
    }
    \label{tab:sf_mtda_comparison}
    \resizebox{0.48\textwidth}{!}{
        \setlength{\tabcolsep}{14pt} 
        \begin{tabular}{lccccc}
            \toprule
            \textbf{Method} & \textbf{A}$\rightarrow$ & \textbf{C}$\rightarrow$ & \textbf{P}$\rightarrow$ & \textbf{R}$\rightarrow$ & \textbf{Avg.} \\
            \midrule
            CoNMix & 75.6 & 81.4 & 71.4 & 73.4 & 75.4 \\
            DIFO & 82.1 & 86.7 & 78.5 & 79.1 & 81.6 \\
            \rowcolor{gray!40} \textbf{\modelshortname} & \textcolor{cmred}{\textbf{83.3}} & \textcolor{cmred}{\textbf{87.6}} & \textcolor{cmred}{\textbf{80.9}} & \textcolor{cmred}{\textbf{80.8}} & \textcolor{cmred}{\textbf{83.1}} \\
            \bottomrule
        \end{tabular}
    }
\end{table}

\begin{table*}[t]
	\setlength{\belowcaptionskip}{2pt}
	\caption{Results (\%) of CLIP and Source+CLIP on the four evaluation datasets.}
	\label{tab:zeroshot}
	\renewcommand\tabcolsep{4.2pt}
	\renewcommand\arraystretch{1.1}
	\scriptsize
	\centering
	  \begin{tabular}{l|cccc|ccccc|c|ccccc|c} 
	\toprule
	  \multirow{2}{*}{Method} & \multicolumn{4}{c|}{\textbf{Office-31}} & \multicolumn{5}{c|}{\textbf{Office-Home}} & \textbf{VisDA} & \multicolumn{5}{c|}{\textbf{DomainNet-126}} &\multirow{2}{*}{Avg.}\\
	\cmidrule(lr){2-5} \cmidrule(lr){6-10} \cmidrule(lr){11-11} \cmidrule(lr){12-16}
	& {$\to$A}  & {$\to$D}  & {$\to$W} & {$\to$Avg.} & {$\to$Ar} & {$\to$Cl} & {$\to$Pr} & {$\to$Rw} & {$\to$Avg.} & {Sy$\to$Re} & {$\to$C} & {$\to$P} & {$\to$R} & {$\to$S} & {$\to$Avg.} \\
	   \midrule
	   CLIP  &72.6 &82.7 &80.6 &79.8 &74.6 &59.8 &84.3 &85.5 &76.1 &82.9 &74.7 &73.5 &85.7 &71.2 &76.3 & 78.8\\
          Source+CLIP  &78.5 &93.0 &89.6 &87.0 &78.9 &62.5 &86.1 &87.7 &78.8 &82.0 &76.8 &73.7 &86.0 &70.8 &76.8 & 81.2\\
         \rowcolor{gray!40} \textbf{\modelshortname} & \textcolor{cmred}{\textbf{83.1}} & \textcolor{cmred}{\textbf{98.1}} & \textcolor{cmred}{\textbf{97.1}} & \textcolor{cmred}{\textbf{92.8}} & \textcolor{cmred}{\textbf{82.9}} & \textcolor{cmred}{\textbf{73.8}} & \textcolor{cmred}{\textbf{91.6}} & \textcolor{cmred}{\textbf{89.8}} & \textcolor{cmred}{\textbf{84.5}} & \textcolor{cmred}{\textbf{90.5}} & \textcolor{cmred}{\textbf{82.1}} & \textcolor{cmred}{\textbf{79.2}} & \textcolor{cmred}{\textbf{89.7}} & \textcolor{cmred}{\textbf{76.5}} & \textcolor{cmred}{\textbf{81.9}} & \textcolor{cmred}{\textbf{87.4}} \\
		\bottomrule
	\end{tabular}
\end{table*}

\subsection{Competitors}
This paper validates the efficacy of our {\modelshortname} method through comparisons with \changed{24} alternative methods, categorized into three distinct groups as outlined below.

\begin{itemize}
\item[(1)] The first group contains Source (the source model's results), CLIP~\citep{radford2021learning}, and Source+CLIP, where Source+CLIP directly averages the results of the source model and the CLIP model. 

\item[(2)] The second group includes 13 state-of-the-art {SFDA} models: 
SHOT \citep{liang2020we}, 
NRC \citep{i0yang2021exploiting}, 
HCL \citep{huang2021model}, 
AaD \citep{yang2022attracting}, 
CoWA \citep{lee2022confidence}, 
SCLM \citep{tang2022sclm}, 
ELR \citep{yi2023source}, 
CRS \citep{zhang2023class}, 
PLUE \citep{Litrico_2023_CVPR}, 
TPDS \citep{tang2023source}, 
SF(DA)$^2$~\citep{hwang2024sfda}, 
SiLAN~\citep{wang2025silan}, and
UCon-SFDA~\citep{xu2025revisiting}.

\item[(3)] The third group includes 8 state-of-the-art {SFDA} models integrated the same CLIP model: 
DIFO~\citep{tang2024sourcefree},
{Co-learn-V}~\citep{zhang2025source},
{ViLAaD-V}~\citep{tarashima2025vilaad}, 
SHOT-ViL, 
GKD-ViL, 
NRC-ViL, 
AaD-ViL, 
and TPDS-ViL. 

Among the third group, the comparisons with the suffix of ``-ViL" are ViL versions of existing SFDA methods. 
In those updated versions, the CLIP's visual backbone (ViT B/32) as the feature extractor, and theirs training process follows the original way. 
As such, all the compared methods benefit from the CLIP pre-trained knowledge in a design-generic manner, achieving a fair comparison. 

\end{itemize}

\subsection{Comparison on Closed-set SFDA setting} 
The results are listed in Tab.~\ref{tab:oc}--Tab.~\ref{tab:dn}. 
{\modelshortname} surpasses the previous best method SiLAN (on Office-31 and Office-Home), PLUE (on VisDA), and GKD (on DomainNet-126) by {\bf 2.1}\%, {\bf 10.9}\%, {\bf 2.2}\%, and {\bf 13.2}\% in average accuracy, respectively. 
{\modelshortname} also improves {\modelshortnameold} by {\bf 0.3}\% on Office-31, {\bf 1.4}\% on Office-Home, {\bf 0.2}\% on VisDA, and {\bf 1.9}\% on DomainNet-126.

In particular, we observe over {\modelshortnameold} and {\modelshortname} that the datasets with more categories (Office-Home, DomainNet-126) benefit more than those with fewer categories (Office-31, VisDA). 
This result confirms the motivation of {\modelshortname}:
As the increasing number of categories tends to complicate the task specificity, the uniform attention on all data hinders {\modelshortnameold} from extracting useful information. 
This can be mitigated by the proposed gap region-driven transfer strategy.

In order to isolate the influence of model design, we compare with previous methods integrated with the same ViL (name ending with ``-ViL"). 
{\modelshortname} still presents significant advantage, achieving at least gains of {\bf 2.4}\%, {\bf 4.8}\%, {\bf 2.3}\% and {\bf 4.4}\% on Office-31, Office-Home, VisDA, and DomainNet-126, respectively.
This confirms the superiority of our design.

\subsection{Comparison on Partial-set and Open-set settings}

Based on the results listed in Tab.~\ref{tab:ob-ps-os}, {\modelshortname} demonstrates significant performance improvements across adaptation settings. 
Compared with DIFO, {\modelshortname} outperforms by {\bf 1.1}\% on both extending SFDA settings. 
Additionally, we observe that in the Partial-set SFDA setting (Top), {\modelshortname} achieves the best performance on 7 of 12 tasks, whereas in the more challenging Open-set SFDA setting (Bottom), it performs best on all transfer tasks as well, except for task Ar$\to$Cl. 

\changed{
\subsection{Continual SFDA and Source-free Multi-target DA} \label{sec:sfmtda-csfda}
In this part, we evaluate {\modelshortname} in 
the continual SFDA and source-free multi-target DA (SF-MTDA) setting, with experiment configurations in \citep{continualwang2022continual} and \citep{mtdaKumar2023WACV}, respectively. 
To evaluate the effect of {\modelshortname}, {\modelshortnameold} is selected as a comparison, while the comparison method CoNMix \citep{mtdaKumar2023WACV} serves for the SF-MTDA setting. 
}

\changed{
Based on the continual SFDA results in Table~\ref{tab:cda-contin}, we present two key observations.
First, across all adaptation trajectories, {\modelshortname} consistently outperforms {\modelshortnameold} in 51 out of 64 cases (exceptions are marked in purple). This demonstrates that the relative advantage of {\modelshortname} remains robust throughout the continuous transfer process.
Second, {\modelshortname} exhibits superior resistance to catastrophic forgetting compared to {\modelshortnameold}.  This is evidenced by the performance in the most challenging scenarios (the second column of each sub-table) and its lower mean Relative Accuracy Drop ($RD_{Acc}$) for {\modelshortname}, as shown in the final average of each $\downarrow$ row.
Additionally, Table~\ref{tab:sf_mtda_comparison} shows that {\modelshortname} surpasses both {\modelshortnameold} and the conventional CoNMix by at least {\bf 2.5}\%, further confirming its effectiveness in the SF-MTDA setting.
}

\begin{figure}[t]
    \centering
    \includegraphics[width=0.75\linewidth]{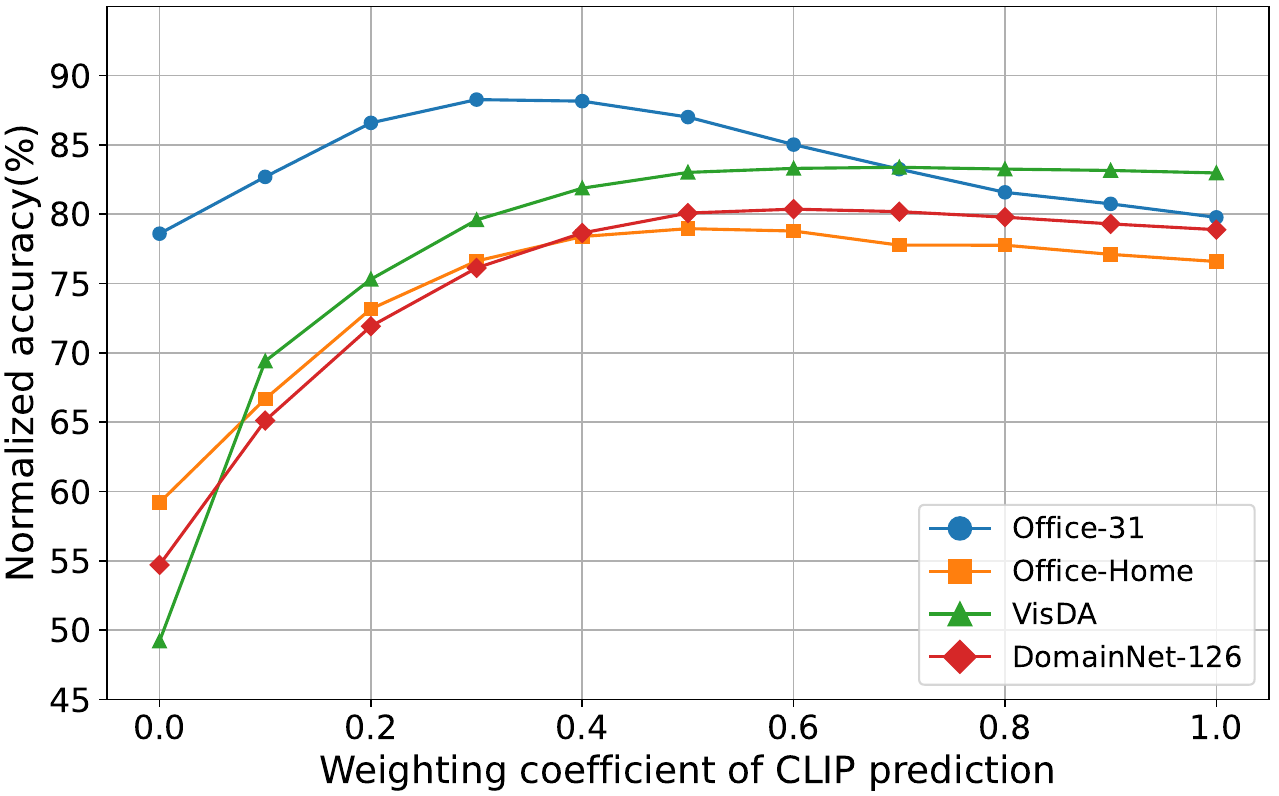}
    \caption{
     The performance of the scheme directly weighting the source model and CLIP. 
     All results are normalized by corresponding {\modelshortname} results for a clear view. 
    } 
    \label{fig:sc-weight}
\end{figure}

\subsection{Comparison to the original CLIP model}
The original CLIP model can conduct general image classification.  
We carry out a quantitative comparison between {\modelshortname}
and CLIP,
averaging the adaptation results of {\modelshortname} grouped by the target domain name.

As presented in the bottom of Tab.~\ref{tab:zeroshot}, {\modelshortname} outperforms CLIP on all tasks. 
On average accuracy, {\modelshortname} increases the performance by {\bf 13.0}\%, {\bf 8.4}\%, {\bf 7.6}\% and {\bf 5.6}\% in Office-31, Office-Home, VisDA and DomainNet-126, respectively. 
The result shows that 
the original CLIP model cannot fully excel to the target domain, and task-specific customization is needed.

Interestingly, compared with CLIP, except for VisDA with a tiny gap of {\bf 0.9}\%, Source+CLIP averagely improve by {\bf 7.2}\% at most on the other datasets. 
Meanwhile, Source+CLIP is beaten by {\modelshortname} with an increase of {\bf 6.2}\%.    
These results imply that directly weighting the source model and CLIP is an intuitive knowledge adaptation scheme, but it is hard to perform adaptation deeply.  
Considering Source+CLIP is an average version, we conduct a comprehensive comparison with the weighting strategy where the weighting coefficient of CLIP prediction varies from 0.0 to 1.0. 
Here, we conduct this experiment based on more challenging CLIP due to its large performance gap with Source (see the first row in \texttt{Tab.\ref{tab:oc}$\sim$\ref{tab:dn}}).  
For a clear view, all weighted accuracies are normalized by the corresponding {\modelshortname} accuracies, respectively.
As shown in Fig.~\ref{fig:sc-weight}, no result can exceed the value of 1.0.
This indicates that \textit{weighting the source model and CLIP in a zero-shot manner cannot obtain desirable task-specific fusion, and a carefully designed distilling is necessary.}

\begin{table}[t]
    \centering
    \caption{
    \changed{Empirical evidence on the error-proneness of pseudo-label based supervision.}
    }
    \label{tab:pseudo_quality}
    \renewcommand\arraystretch{1.1}
    \resizebox{0.48\textwidth}{!}{
        \setlength{\tabcolsep}{2pt} 
        \begin{tabular}{lcccccc}
            \toprule
            \multirow{2}{*}{\textbf{Method}} 
            & \multicolumn{2}{c}{\textbf{Ar$\rightarrow$Cl}} 
            & \multicolumn{2}{c}{\textbf{Cl$\rightarrow$Ar}} 
            & \multicolumn{2}{c}{\textbf{Avg.}} \\
            \cmidrule(lr){2-3} \cmidrule(lr){4-5} \cmidrule(lr){6-7}
            & \textbf{PL-Acc} & \textbf{HC-Acc}
            & \textbf{PL-Acc} & \textbf{HC-Acc}
            & \textbf{PL-Acc} & \textbf{HC-Acc} \\
            \midrule
            Source-only & 42.8 & 44.5 & 49.6 & 47.1 & 46.2 & 45.8 \\
            SHOT        & 55.2 & 57.3 & 68.9 & 72.2 & 62.1 & 64.7 \\
            TPDS        & 57.9 & 63.3 & 70.3 & 78.7 & 64.1 & 71.0 \\
            DIFO        & 64.6 & 80.6 & 77.7 & 81.3 & 71.2 & 80.9 \\
            \rowcolor{gray!40} \textbf{\modelshortname}      & \textcolor{cmred}{\textbf{66.0}} & \textcolor{cmred}{\textbf{84.7}} & \textcolor{cmred}{\textbf{80.1}} & \textcolor{cmred}{\textbf{89.7}} & \textcolor{cmred}{\textbf{73.1}} & \textcolor{cmred}{\textbf{87.2}} \\
            \bottomrule
        \end{tabular}
    }
\end{table}


\changed{
\subsection{Empirical evidence on the error-proneness of conventional methods} \label{sec:err-prone}
Due to the limited knowledge inherent in task-specific source models and unlabeled target data, the resulting pseudo-labels and auxiliary supervision are often contaminated with significant noise. Consequently, adaptation frameworks relying solely on these cues are inevitably error-prone. This challenge motivates us to incorporate a ViL model to bolster the adaptation process with more reliable, cross-modal knowledge.
}

\changed{
To validate our core motivation: 'conventional methods are error-prone', we conducted a comparative study on the symmetric transfer tasks Cl$\to$Ar and Ar$\to$Cl in Office-Home. The evaluation is based on two primary metrics: (1) pseudo-label accuracy (PL-Acc) and (2) the accuracy within the high-confidence subset (HC-Acc).
}

\changed{
The experimental results are summarized in Tab.~\ref{tab:pseudo_quality}. Compared to {\modelshortnameold} and {\modelshortname}, baseline methods such as SHOT and TPDS, which do not leverage ViL knowledge, exhibit substantially higher levels of predictive noise across the target domain, as evidenced by their lower PL-Acc. More compellingly, even within their high-confidence subsets, these traditional methods yield a notably lower HC-Acc than our proposed frameworks. These findings suggest that conventional approaches are inherently more prone to miscalibration and error. Furthermore, the performance gap indicates that integrating an external ViL model can effectively transcend the epistemic boundaries of task-specific source models and unlabeled target data, providing more robust guidance for adaptation.
}


\subsection{Task-Specific Knowledge Adaptation Analysis} \label{sec:tskaa}
In this part, we give a feature space shift analysis using the measure of MMD (Maximum Mean
Discrepancy) distance~\citep{MMDLoss2019} to verify whether the proposed method ensures a task-specific knowledge adaptation. 

\begin{figure}[t]
    \centering
    {\includegraphics[width=0.48\linewidth,height=0.46\linewidth]{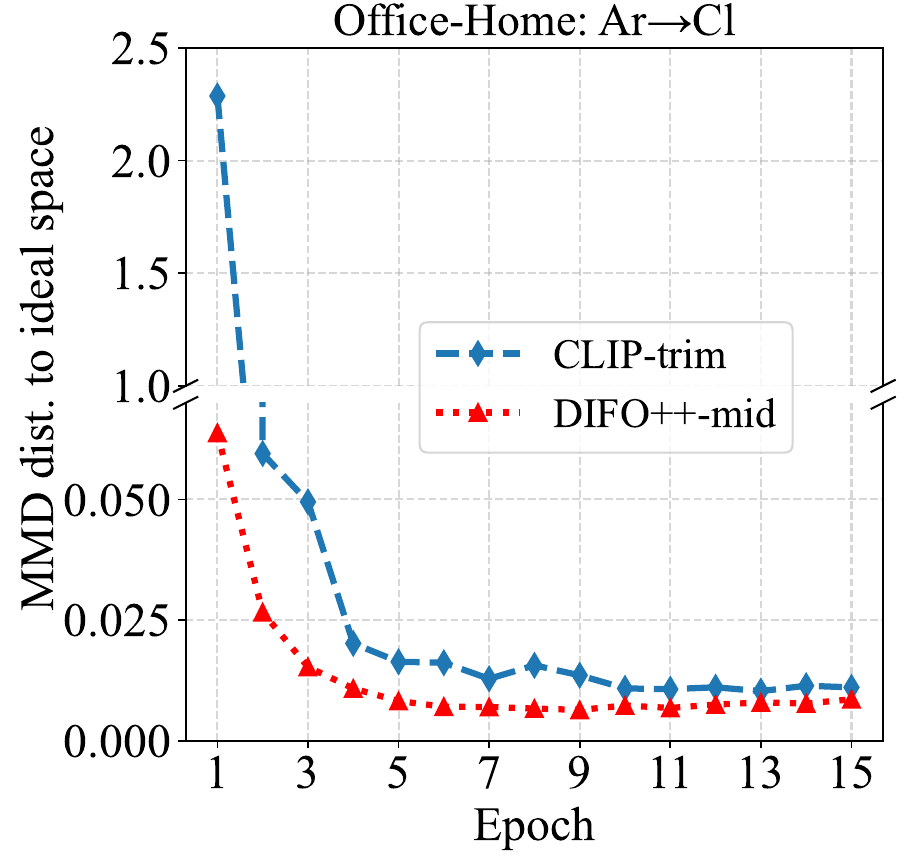}}
    {\includegraphics[width=0.48\linewidth,height=0.46\linewidth]{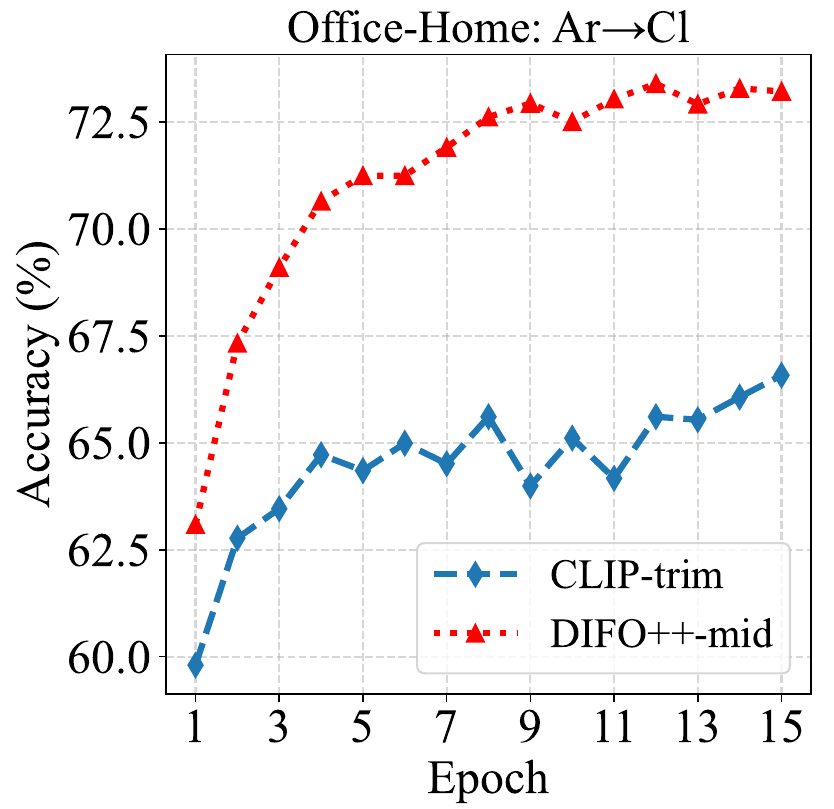}}
    \caption{
     The evolving dynamics of MMD distance during adaptation of Ar$\to$Cl in Office-Home. 
     {\bf Left} and {\bf Right} present the varying curves of MMD distance and accuracy, respectively. 
    } 
    \label{fig:dis-analy}
\end{figure}

In this experiment, we first train a domain-invariant Oracle model over all Office-Home data with real labels, and use the logits to express the ideal task-specific space $\mathcal{O}$. 
After that, an analysis is conducted on the transfer task Ar$\to$Cl. 
During this adaptation, there are $T$ (epoch number) intermediate target models and customized CLIP models. 
We feedforward the target data through each intermediate model and take the logits as a space. 
Thus, we obtain $T$ intermediate target feature spaces $\{\mathcal{U}_k\}_{k=1}^{T}$ and $T$ intermediate customized CLIP feature spaces $\{\mathcal{V}\}_{k=1}^{T}$. 
Within this context, these intermediate spaces can depict the task-specific distillation to $\mathcal{O}$. 

In the left of Fig.~\ref{fig:dis-analy}, we give the MMD distance change curve of $\{\mathcal{U}_k\}_{k=1}^{T}$ (in red, termed {\modelshortname}-mid) and $\{\mathcal{V}\}_{k=1}^{T}$ (in blue, termed CLIP-trim), taking $\mathcal{O}$ as the original space. 
It is seen that at early epochs (1$\sim$4), DIFO++-mid and CLIP-trim sharply decrease and then maintain a gradual decrease in the following epochs. 
Meanwhile, this change is consistent with the accuracy varying shown in the right of Fig.~\ref{fig:dis-analy}.

These results indicate that our {\modelshortname} encourages task-specific knowledge adaptation due to the convergence on the ideal task-specific space. 
Besides, we observe two details. 
First, after epoch 1, CLIP-trim's distance reduces by {\bf 2.2}, which is {\bf 59.5} times of DIFO++-mid's decrease of {\bf 0.037}. 
This is because CLIP represents a heterogeneous space of vision-language, much different from the vision space $\mathcal{O}$. 
Furthermore, the large decrease in distance confirms the effect of customization. 
Second, the synchronized distance reductions of CLIP-trim and DIFO++-mid indicate that the interaction between the target model and CLIP is a crucial design for task-specific distillation.

\subsection{Ablation Study}
\changed{{\bf Effect of designs in {\modelshortname}.}
In this part, we evaluate in \changed{five} aspects.}  
(1) Effect of the proposed objectives: $\mathcal{L}_{\rm{S-I}}$, ones in $\mathcal{L}_{\rm{S-II}}$ including  $\mathcal{L}_{\rm{cac}}$, $\mathcal{L}_{\rm{pc}}$, and $\mathcal{L}_{\rm{rc}}$, 
(2) effect of optimization of mutual information, 
\changed{(3) effect of uncertainty-aware sampling,}
(4) effect of task-specific customization, and
\changed{(5) effect of weighting-based gap-region labeling}.

\begin{table}[t]
    \centering
    \renewcommand\tabcolsep{1.5pt}
    \renewcommand\arraystretch{0.9} 
    \scriptsize
    \caption{
    \changed{Ablation study in Closed-set setting (\%).}
    }
    \label{tab:ab_loss}
    \begin{tabular}{l|cccc|ccc|c}
        \toprule
        \multirow{2}{*}{\#} & \multirow{2}{*}{$\mathcal{L}_{\text{S-I}}$} & \multicolumn{3}{c|}{$\mathcal{L}_{\text{S-II}}$} & \multirow{2}{*}{\bf Office-31} & \multirow{2}{*}{\bf Office-Home} & \multirow{2}{*}{\bf VisDA} & \multirow{2}{*}{\bf Avg.} \\
        \cmidrule(lr){3-5}
        & & {$\mathcal{L}_{\rm{cac}}$} & {$\mathcal{L}_{\rm{pc}}$} &{$\mathcal{L}_{\rm{rc}}$} & & & & \\
        \midrule
        1  & \xmark & \xmark & \xmark & \xmark & 78.6 & 59.2 & 49.2 & 62.3 \\
        2  & \cmark & \xmark & \xmark & \xmark & 82.4 & 77.4 & 84.4 & 81.4 \\
        3  & \xmark & \cmark & \xmark & \xmark & 82.1 & 76.5 & 88.6 & 82.4 \\
        4  & \xmark & \xmark & \cmark & \xmark & 90.7 & 77.6 & 88.6 & 85.6 \\
        5  & \xmark & \xmark & \xmark & \cmark & 76.7 & 60.1 & 54.1 & 63.6 \\
        \midrule
        6  & \cmark & \cmark & \cmark & \xmark & 92.5 & 83.1 & 90.3 & 88.6 \\
        7  & \cmark & \cmark & \xmark & \cmark & 86.5 & 68.7 & 90.0 & 81.7 \\
        8  & \cmark & \xmark & \cmark & \cmark & 83.4 & 70.1 & 88.6 & 80.7 \\
        9  & \xmark & \cmark & \cmark & \cmark & 91.3 & 82.1 & 89.8 & 87.7 \\
        10 & \cmark & \cmark & \cmark & \cmark &\textcolor{cmred}{\textbf{92.8}} & \textcolor{cmred}{\textbf{84.5}} & \textcolor{cmred}{\textbf{90.5}} & \textcolor{cmred}{\textbf{89.3}} \\
        \midrule
        11 & \multicolumn{4}{l}{\textbf{\modelshortname} w/ KL}   \vline & 87.5 & 75.4 & 85.4 & 82.8 \\
        12 & \multicolumn{4}{l|}{\textbf{\modelshortname} w/ RS}   \vline & 90.9 & 83.3 & 89.4 & 87.8 \\
        13 & \multicolumn{4}{l|}{\textbf{\modelshortname} w/ ${p}'_{i}$}  \vline & 89.0 & 73.7 & 87.1 & 83.3 \\
        14 & \multicolumn{4}{l|}{\textbf{\modelshortname} w/ ${p}''_{i}$} \vline & 90.6 & 83.5 & 88.7 & 87.6 \\
        \midrule
        15 & \multicolumn{4}{l}{\textbf{\modelshortname} w/ ENT}  \vline & 91.9 & 83.0 & 88.9 & 87.9 \\
         \bottomrule
    \end{tabular} 
\end{table}

\begin{figure}[t]
    \centering
    \includegraphics[width=0.7\linewidth]{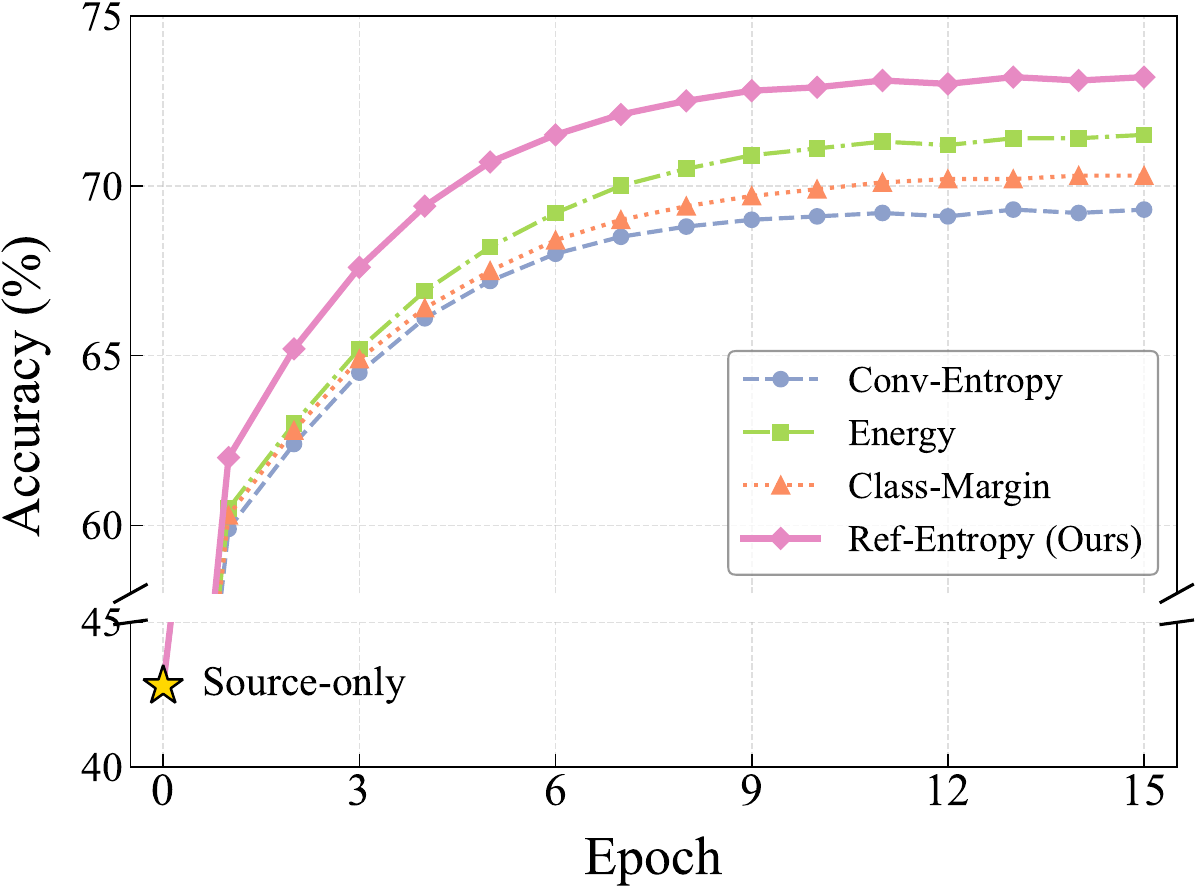}
    \caption{
    \changed{Convergence curves of different uncertainty metrics on the Ar$\rightarrow$Cl task in Office-Home.}
    }
    \label{fig:metric}  
\end{figure}

In the first issue, we isolate the effect of each component by introducing or removing it singly.
As shown in the first 5 rows of Tab.~\ref{tab:ab_loss}, there are performance improvements to some extent compared with the baseline of the source model (row 1) when the four losses work alone. 
In particular, the {\bf 1.3}\% gain of $\mathcal{L}_{\rm{rc}}$ (row 5) confirms the rationale of identifying confident data 
Additionally, removing any loss component leads to performance deterioration (rows 6--11).
When $\mathcal{L}_{\rm{cac}}$ or $\mathcal{L}_{\rm{pc}}$ are unavailable, the average accuracy drops the most (by about {\bf 9}\%).
The phenomenon is consistent with the cases in rows 3-4: as $\mathcal{L}_{\rm{cac}}$ or $\mathcal{L}_{\rm{pc}}$ work alone, the accuracy significantly improves. 
This observation confirms the effect of our designs: encouraging category attention calibration, paired with predictive consistency.

For the other \changed{four} issues, we propose three variations of {\modelshortname} to accomplish effect evaluation. 
The first one is {\modelshortname} w/ KL, where the mutual information maximization loss in $\mathcal{L}_{\rm{S-I}}$ and $\mathcal{L}_{\rm{pc}}$ are replaced by KL divergence loss. 
\changed{The second one is {\modelshortname} w/ RS, where referenced entropy-based hard data sampling is changed to a random selection.}
\changed{The third one} is the same as the loss ablation case in row 9, where the prompt learning-based customization for CLIP is canceled, and the input prompt is set to a fixed template ("a photo of a [CLS]"). 
\changed{
The last one includes {\modelshortname} w/ ${p}'_{i}$ and {\modelshortname} w/ ${p}''_{i}$, which substitute the weighted combination (Eq.~\eqref{eqn:onehot-fusion}) with either the individual target prediction $p'_i$ or the ViL prediction $p''_i$.} 
As shown in \changed{rows 11$\sim$14} of Tab.~\ref{tab:ab_loss}, {\modelshortname} outperforms all three comparisons, with an average improvement of at least \changed{\textbf{1.5\%}}, confirming the effect of adopting mutual information optimization, task-specific customization, and referenced uncertainty, respectively.

\begin{figure}[t]
    \centering
    \includegraphics[width=0.9\linewidth]{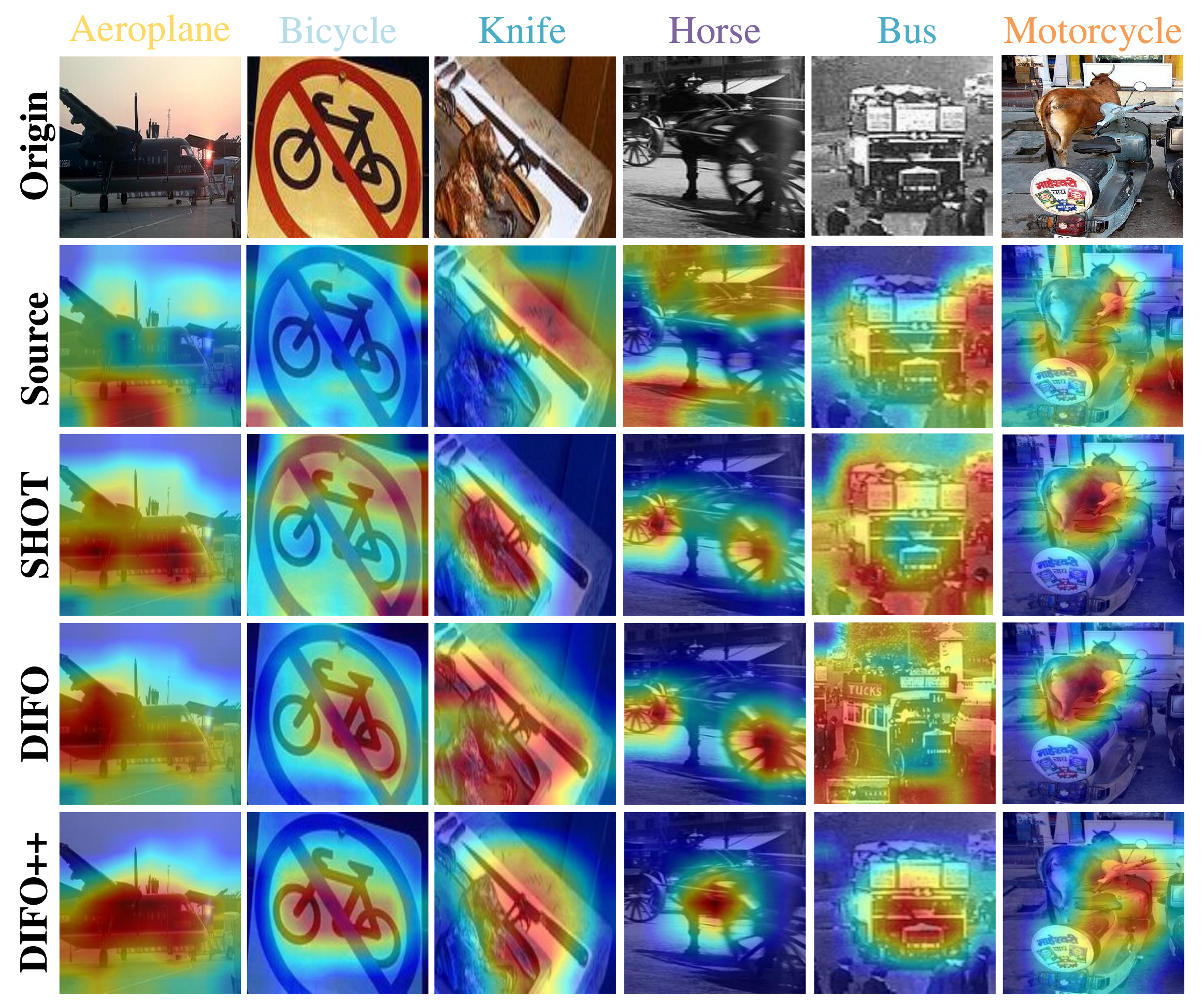}
    \caption{Visualization of six classes from VisDA.}
    \label{fig:cam-dys}
\end{figure}

\begin{figure*}[t]
    \centering
    {\includegraphics[width=0.13\linewidth,height=0.135\linewidth]{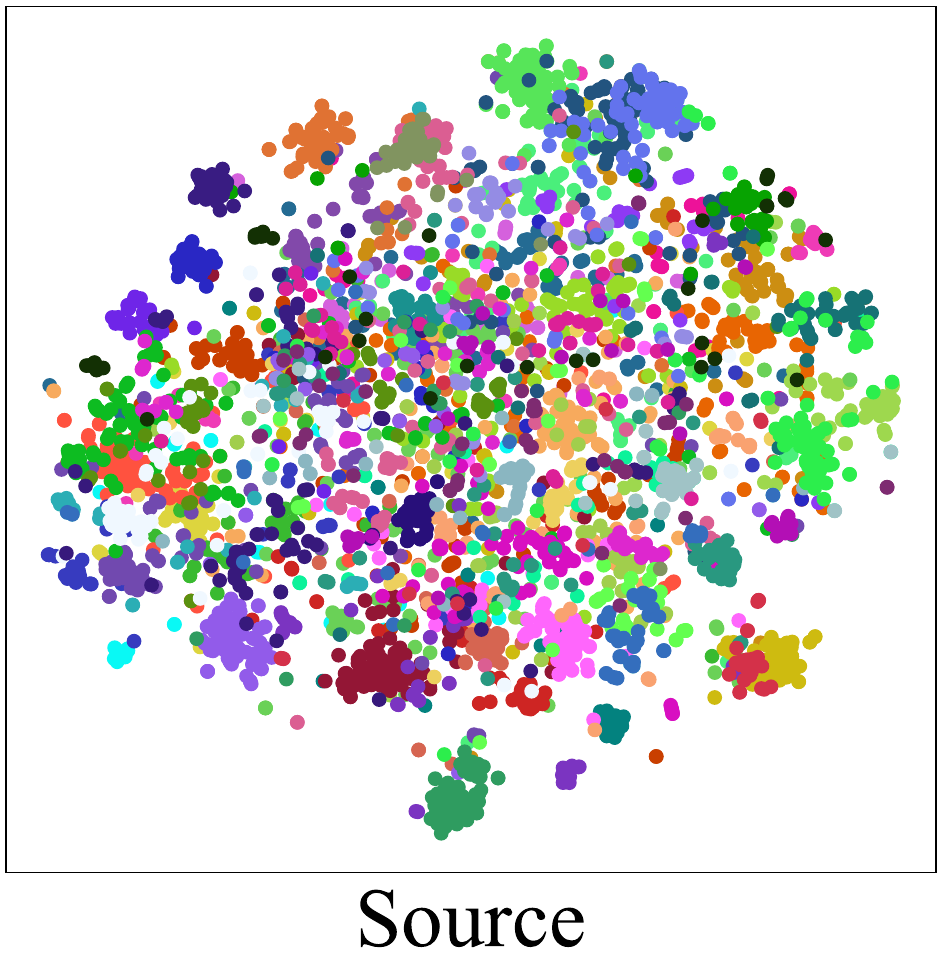}}
    {\includegraphics[width=0.13\linewidth,height=0.135\linewidth]{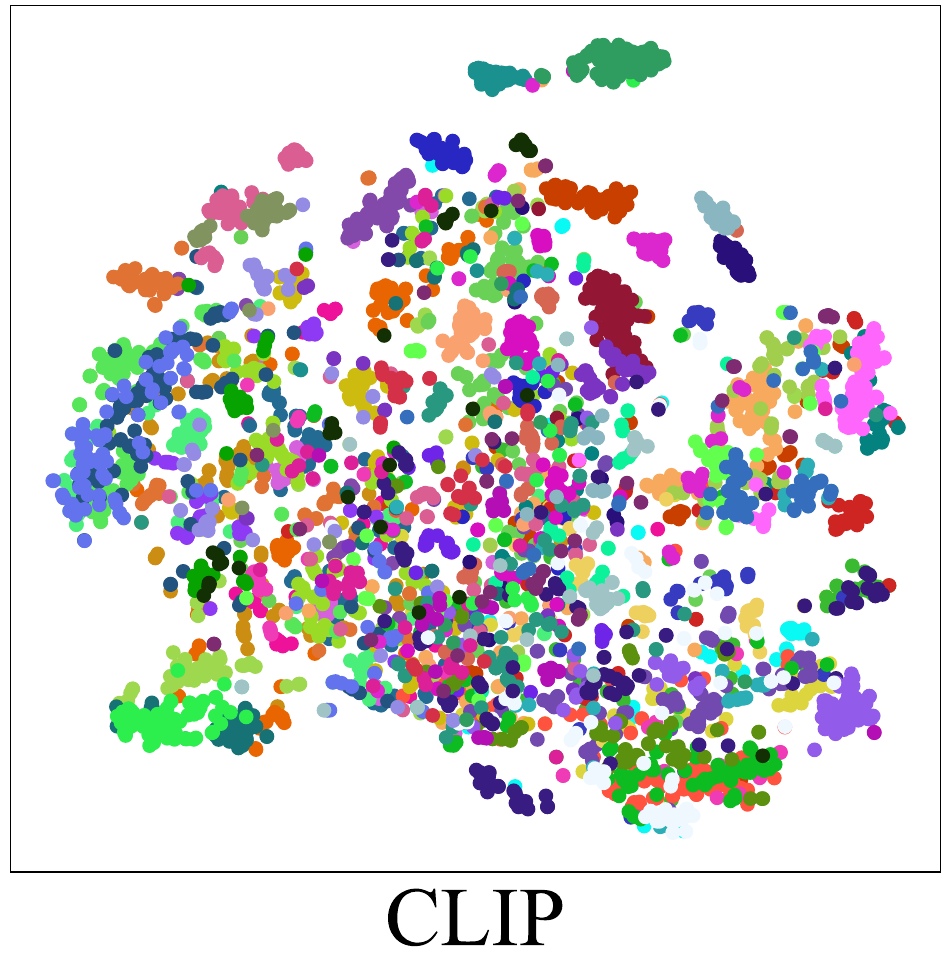}}
    {\includegraphics[width=0.13\linewidth,height=0.135\linewidth]{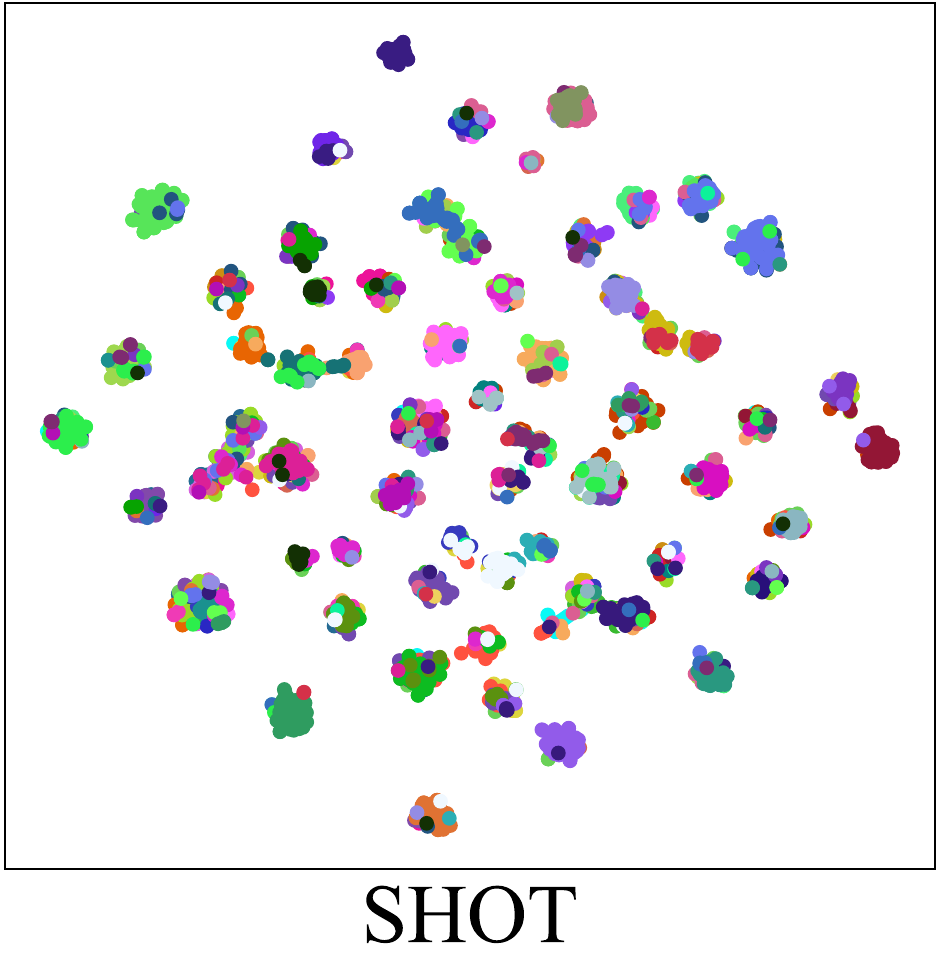}}
    {\includegraphics[width=0.13\linewidth,height=0.135\linewidth]{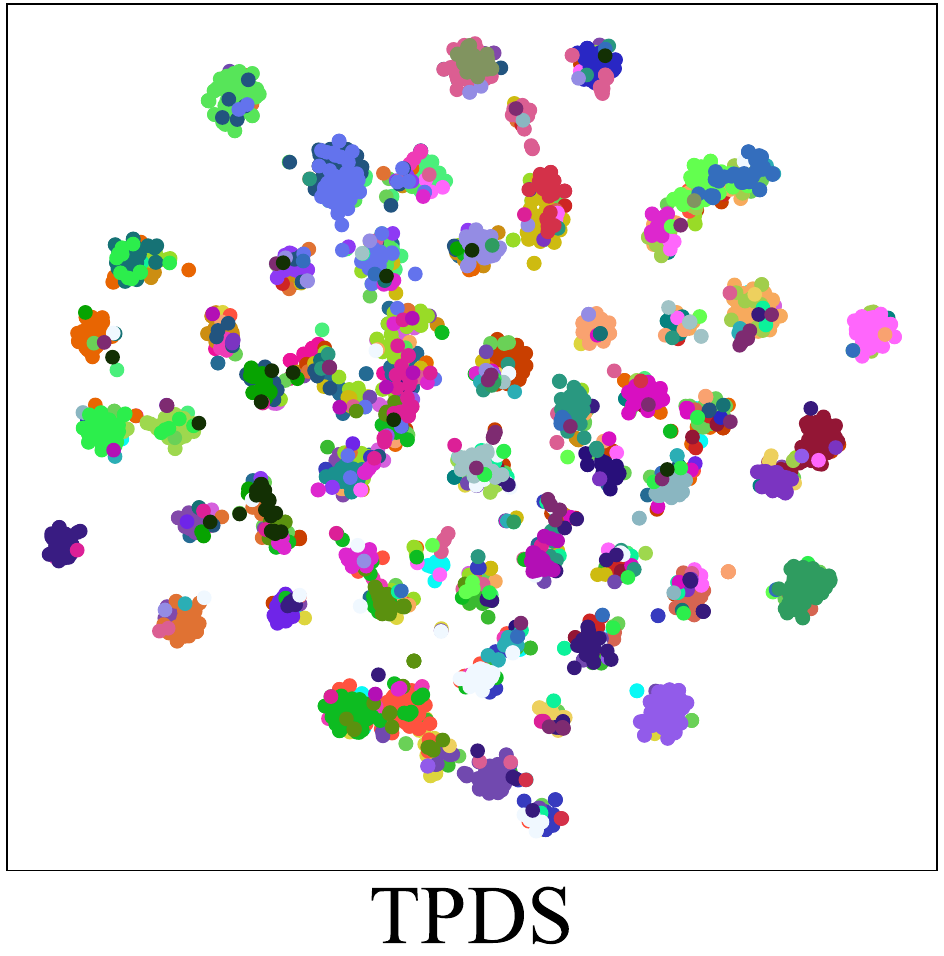}}
    {\includegraphics[width=0.13\linewidth,height=0.135\linewidth]{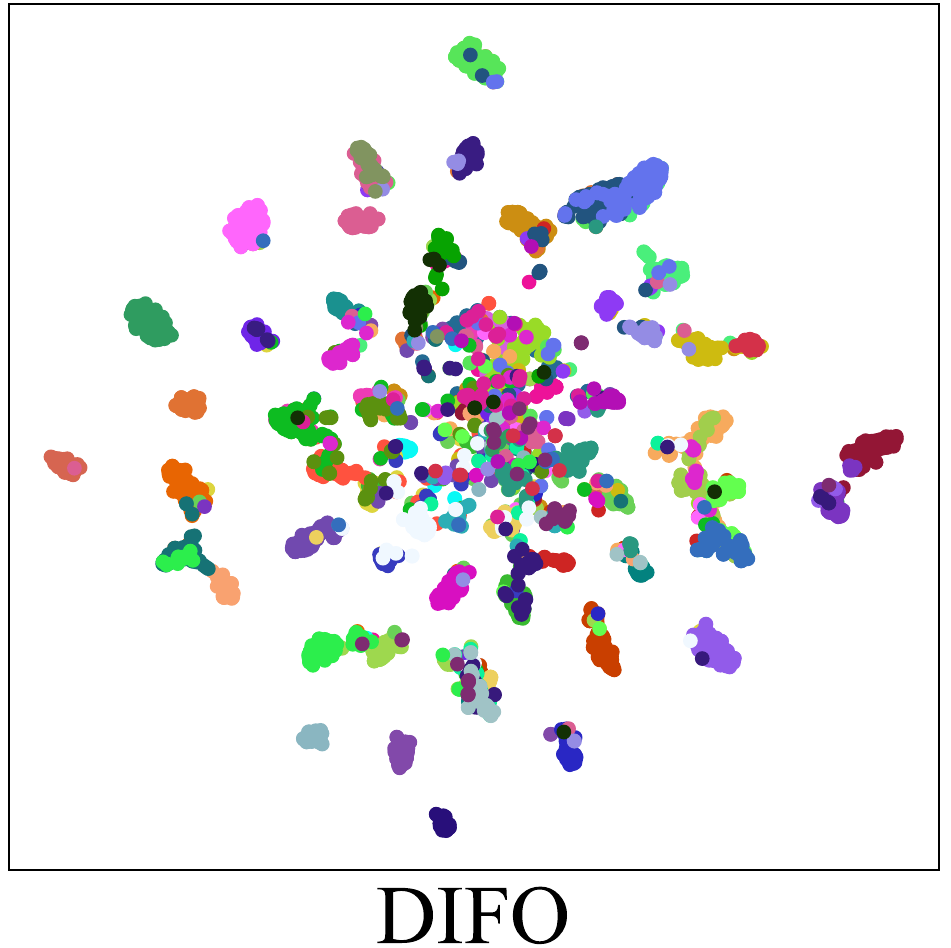}}
    {\includegraphics[width=0.13\linewidth,height=0.135\linewidth]{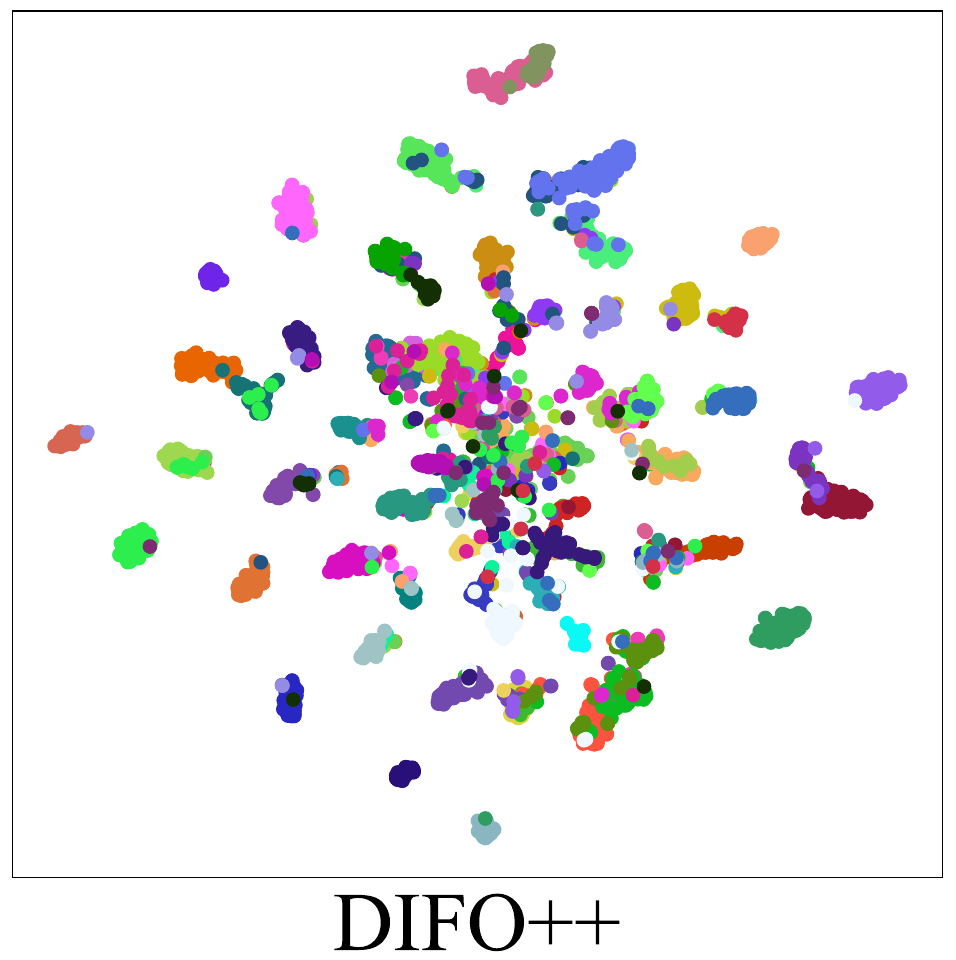}}
    {\includegraphics[width=0.13\linewidth,height=0.135\linewidth]{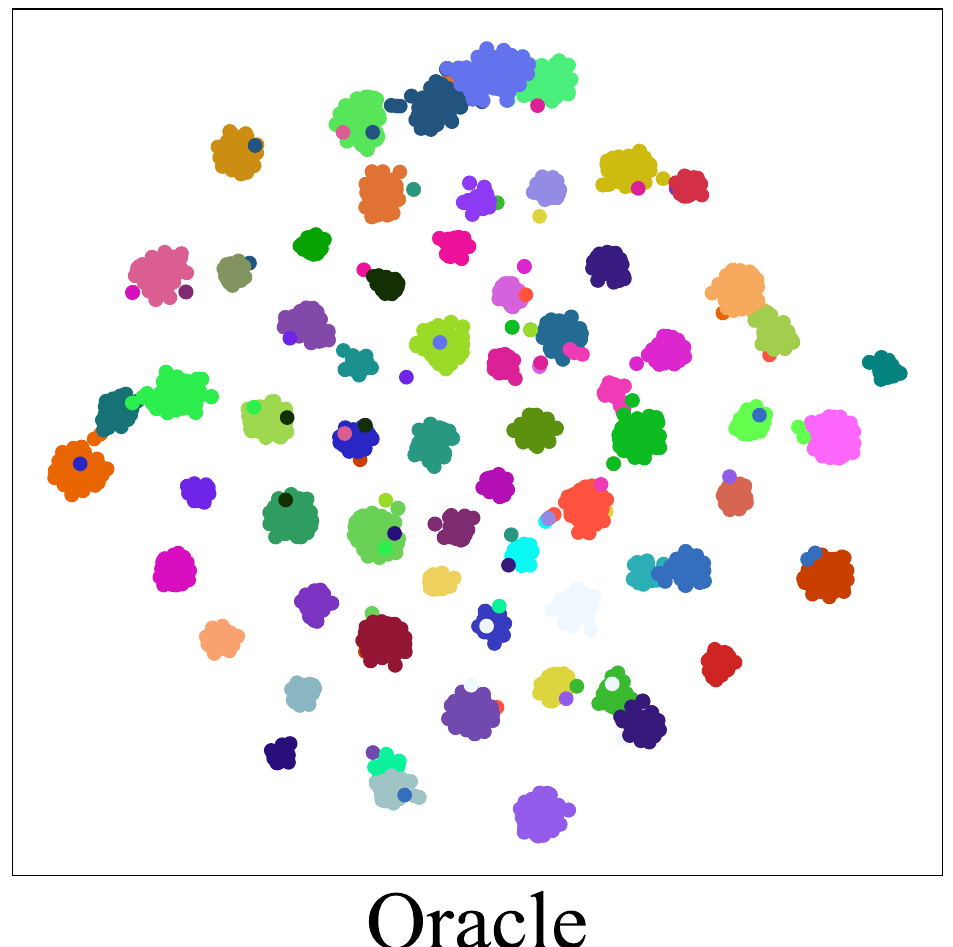}}\\
    \vspace{5pt}
    {\includegraphics[width=0.13\linewidth,height=0.135\linewidth]{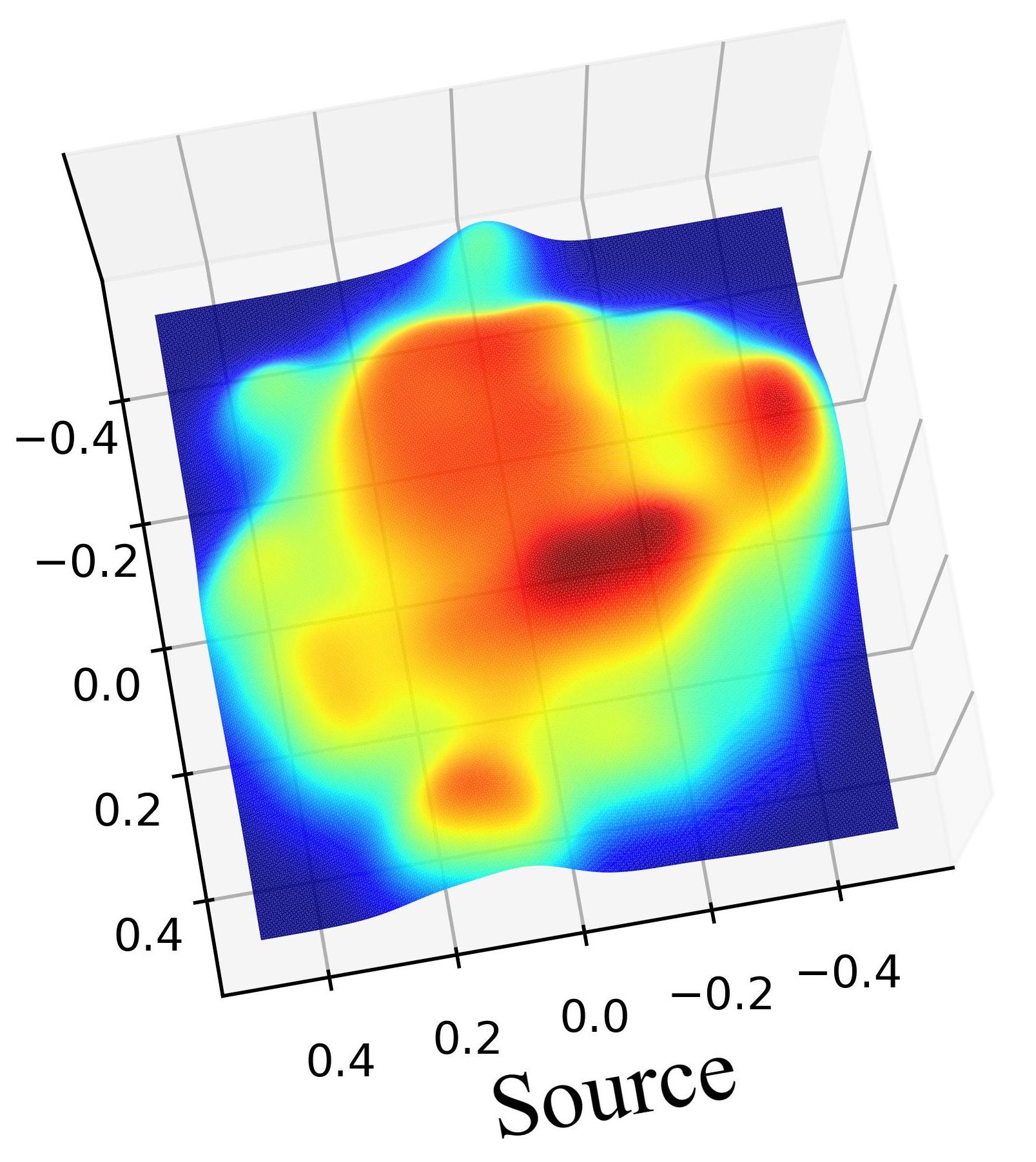}}
    {\includegraphics[width=0.13\linewidth,height=0.135\linewidth]{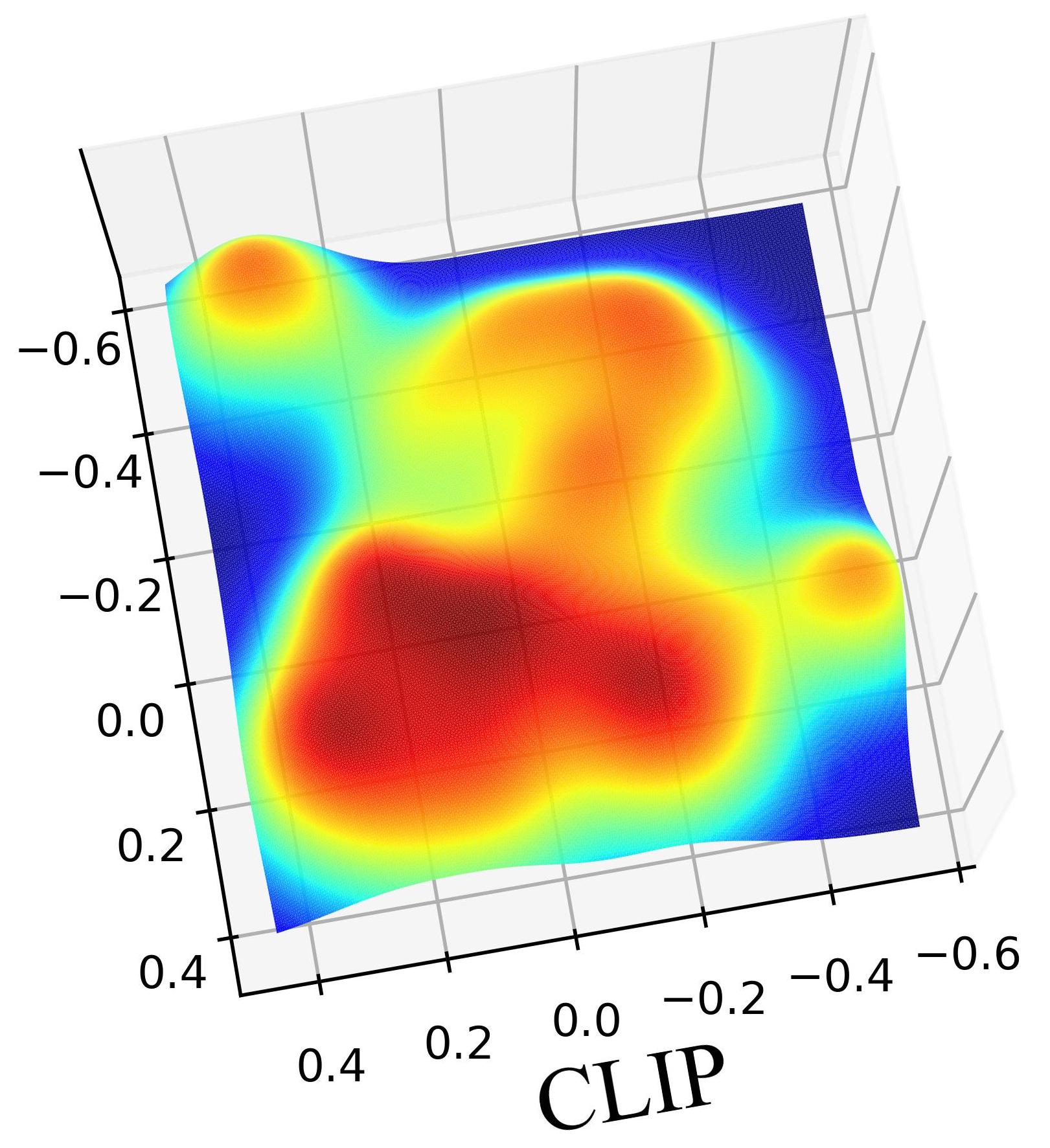}}
    {\includegraphics[width=0.13\linewidth,height=0.135\linewidth]{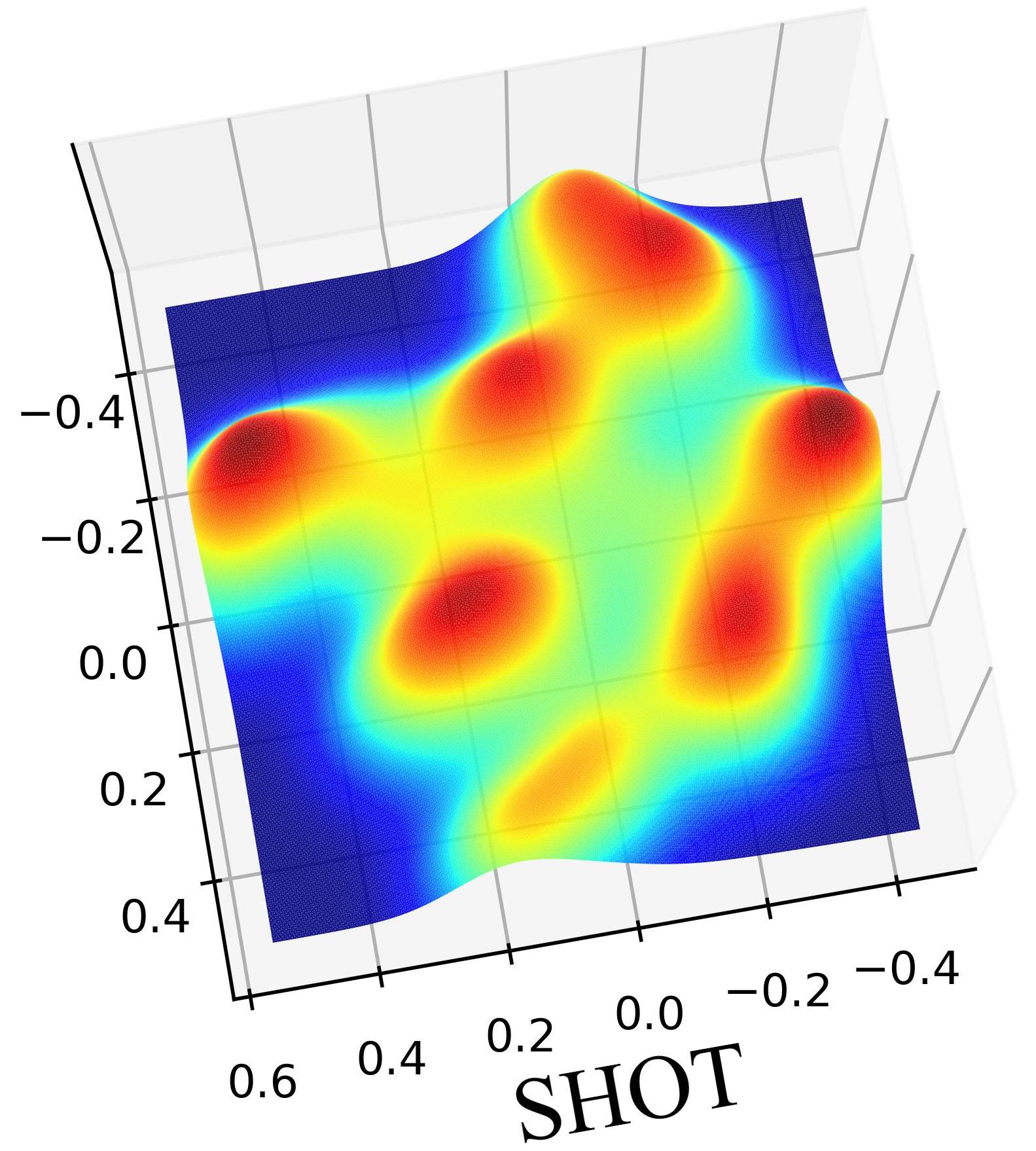}}
    {\includegraphics[width=0.13\linewidth,height=0.135\linewidth]{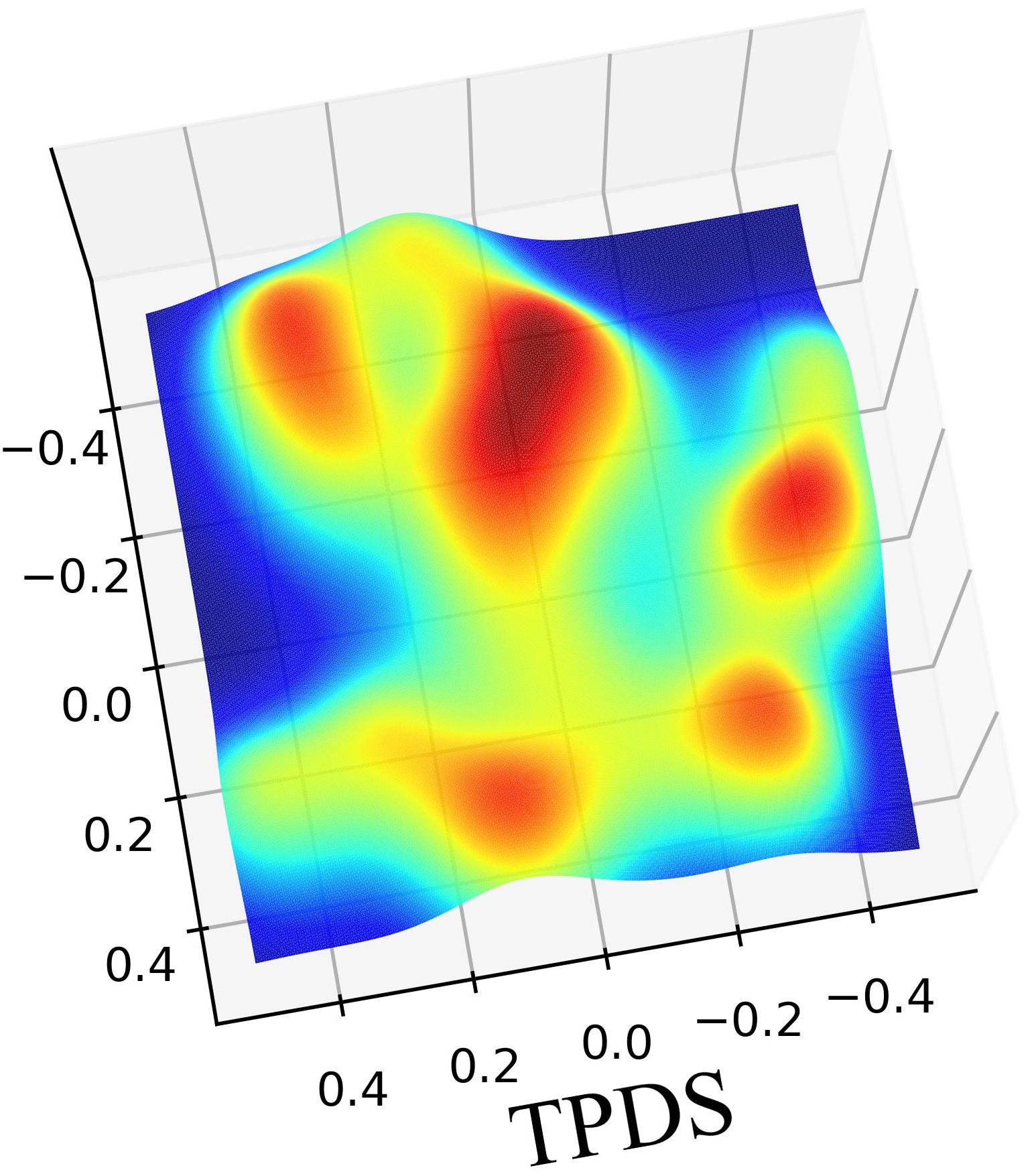}}
    {\includegraphics[width=0.13\linewidth,height=0.135\linewidth]{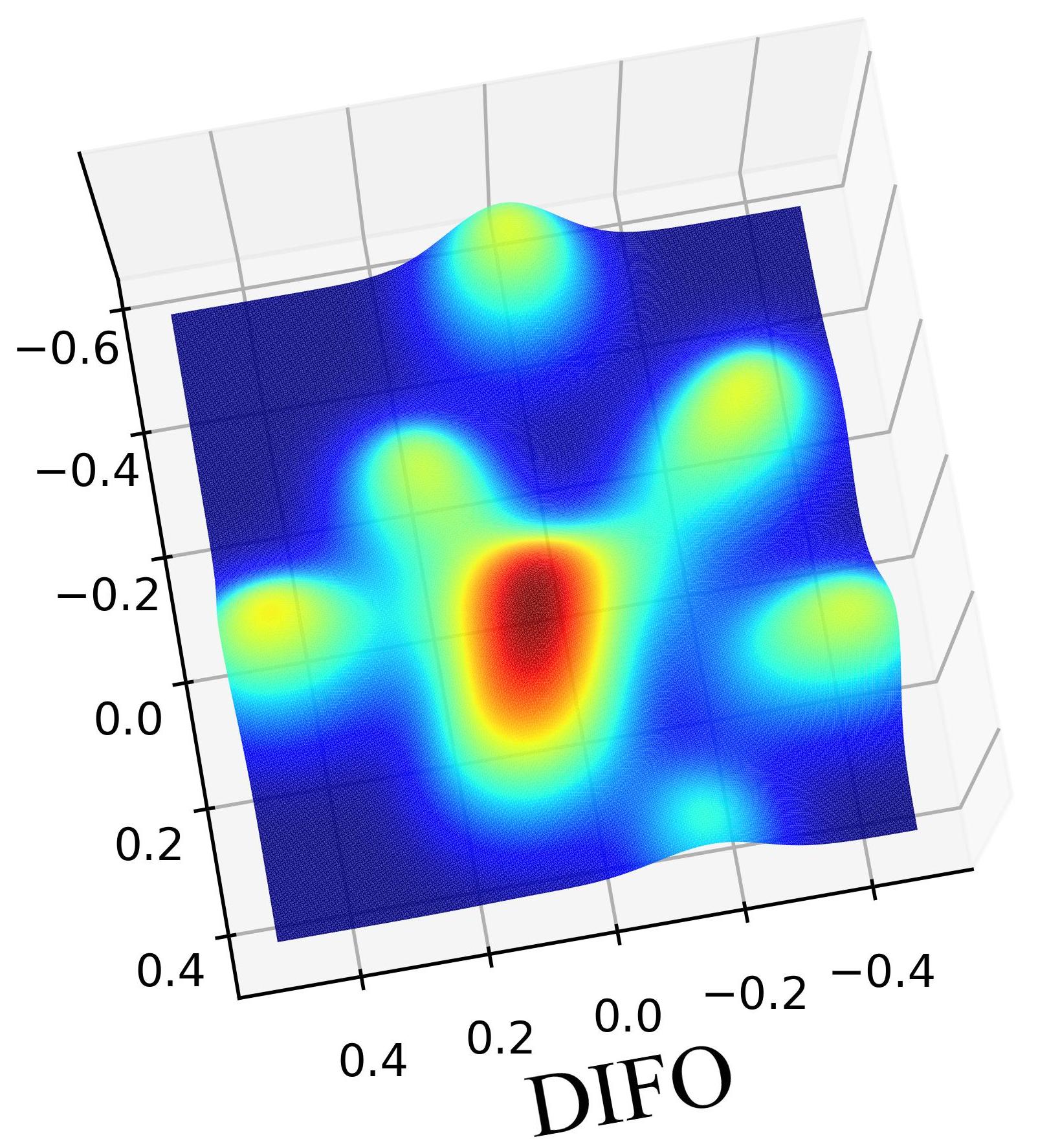}}
    {\includegraphics[width=0.13\linewidth,height=0.135\linewidth]
    {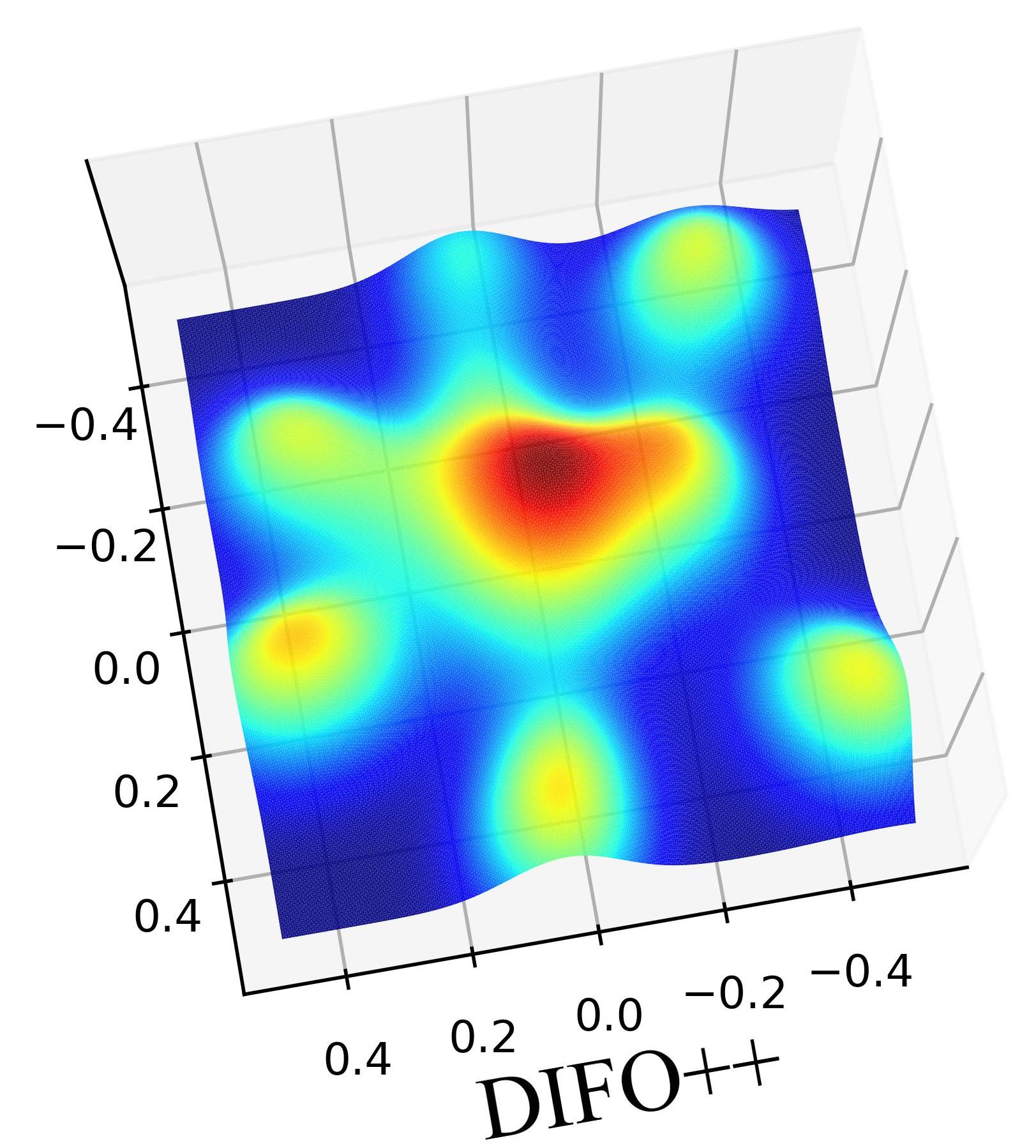}}
    {\includegraphics[width=0.13\linewidth,height=0.135\linewidth]{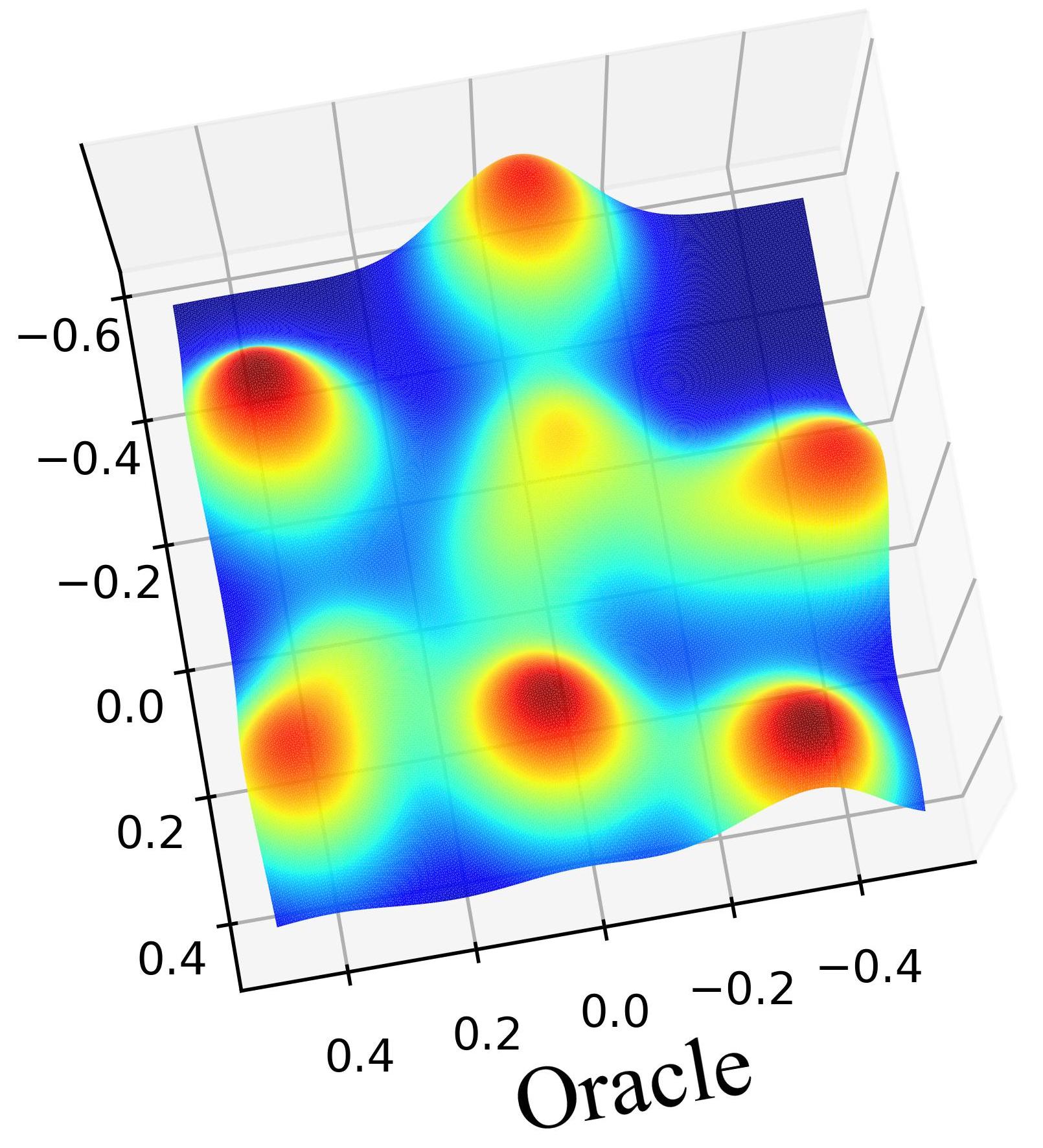}}
    \setlength{\abovecaptionskip}{0.3cm}
    \caption{
     Feature distribution visualization on transfer task Ar$\to$Cl in Office-Home. 
     Oracle: the model trained on target domain Cl using ground-truth labels. 
     Different colors stand for different categories. 
     {\bf Top}: t-SNE feature distribution over 65 categories.  
     {\bf Bottom}: The corresponding 3D density charts. 
     For clarity, the first 10 categories were used in this plot. 
     } 
    \label{fig:visfea}
\end{figure*}

\begin{figure*}[t]
    \centering
        \includegraphics[width=0.13\linewidth,height=0.13\linewidth]{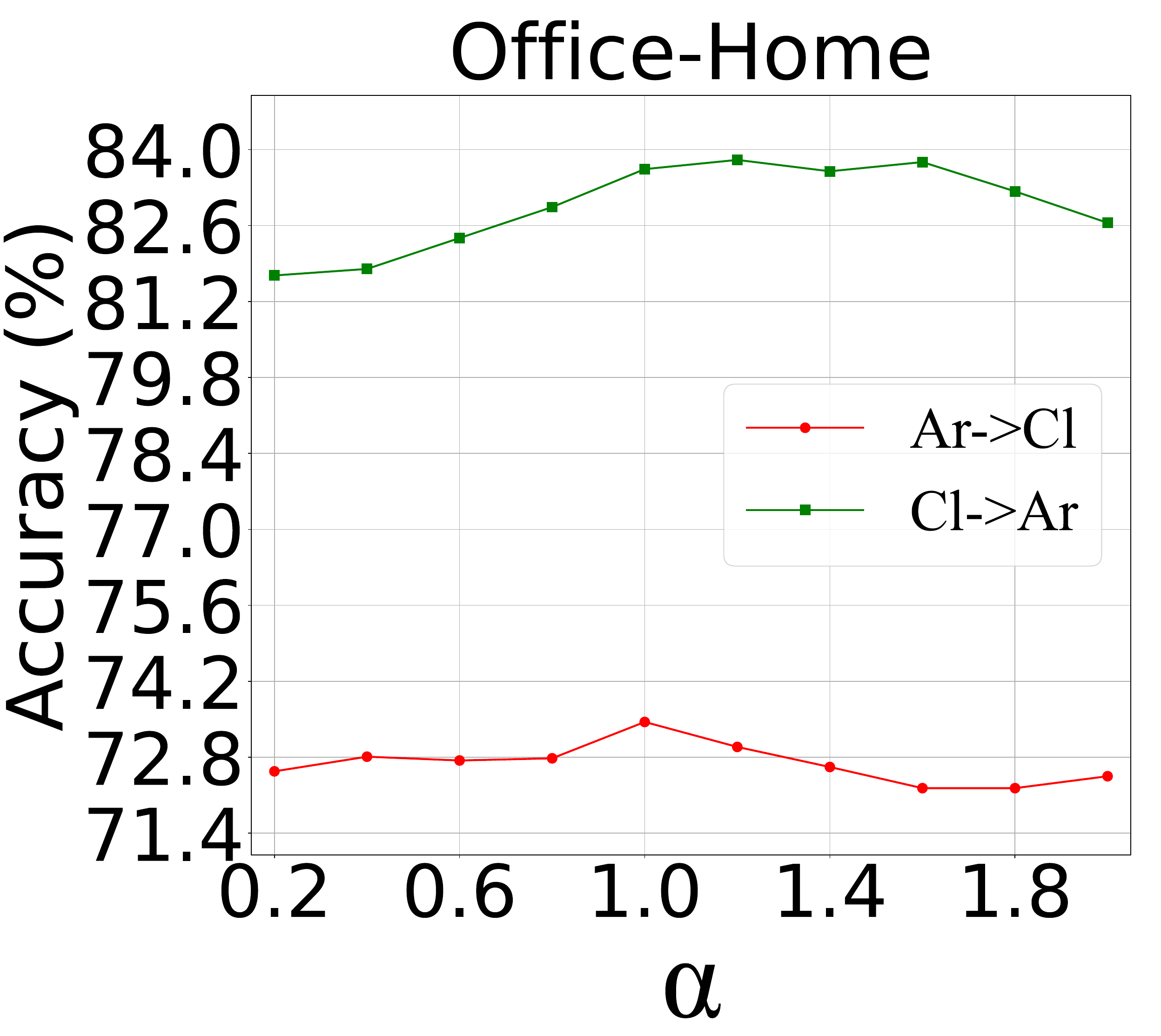}
        \includegraphics[width=0.13\linewidth,height=0.13\linewidth]{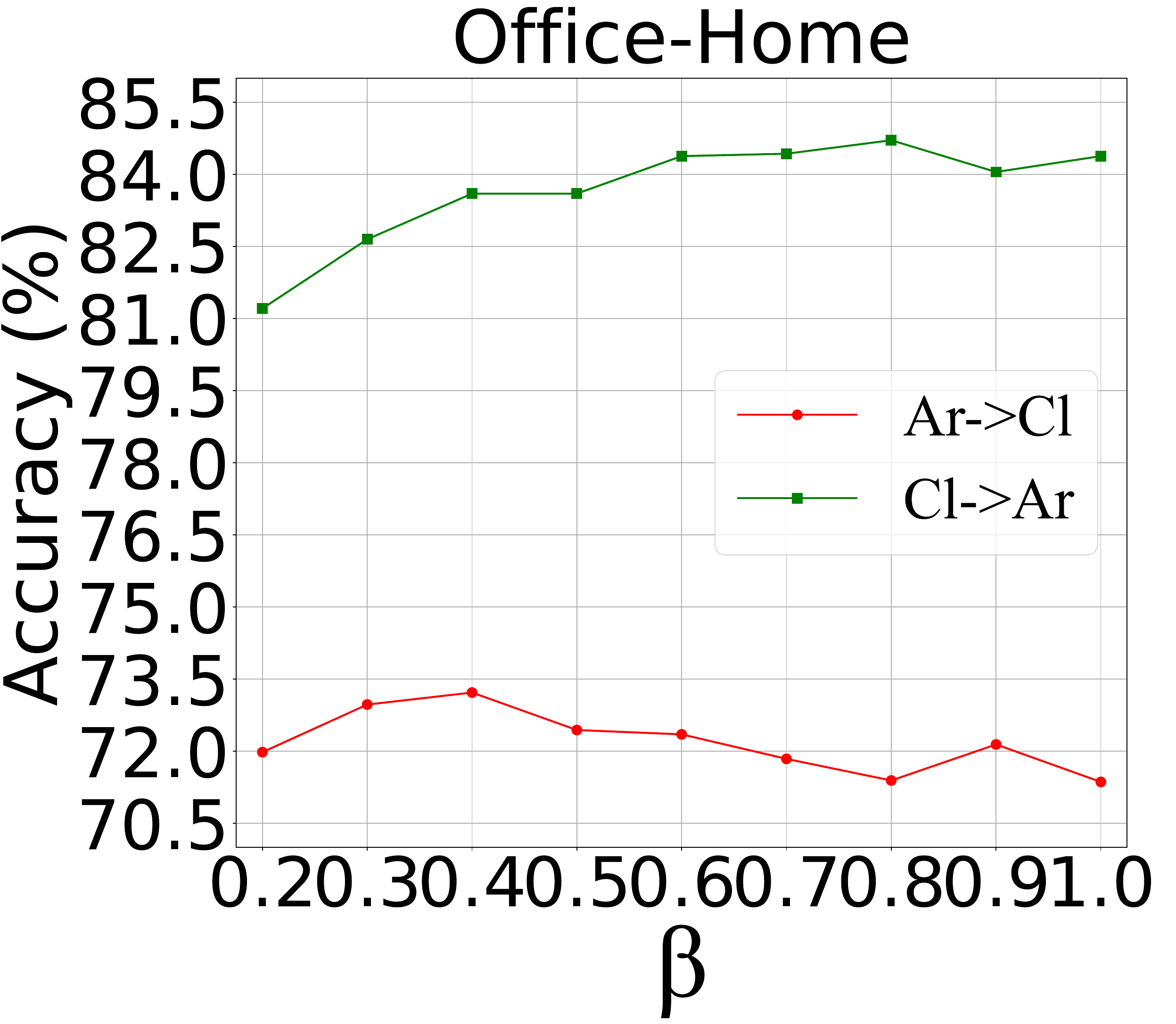}
        \includegraphics[width=0.13\linewidth,height=0.13\linewidth]{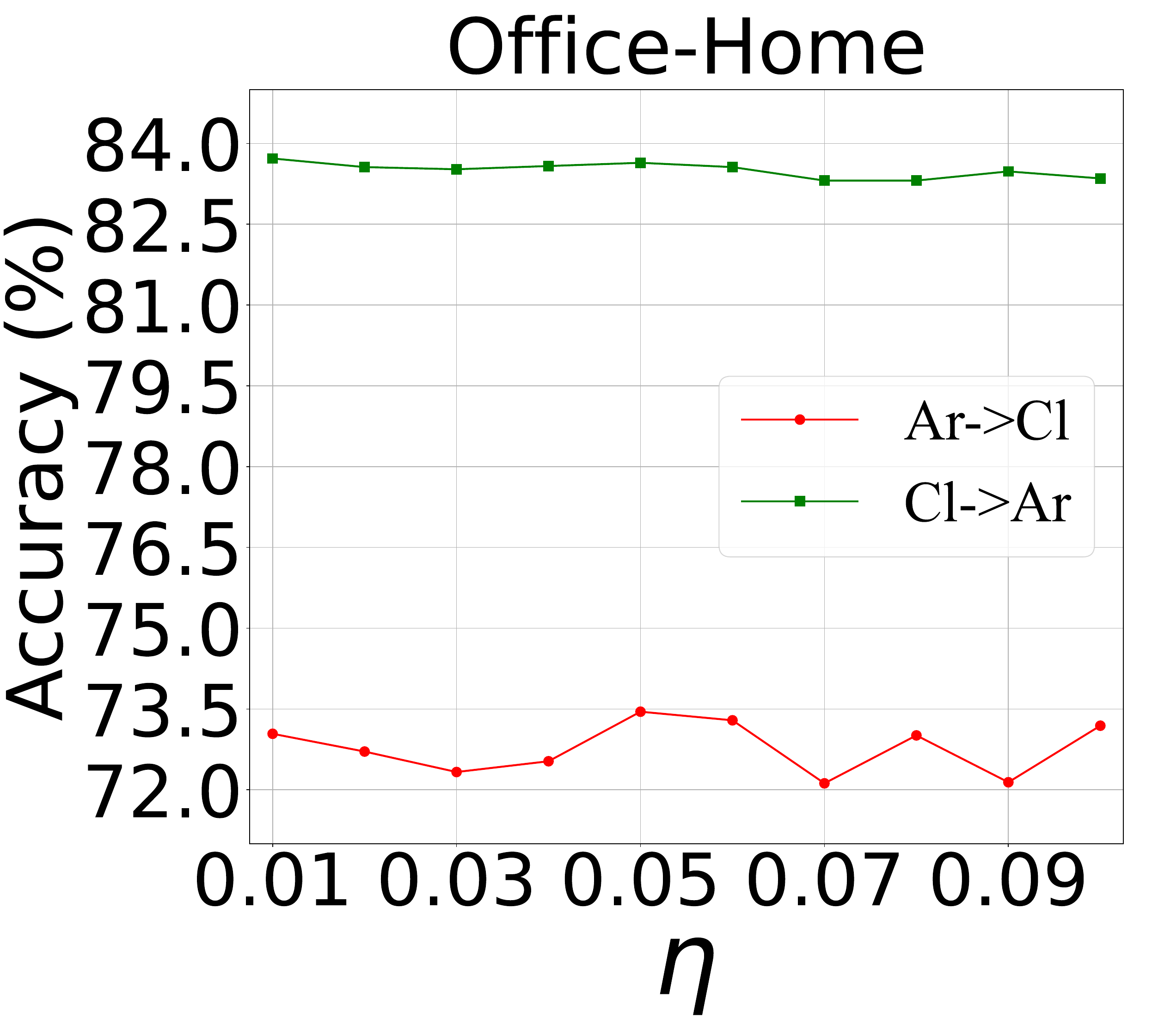}
        \includegraphics[width=0.13\linewidth,height=0.13\linewidth]{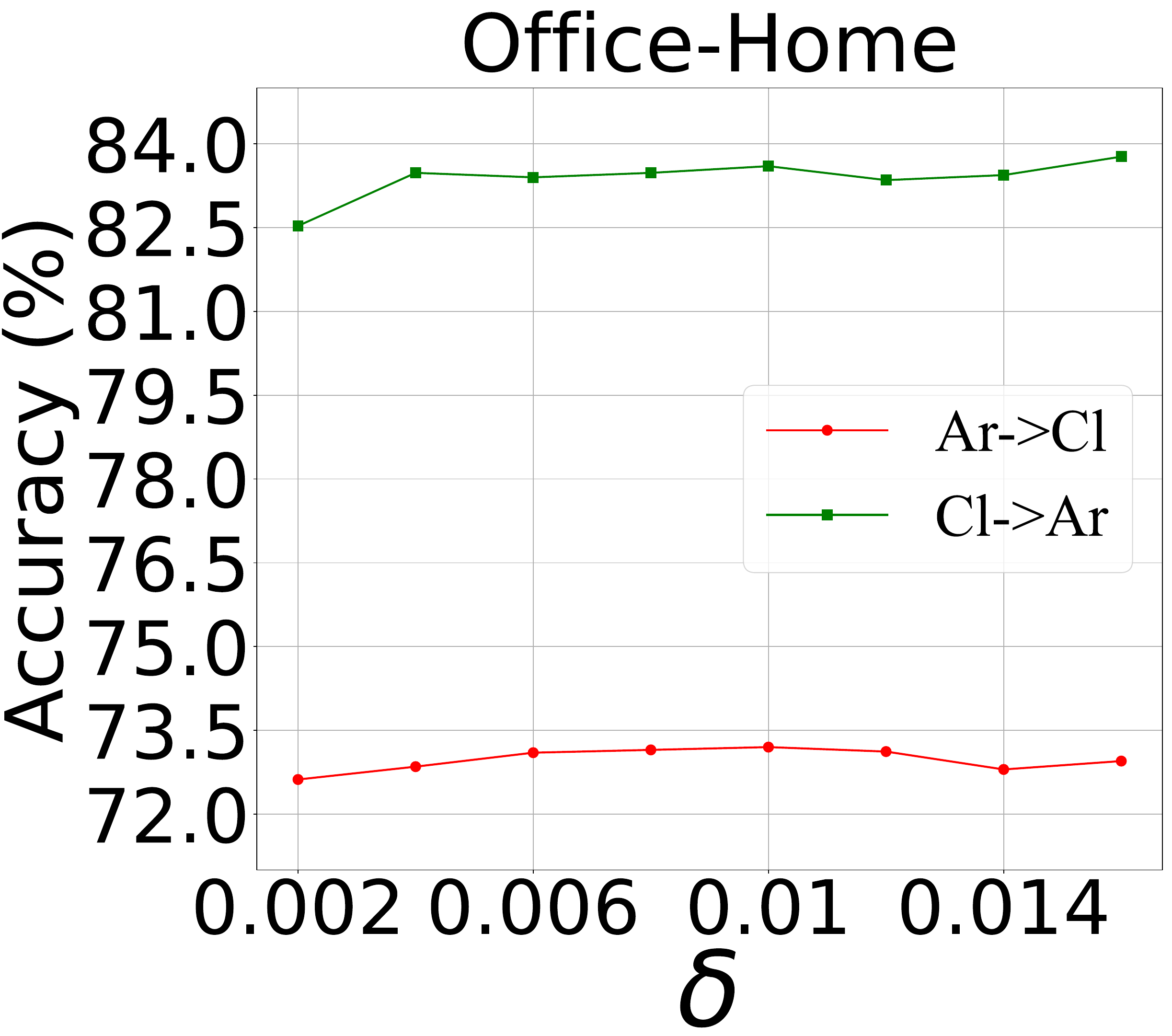}
        \includegraphics[width=0.13\linewidth,height=0.13\linewidth]{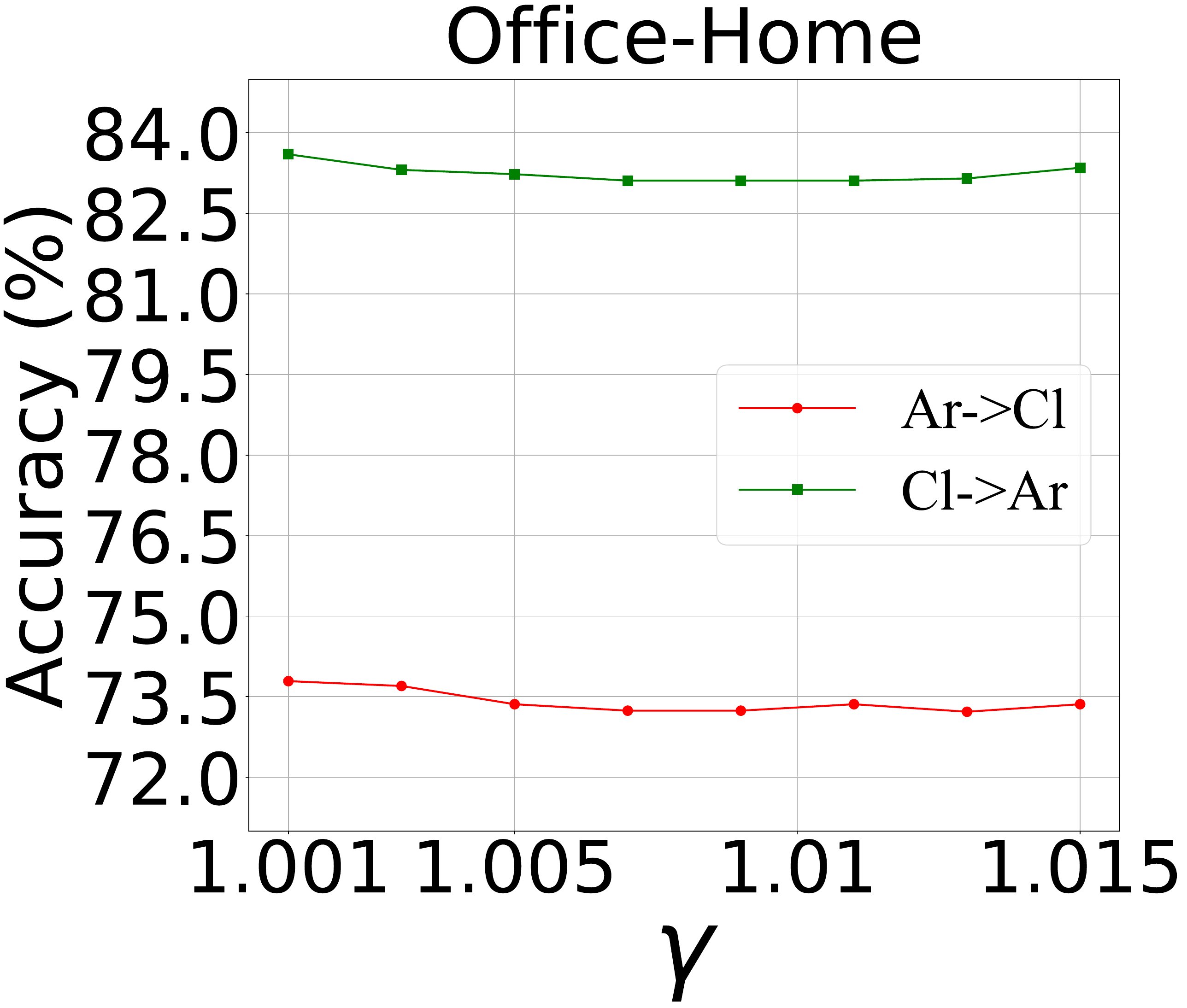}
        \includegraphics[width=0.13\linewidth,height=0.13\linewidth]{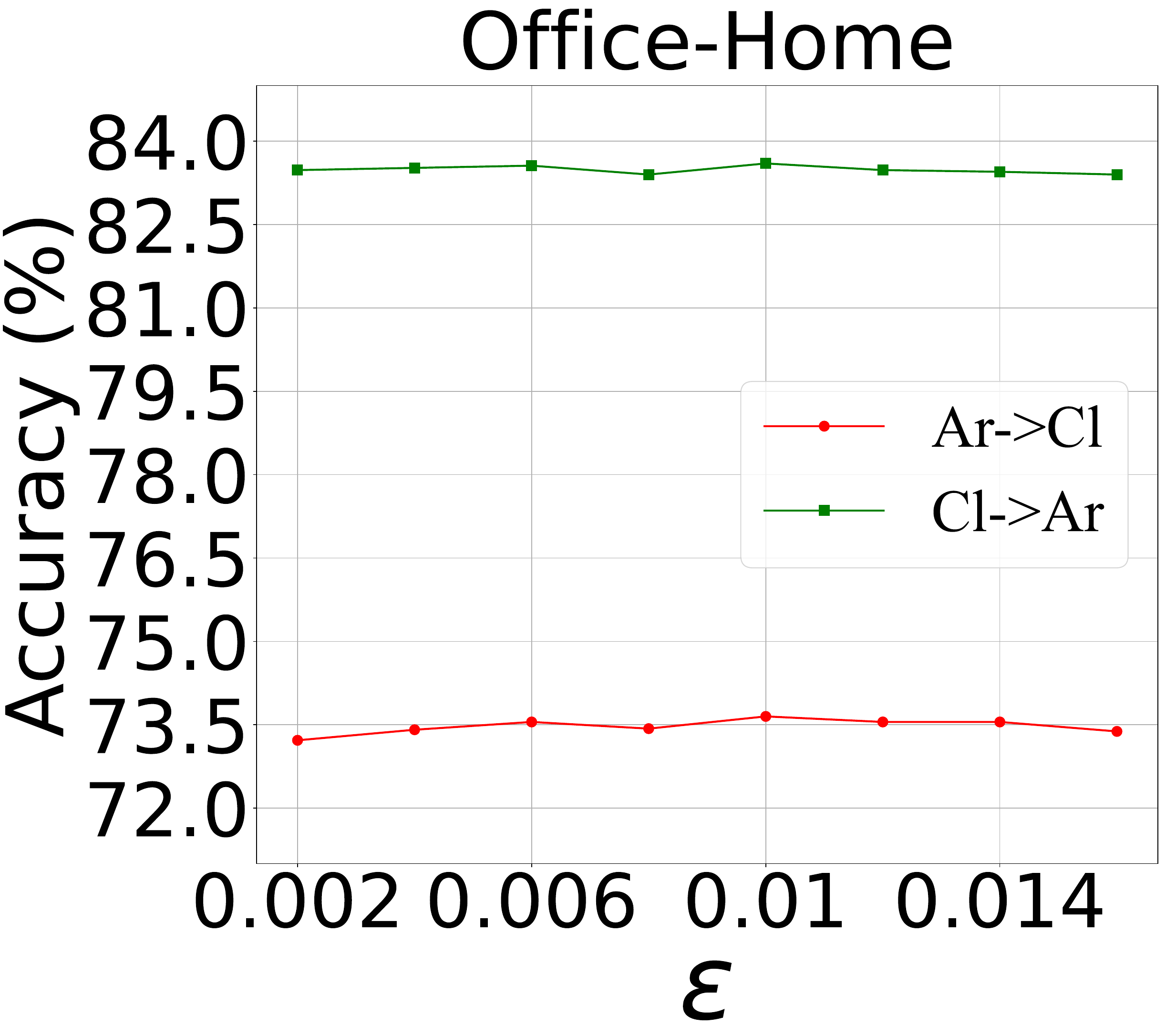}
        \includegraphics[width=0.13\linewidth,height=0.13\linewidth]{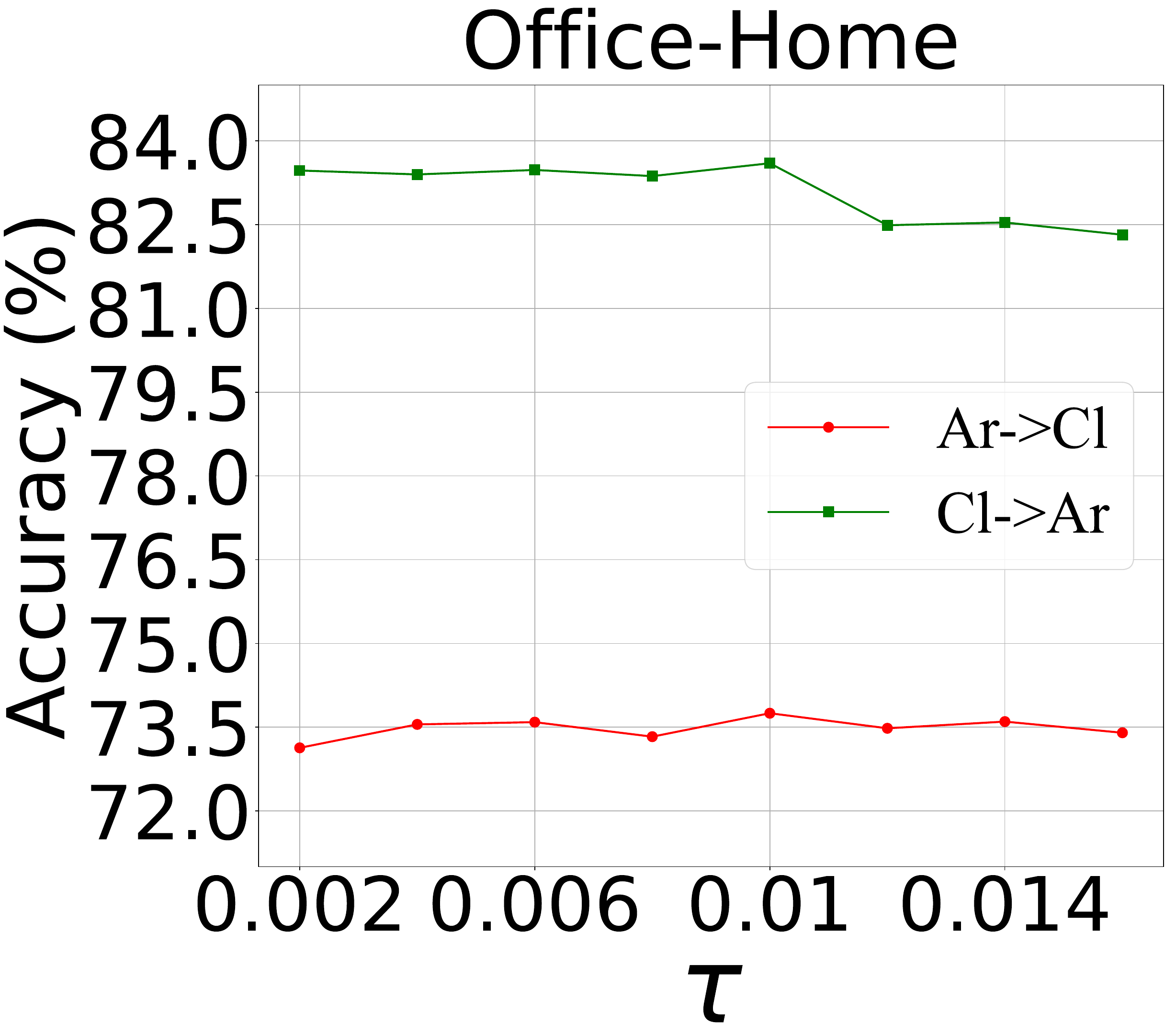}
    \caption{
     \changed{Sensitivity of the hyper-parameters $\alpha$, $\beta$, $\eta$, $\delta$, $\gamma$, $\epsilon$, and $\tau$.} 
    } 
    \label{fig:param}
\end{figure*}

\begin{figure}[t]
    \centering
        \includegraphics[width=0.35\linewidth]{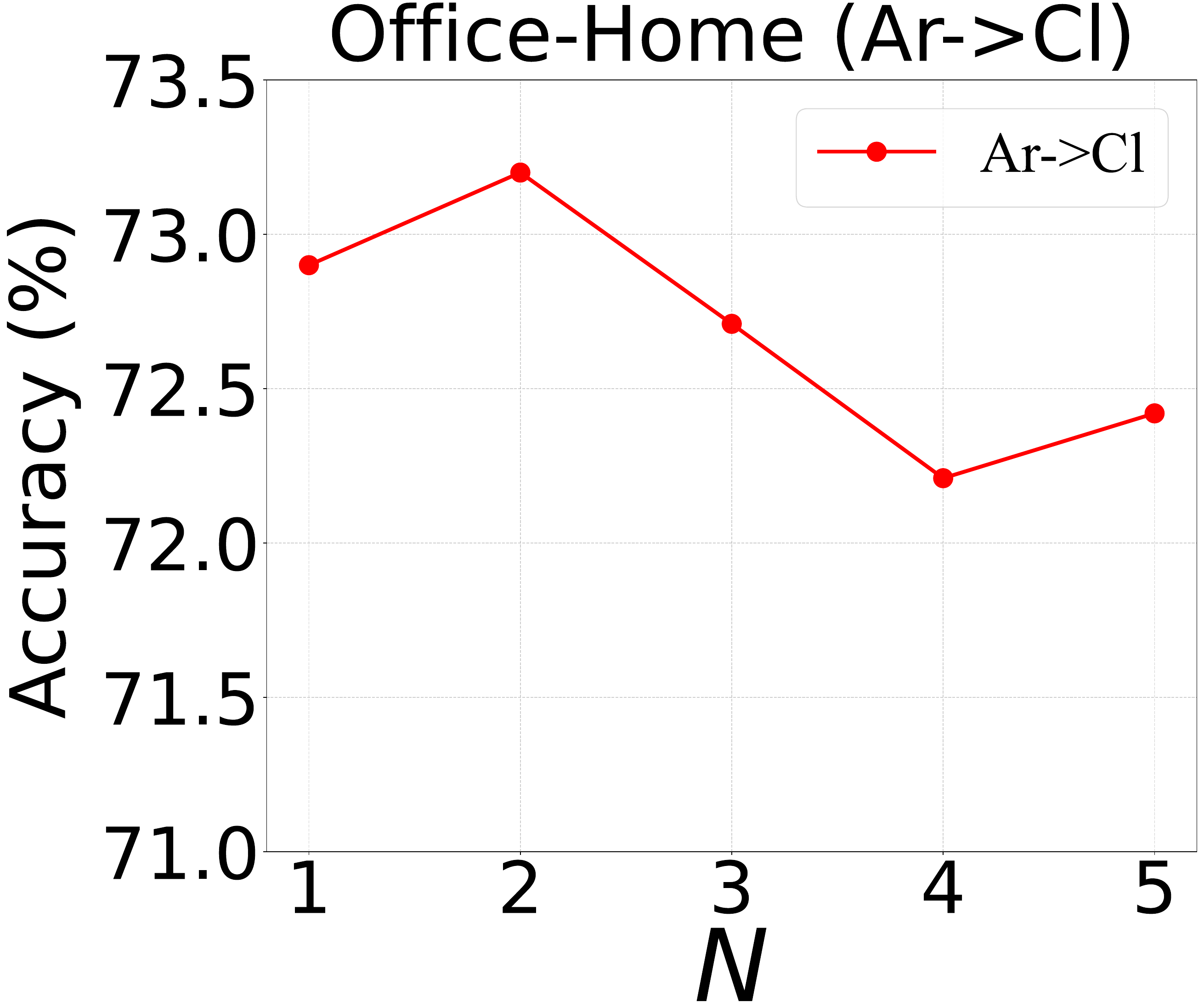}~~
        \includegraphics[width=0.35\linewidth]{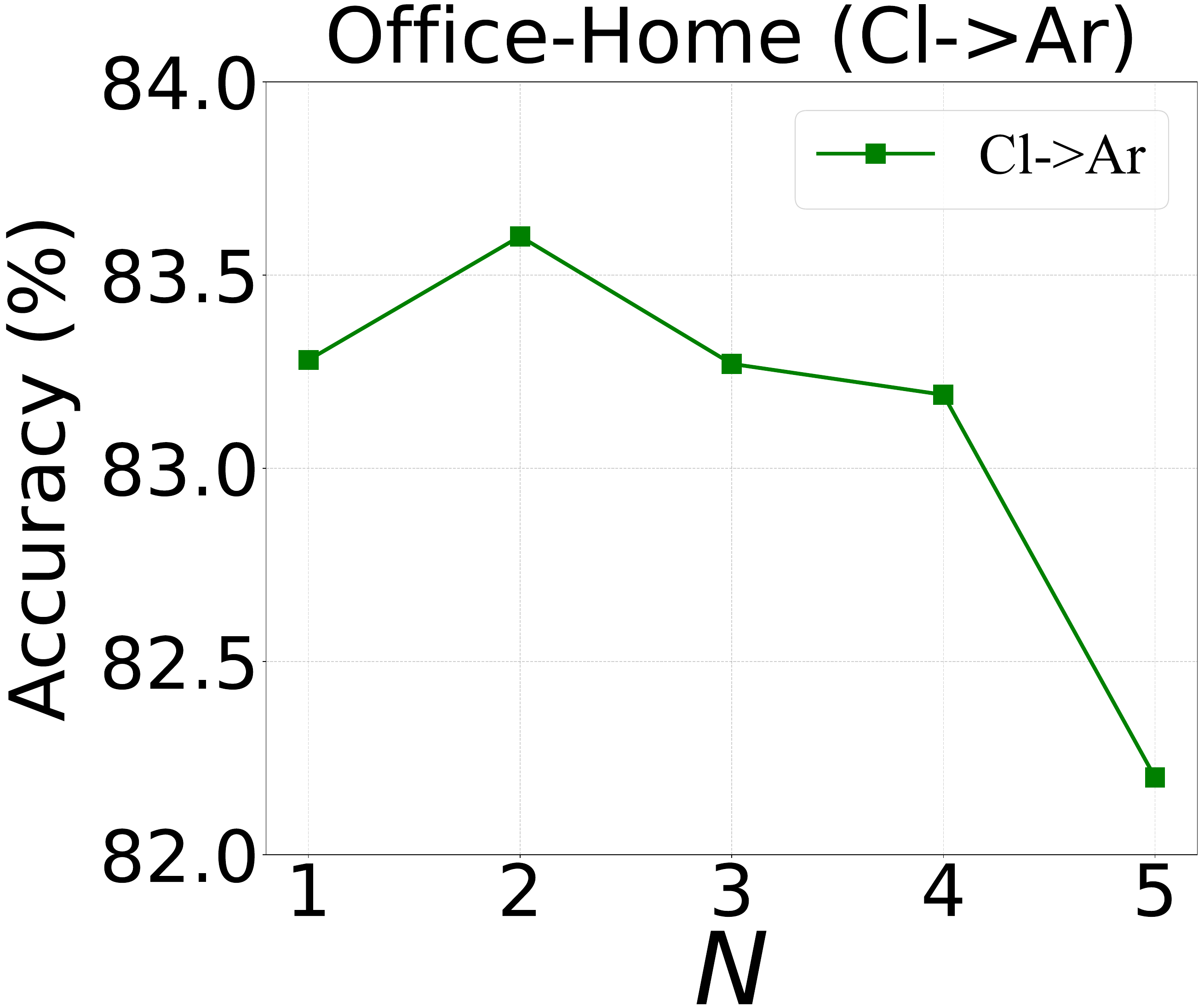}
    \caption{
    \changed{Sensitivity of parameters N.} 
    } 
    \label{fig:param-N}
\end{figure}

\changed{
{\bf Effect of referenced entropy.} 
We introduce a variant, {\modelshortname} w/ ENT, which substitutes our referenced uncertainty with conventional information entropy. 
As reported in Row 15 of Tab.~\ref{tab:ab_loss}, this substitution results in a performance drop of \textbf{1.4\%}.
Furthermore, we conduct a comparative analysis against established uncertainty metrics, including conventional entropy, Energy-based scores \citep{liu2020energy}, and Class-Margin \citep{f0song2024multi}, under an identical experimental configuration as {\modelshortname}. As illustrated in Fig.~\ref{fig:metric}, our referenced entropy (Ref-Entropy) exhibits a superior convergence profile throughout the adaptation phase. These empirical results underscore the effectiveness of Ref-Entropy in providing more stable and robust guidance for our curriculum-learning-based knowledge distillation. 
}

\subsection{Further Model Analysis}
\vspace{0.1cm}
\noindent 
\textbf{Grad-CAM visualization.}
To better understand our {\modelshortname} method, we present a Grad-CAM visualization comparison between the source model, the state-of-the-art SFDA methods SHOT, DIFO, and our {\modelshortname}. 
Specifically, as shown in Fig.~\ref{fig:cam-dys}, for images containing a single object (see columns 1 to 3), {\modelshortname} accurately focuses on the discriminative visual part, whereas other methods distribute attention across the entire image. 
For images containing multiple objects (see columns 4 to 6), {\modelshortname} shows superior alignment with the target semantics described by the ground truth labels, compared to other methods, which often misdirect attention to incorrect objects.
These findings highlight the effectiveness of {\modelshortname} in combining the domain generality of ViL models with the task-specific adaptability of source models.

\textbf{Feature distribution visualization.}
To visualize feature distribution, we conducted a toy experiment on the Ar$\to$Cl task in the Office-Home dataset, using the t-SNE tool. 
We include five comparison models: CLIP's zero-shot (labeled as CLIP), SHOT, TPDS, {\modelshortnameold}, and Oracle (trained on domain Cl with real labels). As shown at the top of Fig.~\ref{fig:visfea}, from CLIP to {\modelshortname}, category aliasing gradually diminishes. 
Compared to Oracle, {\modelshortname} exhibits the most similar distribution shape. 
In particular, the mixture degree of multi-class in the middle zone is weaker than that of {\modelshortnameold}. 
The observation offers an explanation for the advantage of {\modelshortname} over {\modelshortnameold}.
To further validate this observation, we present the 3D density chart results at the bottom of Fig.~\ref{fig:visfea}. 
These results confirm the effectiveness of {\modelshortname} in terms of feature distribution.

\vspace{0.1cm}
\textbf{Sensitivity to hyper-parameters.}
In {\modelshortname}, $\alpha$ in objective loss $L_{\rm{PC}}$ (Eq.~\eqref{eqn:loss_pc}) and $(\beta, \eta)$ in $\mathcal{L}{_{{\text{S-II}}}}$ (Eq.~\eqref{eqn:loss-ka}), are the trade-off parameters.  
\changed{
$\delta$ in Eq.~\eqref{eqn:uncer-anchor} and $\tau$ in Eq.~\eqref{eqn:onehot-fusion} are weighted parameters, while threshold parameters $(\epsilon, \gamma)$ in Eq.~\eqref{eqn:confi-split} govern the sample selection process within the curriculum learning framework.
}
We test their sensitivity of performance on the symmetric transfer tasks Cl$\to$Ar and Ar$\to$Cl in Office-Home. 
Fig.~\ref{fig:param} shows no significant drops of accuracy from varying these parameters,
suggesting good stability.

\changed{
In addition, we investigate the sensitivity of the hyperparameter $N$ in $\mathcal{L}_{\text{cac}}$, as illustrated in Fig.~\ref{fig:param-N}. We observe that performance peaks at $N=2$ and subsequently declines as $N$ increases. This indicates that although considering multiple top categories is beneficial, a high value of $N$ leads to the inclusion of excessive noise. Based on our empirical findings, $N=2$ serves as a robust choice across different scenarios.
}

\begin{figure*}[t]
    \centering
        \includegraphics[width=0.135\linewidth,height=0.135\linewidth]{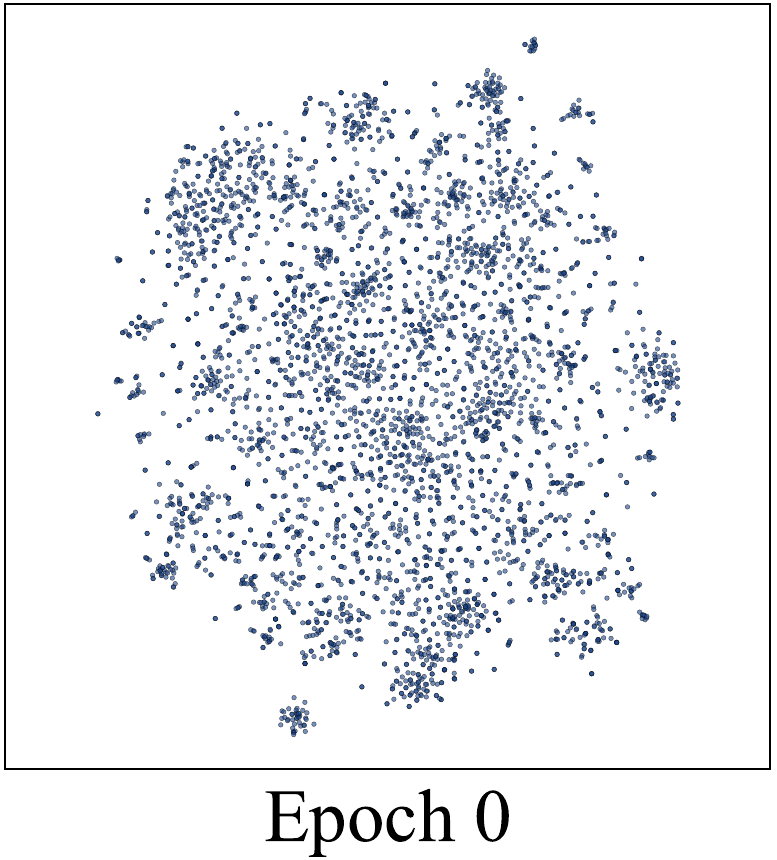} 
        \includegraphics[width=0.135\linewidth,height=0.135\linewidth]{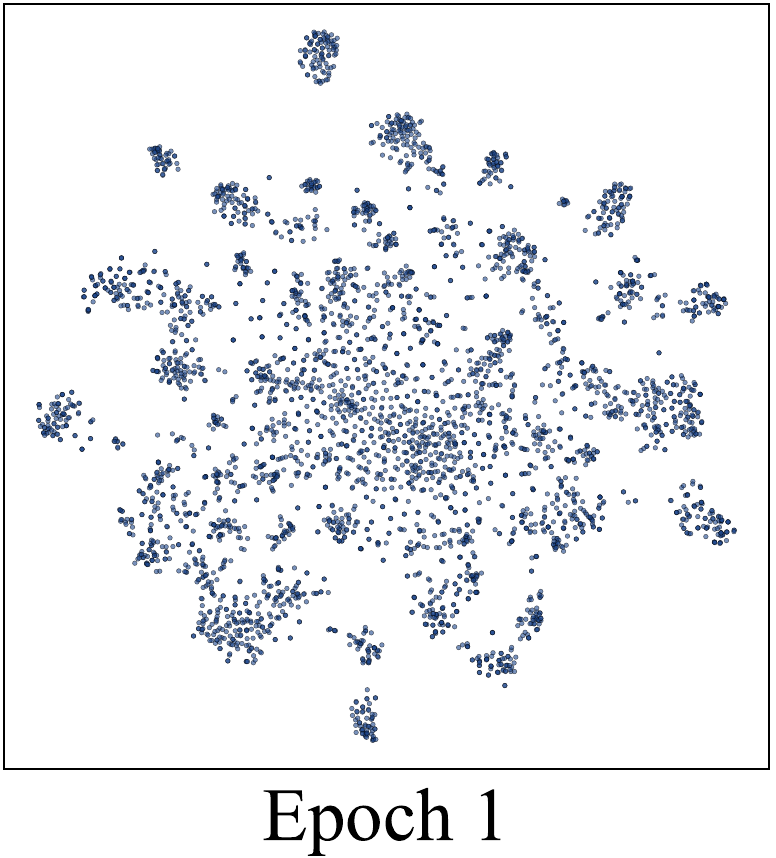}
        \includegraphics[width=0.135\linewidth,height=0.135\linewidth]{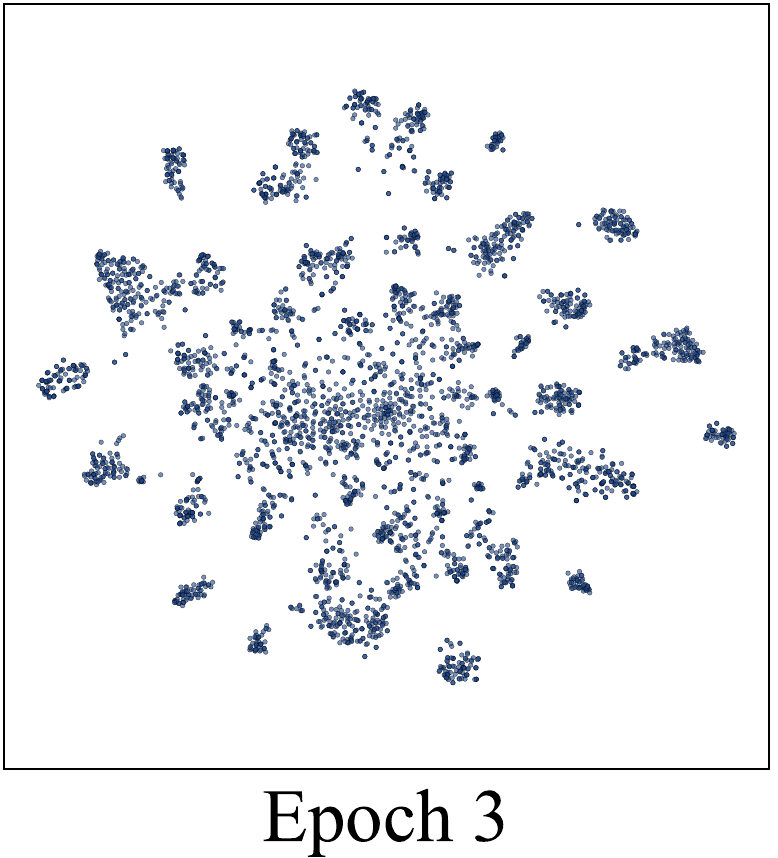}
        \includegraphics[width=0.135\linewidth,height=0.135\linewidth]{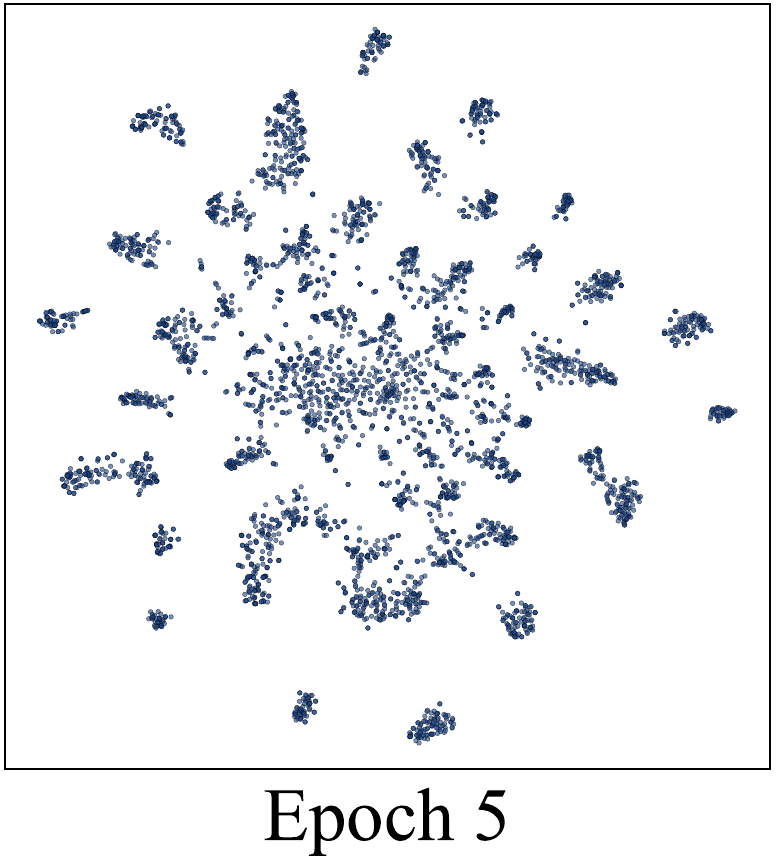} 
        \includegraphics[width=0.135\linewidth,height=0.135\linewidth]{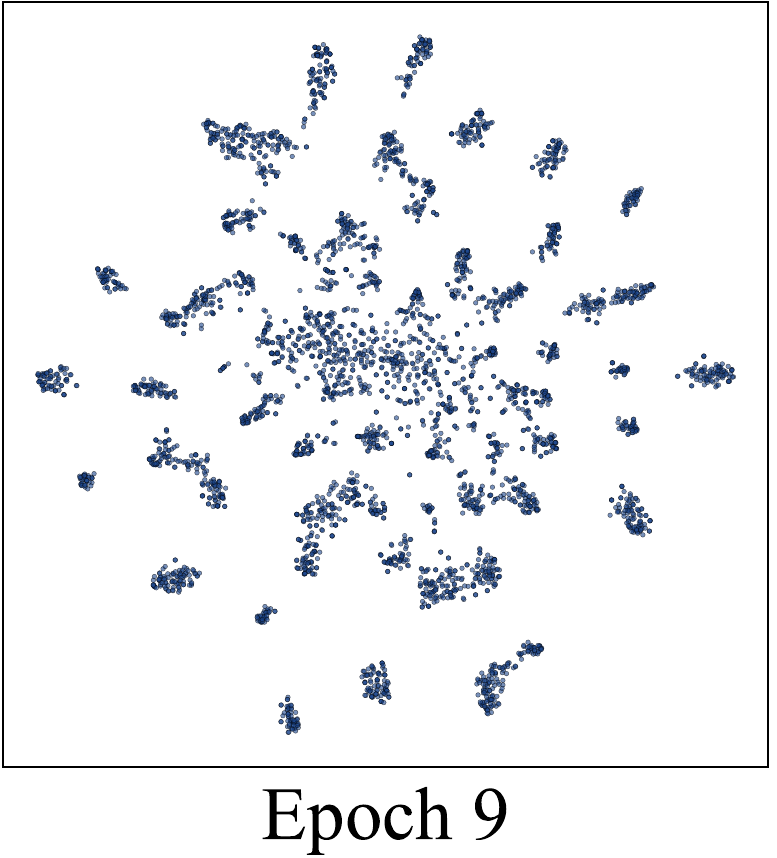}
        \includegraphics[width=0.135\linewidth,height=0.135\linewidth]{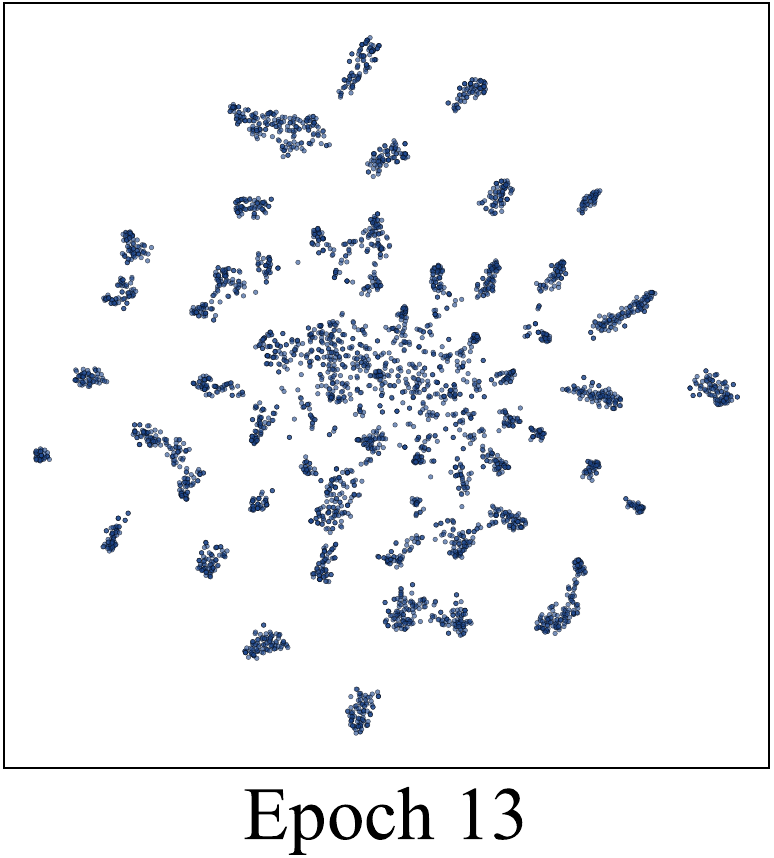}
        \includegraphics[width=0.135\linewidth,height=0.135\linewidth]{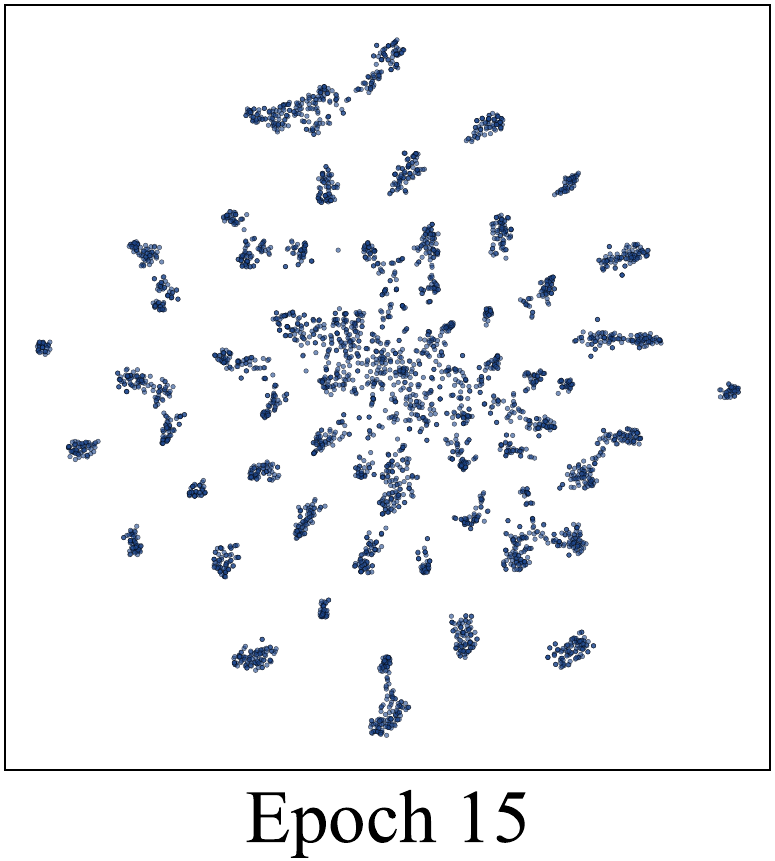}
    \caption{
     \changed{Dynamics of gap-region reduction in epoch-wise view.} 
    } 
    \label{fig:dyn-gap}
\end{figure*}

\changed{
\textbf{Dynamics of gap-region reduction.}
To validate that our method effectively encourages gap-region reduction, we epoch-wise visualize the adaptation using the t-SNE tool. 
As illustrated in Fig.~\ref{fig:dyn-gap}, at Epoch 0 (representing the source model results), the feature space exhibits several isolated clusters surrounded by a large number of samples forming a mist-like distribution, which constitutes the 'gap-region'. 
Throughout the adaptation process, these diffuse features are progressively eliminated, resulting in clear semantic clustering. Such a transition confirms the efficacy of our approach in driving gap-region reduction and feature alignment.
}

\changed{
\textbf{Resource demands.}
In this part, we evaluate the resource demands of {\modelshortname}, as reported in Tab.~\ref{tab:resource_comparison}.
Regarding GPU memory consumption, {\modelshortnameold} and {\modelshortname} require an additional 2 GB compared to the baseline SHOT and TPDS. This overhead is primarily attributed to the integration of the external Vision-Language (ViL) model. However, the computational latency per iteration does not increase substantially; compared to SHOT and TPDS, our methods only introduce an additional 0.5 s and 0.2 s, respectively, due to ViL inference and prompt learning. Notably, {\modelshortname} exhibits lower overall resource demands than {\modelshortnameold}. This efficiency gain is expected, as the curriculum learning strategy adopted in {\modelshortname} gradually reduces the effective scale of training data throughout the adaptation process.
}

\changed{
\textbf{Convergence analysis.}
In Fig.~\ref{fig:convergence}, we illustrate the convergence behavior of {\modelshortname} on the VisDA dataset throughout the adaptation phase. For benchmarking, we compare our method against SHOT, AaD, and TPDS. It is evident that the area under the convergence curve for {\modelshortname} is significantly larger than that of the competitors. This result demonstrates that {\modelshortname} achieves superior classification performance with an accelerated convergence rate, reaching a higher steady-state accuracy in fewer iterations. 
}

\begin{table}[t]
    \centering
    \caption{\changed{Resource demands on Ar$\rightarrow$Cl task.}}
    \label{tab:resource_comparison}
    \scriptsize
    \setlength{\tabcolsep}{4pt}
    \begin{tabular}{clcccc}
        \toprule
        \# & Item & SHOT & TPDS & DIFO & \modelshortname \\
        \midrule
        1 & GPU Mem. cost $\downarrow$ (G) & \textcolor{cmblu}{\textbf{7.868}} & 8.535 & 10.164 & 10.060 \\
        2 & Time per iter. $\downarrow$ (s) & \textcolor{cmblu}{\textbf{0.407}} & 0.742 & 0.929 & 0.903 \\
        \bottomrule
    \end{tabular}
\end{table}

\section{Conclusion}\label{sec:cons}
In this work, we introduce a novel {\modelshortname} approach to address the SFDA problem.
To the best of our knowledge, this marks the initial endeavor to address SFDA by leveraging a pre-trained ViL foundation model, departing from previous approaches that predominantly concentrated on self-mining auxiliary information. 
{\modelshortname} is featured with alternating between customization of the ViL model and the transfer of task-specific knowledge from the customized ViL model. 
Our approach involves three pivotal designs: 
(1) a mutual information-based alignment for ViL customization, 
(2) a gap region-driven knowledge adaptation for transferring the customized ViL knowledge, and
(3) a referenced entropy for uncertainty evaluation.  
Extensive experiments demonstrate that {\modelshortname} outperforms the state-of-the-art alternatives across multiple challenging benchmarks.

\begin{figure}[t]
    \centering
    \includegraphics[width=0.7\linewidth]{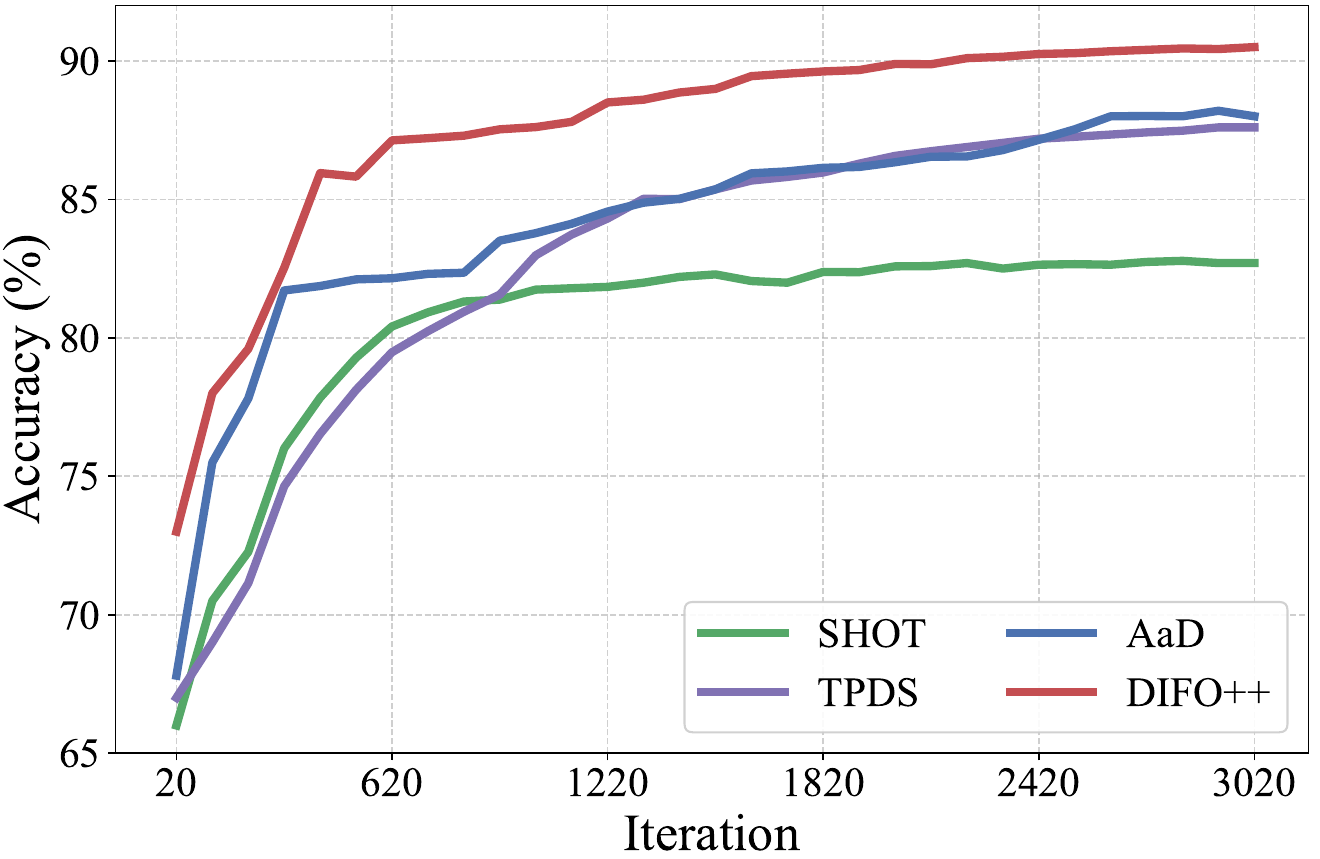}
    \caption{\changed{Convergence behavior comparison on VisDA during the adaptation phase.}}
    \label{fig:convergence}  
\end{figure}

{\bf Limitations.} 
While {\modelshortname} demonstrates the potential of leveraging multimodal foundation models for the SFDA problem, it has several limitations. 
First, its ViL model customization is based on prompt learning, which requires access to the model itself.
As a result, {\modelshortname} cannot be deployed in black-box or cloud-based settings where model access is restricted. 
Second, {\modelshortname} lacks an anti-forgetting mechanism, making it unsuitable for continuous learning scenarios.
Third, the referenced entropy and gap region labeling depend on historical information accumulated during adaptation, limiting its applicability in online settings such as Test-Time Adaptation (TTA).
Extending {\modelshortname} to address these challenging scenarios presents an important direction for future research.

\section*{Acknowledgment}
This work is partially supported by the National Natural Science Foundation of China (NSFC) (62476169); the CAMS Innovation Fund for Medical Sciences (CIFMS) (2025-I2M-C\&T-A-003); the German Research Foundation and NSFC in project Crossmodal Learning under contract Sonderforschungsbereich Transregio 169, the Hamburg Landesforschungsf{\"o}-rderungsprojekt Cross, NSFC (61773083).

\section*{Data Availability}
The authors confirm that the data supporting the findings of this study are available within the articles: \citep{saenko2010adapting} (Office-31), \citep{venkateswara2017deep} (Office-Home) \citep{peng2017visda} (VisDA), and \citep{peng2019moment} (DomainNet-126)

\appendix

\section{A Proof of Lemma 1} \label{app:lemma-proof}
\begin{reslemma}
\textit{Given two random variables $X$, $Y$. Their mutual information ${\rm{I}}\left( X, Y \right)$ and KL divergence $D_{\rm{KL}}\left( X||Y \right)$ satisfy the unequal relationship as follows.} 
\begin{equation}
    \label{eqn:dsib}  
    -{\rm{I}}\left( X, Y \right) \leq D_{\rm{KL}}\left( X, Y \right). 
\end{equation} 
\label{thm-one} 
\end{reslemma}

\begin{proof}
Suppose the probability density function (PDF) of $X$ and $Y$ are $p(\boldsymbol{x})$ and $p(\boldsymbol{y})$, respectively; their joint PDF is $p(\boldsymbol{x},\boldsymbol{y})$. 
We have  
\begin{equation*}
    \begin{split}
    {\rm{I}}\left(X, Y\right) 
    &= \sum p(\boldsymbol{x}, \boldsymbol{y})\log\frac{p(\boldsymbol{x},\boldsymbol{y})}{p(\boldsymbol{x})\cdot p(\boldsymbol{y})} \\
    &= D_{\rm{KL}}\left( p(\boldsymbol{x},\boldsymbol{y})~||~p(\boldsymbol{x})\cdot p(\boldsymbol{y}) \right). \\
    \end{split}
\end{equation*} 

Well known, the KL divergence is non-negative~\citep{eguchi2006interpreting}. 
Thus, 
\begin{equation*}
-{\rm{I}}\left( X, Y \right) \leq 0 \leq D_{\rm{KL}}\left( X || Y \right).
\end{equation*}
\end{proof}

\bibliographystyle{apacite}
\bibliography{sample-base-ijcv}

\end{document}